\documentclass[twoside,11pt]{article}

\usepackage[preprint]{jmlr2e}
\usepackage[protrusion=true,expansion=false]{microtype}

\usepackage{float}
\usepackage{nccmath}
\usepackage{mathtools}
\mathtoolsset{showonlyrefs}
\usepackage{dsfont}
\usepackage{bigints}
\usepackage{cleveref}
\usepackage{crossreftools}
\usepackage{enumerate}
\usepackage{mathrsfs}
\usepackage[inline,shortlabels]{enumitem}
\usepackage{blindtext}

\makeatletter
\newcommand*\widebar[1]{%
  \begingroup
  \def\mathaccent##1##2{%
    \relax
    \ifmmode
      \setbox\z@\hbox{$\m@th\displaystyle#1$}%
      \setbox\tw@\hbox{$\m@th\overline{\phantom{\copy\z@}}$}%
      \dimen@=\ht\tw@
      \advance\dimen@-\dp\tw@
      \advance\dimen@-\ht\z@
      \advance\dimen@-\dp\z@
      \divide\dimen@ by 2
      \raise\dimen@\hbox to\wd\z@{\hfil$\overline{\phantom{\copy\z@}}$\hfil}%
      \box\z@
    \else
      \PackageWarning{widebar}{Argument of \string\widebar\space should be math material.}%
      #1%
    \fi
  }%
  \mathaccent"0362{#1}%
  \endgroup
}

\newsavebox\myboxA
\newsavebox\myboxB
\newlength\mylenA
\newcommand*\xoverline[2][0.75]{%
    \sbox{\myboxA}{$\m@th#2$}%
    \setbox\myboxB\null
    \ht\myboxB=\ht\myboxA
    \dp\myboxB=\dp\myboxA
    \wd\myboxB=#1\wd\myboxA
    \sbox\myboxB{$\m@th\overline{\copy\myboxB}$}%
    \setlength\mylenA{\the\wd\myboxA}%
    \addtolength\mylenA{-\the\wd\myboxB}%
    \ifdim\wd\myboxB<\wd\myboxA
       \rlap{\hskip 0.5\mylenA\usebox\myboxB}{\usebox\myboxA}%
    \else
       \hskip -0.5\mylenA\rlap{\usebox\myboxA}{\hskip 0.5\mylenA\usebox\myboxB}%
    \fi}
\newcommand{\biggg}{\bBigg@\thr@@}
\newcommand{\Biggg}{\bBigg@{3}}

\newcommand{\bigggg}{\bBigg@\thr@@}
\newcommand{\Bigggg}{\bBigg@{3.5}}

\newcommand\runinsubsection{\@startsection{subsection}{1}{\z@}%
                               {3.5ex \@plus 1ex \@minus .2ex}%
                               {-1em}%
                               {\normalfont\normalsize\bfseries}}
\makeatother

\let\thelemma\relax

\newtheorem{lemma}{Lemma}
\newtheorem{remark}{Remark}
\newtheorem{corollary}{Corollary}
\newtheorem{prop}{Proposition}

\newenvironment{noheadproof}[1][]{\noindent\textit{}}{}
\usepackage[T1]{fontenc}
\usepackage{amsmath}
\usepackage{mathtools}
\mathtoolsset{showonlyrefs}
\usepackage{physics}

\usepackage{bm}
\usepackage{enumitem}
\usepackage{booktabs}
\usepackage{tikz}
\usetikzlibrary{arrows.meta,positioning}
\usepackage{fix-cm}
\usepackage{xcolor}
\usepackage{array} 
\usepackage{makecell} 
\usepackage{pgfplots}
\pgfplotsset{compat=1.18} %

\usepackage[caption=false]{subfig} 
\usepackage{placeins}

\usepackage[english]{babel}

\newcommand{\mS}{\m{S}}

\newcommand{\N}{\mb{N}} %
\newcommand{\bomega}{\bm{\omega}} %
\newcommand{\R}{\mb{R}} %
\renewcommand{\S}{\mb{S}} %
\renewcommand{\d}{\mm{d}} %
\renewcommand{\b}{\mathbf} %
\newcommand{\kB}{k^{(\mm{B})}} %
\newcommand{\mP}{\m{P}} %
\newcommand{\mU}{\m{U}} %
\newcommand{\mY}{\m{Y}} %
\newcommand{\mX}{\m{X}} %
\newcommand{\mC}{\m{C}} %
\newcommand{\mH}{\m{H}} %

\newcommand{\mb}[1]{\mathbb{#1}}
\newcommand{\m}[1]{\mathcal{#1}}

\newcommand{\mm}[1]{\mathrm{#1}}

\newcommand{\mV}{\m{V}} %
\newcommand{\VL}{\mV^{(L)}} %
\newcommand{\CLv}{\widehat{\mathcal C}_{\mathrm{var}}^{(L)}} %
\newcommand{\btheta}{\bm{\theta}} %

\definecolor{lightgreen1}{rgb}{0.6, 0.8, 0.6} %
\definecolor{pastelgreen}{rgb}{0.6, 1.0, 0.6}

\allowdisplaybreaks

\setenumerate{labelindent=0em,leftmargin=1.7em,topsep=0cm,partopsep=0cm,parsep=0cm,itemsep=1mm}
\setitemize{labelindent=0em,leftmargin=1.3em,topsep=0cm,partopsep=0cm,parsep=0cm,itemsep=1mm}

\newcommand{\figw}{0.36\linewidth}   %
\newcommand{\figh}{0.33\linewidth}   %

\pgfplotsset{
  compat=1.18, %
  every axis/.append style={
    width=\figw,
    height=\figh,
    xlabel={$u$},
    ylabel={$\ell(u,0)$},
    xmin=-2, xmax=2,
    ymin=0, ymax=4,
    grid=major,
    axis equal image,
    ticklabel style={font=\scriptsize},
    label style={font=\scriptsize},
    title style={font=\scriptsize, yshift=-1ex},
    clip=false,
    scale only axis,
  },
}

\usetikzlibrary{calc,arrows.meta,positioning}

\hypersetup{
  hidelinks,
  pdftitle={Variation Brownian Kernel Ladders},
  pdfauthor={Mahdi Mohammadigohari},
  pdfsubject={Recursive Brownian variation spaces, statistical learning, and constructive approximation},
  pdfkeywords={variation Brownian kernel ladders, variation spaces, Brownian kernels, recursive function spaces, statistical learning theory, constructive approximation}
}

\ShortHeadings{Variation Brownian Kernel Ladders}{Mohammadigohari}
\firstpageno{1}

\makeatletter%
\begin{document}

\makeatletter
\def\label@in@display@noarg#1{%
  \ifx\df@label\@empty\else
    \@amsmath@err{Multiple \string\label's: label '\df@label' will be lost}\@eha
  \fi
  \gdef\df@label{#1}%
}
\def\ltx@label#1{\cref@label{#1}}
\makeatother

\title{Variation Brownian Kernel Ladders}

\author{\name Mahdi Mohammadigohari \email Mahdi.Mohammadigohari@gmail.com \\
       \addr Faculty of Engineering,
       Free University of Bozen-Bolzano\\
       Bruno Buozzi 1, Bolzano, 39100, Italy} 

\editor{}

\maketitle

\begin{abstract}

Claims about the benefit of depth depend on the complexity assigned to a
representation.
We introduce the
\emph{Variation Brownian Kernel Ladder}
(VBKL), a path-atomic function-space framework that separates nonlinear
recursive dictionary construction from linear variation superposition.
Starting from linear projections, each atom recursively composes unit-ball
profiles from the Brownian reproducing kernel Hilbert space; the full VBKL
space is then the signed-measure variation hull of the completed dictionary.
We identify each recursive dictionary as a union of Brownian pullback RKHS
balls and establish variation-controlled H\"older regularity, compactness and
attainment, and strict growth with depth under a local non-degeneracy condition
whose trace lies in the support of the input measure.
For associated finite lower-support architectures, we derive Rademacher and
generalization bounds through Brownian quadratic chaos, signed threshold
traces, and VC entropy.
We also construct two-stage approximants by discretizing the outer measure
and the selected outer Brownian profiles, obtaining an
$M^{-1/2}+m^{-1/2}$
error bound, a sharp interpolation constant
$\sqrt{A/2}$,
and at most
$2M$
active outer-profile basis contributions per evaluation.
Controlled experiments illustrate the approximation mechanisms and indicate a
favorable limited-data accuracy--complexity trade-off.

\end{abstract}


\section{Introduction}
\label{sec:intro}

Any claim about the representational benefit of depth is relative to a
notion of complexity.
A function may be inexpensive when complexity is measured by width or
parameter count and expensive when it is measured by weight magnitude,
path variation, or an intrinsic function-space norm.
This distinction is especially important for overparameterized models,
where the size of one finite realization may reveal little about the class
of functions favored by learning.
Classical and modern capacity analyses therefore control neural function
classes through weight magnitudes and norm-based representation costs,
often independently of width
\citep{bartlett1996,neyshabur2015,golowich2018,barron2019}.
The corresponding function-space viewpoint associates a predictor with the
function it represents rather than with one particular parameterization and
asks which functions lie in a norm- or gauge-bounded class.

A rich theory has been developed for shallow networks through Barron spaces,
measure-generated variation spaces, ridge-spline representations, and
reproducing kernel Banach spaces
\citep{Barron1993,Bach2017,ongie2019,savarese2019,parhi2021,parhi2022b,ma2022,siegel2023,siegel2024,bartolucci2023}.
These constructions connect approximation, regularization, statistical
complexity, and representer theorems through intrinsic function-space
geometry.
Extending this perspective to depth has led to several distinct recursive
models.
Compositional and recursive RKHS theories build Hilbert-space hierarchies
through kernel or function composition
\citep{Chen2024,heeringa2025}.
Recursive RKBS and Banach-space constructions emphasize vector-valued
measures, layerwise variational structure, and representer theorems
\citep{wojtowytsch2020,bartolucci2024}.
Deep variation-space and neural-tree formulations instead iterate atomic,
measure, or composed shallow-space constructions
\citep{parhi2022a,shenouda2024,nakhleh2026deep,ongie2026}.

The resulting theories differ not only in their choice of activation or
kernel, but also in the location at which linear superposition is introduced.
This design choice determines the recursive object being studied.
If convexification or measure superposition is applied at every layer, then
nonlinear feature generation and linear combination evolve together.
If superposition is postponed, one may first study the geometry of individual
compositional paths and only afterward ask how much linear variation is
required to combine them.
These operations need not commute, and they can lead to different conclusions
about regularity, expressivity, statistical complexity, and finite
realizability.
This motivates the central question of the paper:
\emph{how does the placement of linear superposition within a recursive
function-space construction affect the geometry, learning complexity, and
constructive realization of the resulting deep function classes?}

We answer this question through the
\emph{Variation Brownian Kernel Ladder}
(VBKL).
The construction begins with linear projections and recursively builds a
nonlinear path dictionary.
At each subsequent level, one chooses a lower-level atom and composes it with
a unit-ball function from the Brownian RKHS\@.
Thus, a depth-$L$ atom contains one support path and exactly
$L-1$
normalized Brownian profiles.
No convex or variation hull is taken at the intermediate levels.
Only after the depth-$L$ dictionary has been constructed do we introduce
linear superposition, by integrating its atoms against finite signed
measures.
The infimal total variation over all such representations defines the
intrinsic VBKL complexity.
This order of construction separates the geometry of recursive atoms, the
cost of their outer linear combination, and the complexity of a finite
realization used for computation.

\paragraph{Why the Brownian RKHS?}
The Brownian kernel is not used merely as a convenient nonlinear activation.
Its RKHS is the explicit anchored Sobolev--Cameron--Martin space of absolutely
continuous functions with square-integrable derivative.
Brownian pullbacks admit exact RKHS descriptions, while the Brownian kernel
metric propagates square-root regularity through recursive composition.
The signed-threshold representation of the kernel connects finite Brownian
variation classes to quadratic chaos and VC entropy.
Finally, the one-dimensional profile geometry permits piecewise-linear
approximation with a sharp constant.
These properties allow analytical, statistical, and constructive theories to
be developed within the same recursive model.

\paragraph{Brownian and recursive-kernel lineage.}
The shallow Brownian projection model of
\citet{follain2025}
represents predictors as expectations of one-dimensional Sobolev functions
over learned projections and identifies the associated Brownian projection
kernel.
Brownian Kernel Ladders
\citep{MohammadigohariDiFattaNicosiaPardalos2026BKL}
recursively average Brownian pullback kernels to construct integral RKHS
hierarchies.
VBKL uses the same Brownian function-generation mechanism in a different
way: it follows individual compositional supports to construct a path
dictionary and places the signed-measure superposition only at the outermost
level.
The three constructions are therefore related through the Brownian kernel,
but differ in whether projection measures, recursive kernel measures, or
outer atomic measures are the primary representation variables.

\paragraph{Relation to Deep Neural Variation Spaces.}
The closest deep variation-space framework is that of
\citet{nakhleh2026deep}.
Schematically, its depth-$l$ unit ball and the VBKL depth-$l$ dictionary are
constructed as
\begin{align}
\mathcal B_l^{\mathrm{DNVS}}
&=
\overline{\operatorname{aconv}}
\left(
\left\{
\sigma_s\circ f:
s>0,
\; f\in\mathcal B_{l-1}^{\mathrm{DNVS}}
\right\}
\right),
\notag\\
\mU_l
&=
\left\{
g\circ u:
u\in\mU_{l-1},
\; g\in\mH_{\kB},
\; \left\|g\right\|_{\mH_{\kB}}\le1
\right\}.
\notag
\end{align}
Deep Neural Variation Spaces therefore alternate nonlinear activation and
closed absolute convexification at every level, whereas VBKL recursively
constructs only the nonlinear dictionary and takes its variation hull after
the desired depth is reached.
There is a second distinction: the former varies a normalized family derived
from one prescribed activation, while each VBKL profile ranges over the
entire unit ball of the Brownian RKHS\@.
Because nonlinear composition and absolute convexification do not generally
commute, VBKL is not obtained merely by specializing that framework to a
Brownian activation family, and we make no general inclusion or equivalence
claim between the resulting spaces.
The two theories also reveal complementary depth phenomena.
Deep Neural Variation Spaces establish depth saturation for norm-controlled
univariate ReLU classes, whereas the Brownian path construction below yields
a strict hierarchy under a local non-degeneracy condition along the support
of the input measure.
Their finite-representation results are likewise complementary: the former
proves a representer theorem for norm-penalized data fitting, while VBKL
gives a two-stage approximation of every element of its full
infinite-dimensional space.

\paragraph{Depth separation and approximation.}
Classical depth-separation results compare finite networks through width,
parameter count, or related representation costs
\citep{telgarsky2016,eldan2016,venturi2022,vardi2020,parkinson2024}.
Our strict-hierarchy result addresses a different question: whether adjacent
intrinsic, potentially infinite-width, norm-controlled function spaces
contain genuinely different functions.
The constructive theory is also connected to nonlinear, greedy, and
variable-basis approximation
\citep{DeVore1998,Temlyakov2011,kurkova2002,siegel2024}.
Here the two approximation resources have distinct mathematical meanings:
the number of atoms controls discretization of the outer signed measure,
whereas the profile resolution controls realization of each selected
Brownian nonlinearity.

\paragraph{Scope of the results.}
The full path-atomic VBKL space and the associated finite mixed architecture
play different roles.
Our analytical and constructive results concern the full
infinite-dimensional space generated by signed measures over the path
dictionary.
The statistical guarantees concern an explicitly parameterized finite
Brownian variation class whose lower-support architecture may include
intermediate linear mixing.
The path-only specialization of that architecture is a subclass of the full
VBKL space, but no such inclusion is asserted for the general mixed
architecture.
Stating this distinction explicitly prevents finite-parameter complexity
from being attributed to the unrestricted infinite-dimensional variation
ball.

The principal contributions are as follows.

\begin{enumerate}

\item
\textbf{Path-atomic recursive variation spaces.}
We introduce a recursive Brownian dictionary in which every new atom is the
composition of one lower-level atom with one normalized Brownian RKHS profile.
We identify each dictionary exactly as a union of Brownian pullback RKHS unit
balls and define the depth-$L$ VBKL space as its finite signed-measure
variation hull.
This gives an explicit infinite-dimensional representation while keeping
nonlinear recursion and linear superposition mathematically distinct.

\item
\textbf{Full-space geometry.}
We establish well-definedness and non-degeneracy of the variation complexity.
Under compactness of the input and direction sets, the recursive dictionary
is compact and the minimum-total-variation representation is attained.
Variation complexity controls pointwise magnitude and
$2^{-(L-1)}$-H\"older regularity. When the input measure has full support, the
resulting representative is unique and yields a continuous H\"older embedding.
Under a local non-degeneracy condition whose trace lies in the support of the
input measure, the spaces form a strict hierarchy across depth.

\item
\textbf{Architecture-dependent statistical guarantees.}
For explicit finite lower-support architectures with piecewise-linear
Brownian profiles, normalized mixing, and parameter count
$P_{L-1,m,G}$,
we derive empirical and expected Rademacher bounds and a high-probability
generalization guarantee.
The proof passes from the union-of-RKHS representation to Brownian quadratic
chaos, represents that chaos through signed threshold traces, and controls
the trace entropy through VC theory.
The bound separates outer variation radius, recursive Brownian range, finite
architecture size, and sample size.

\item
\textbf{Constructive approximation with separated resources.}
Every function in the full VBKL space is first approximated by at most
$M$
recursive atoms with error of order
$M^{-1/2}$.
The selected outer Brownian profiles are then replaced by piecewise-linear
interpolants with an additional
$m^{-1/2}$
error.
Balanced refinement yields an
$O\left(N^{-1/2}\right)$
approximation with finite outer atomic support and discretized outer profiles.
The interpolation constant
$\sqrt{A/2}$
is optimal, and evaluating the resulting model uses at most
$2M$
active outer-profile basis contributions independently of
$m$.

\item
\textbf{Theory-directed experiments.}
Controlled experiments isolate signed-measure discretization, profile
discretization, and balanced refinement, and the normalized tent construction
attains the worst-case interpolation bound exactly.
Supervised studies compare validation-selected VBKL realizations with Deep
Neural Variation Spaces and kernel baselines.
They indicate the strongest VBKL behavior in limited-data regimes and a
favorable accuracy--parameter trade-off rather than universal predictive
dominance.
Additional experiments verify the finite-difference directional estimator,
the variance reduction from Monte Carlo directional averaging, and the
practical realizability of the finite models.

\end{enumerate}

The remainder of the paper is organized as follows.
Section~\ref{sec:notations}
introduces the recursive Brownian dictionary, the associated VBKL space, and
the finite lower-support architectures.
Section~\ref{sec:analytical-properties}
establishes analytical properties of the full VBKL spaces and
architecture-dependent statistical guarantees for the associated finite
Brownian variation classes.
Section~\ref{sec:constructive-approximation}
develops the two-stage constructive approximation theory.
Section~\ref{sec:experiments}
presents theory-directed experiments on approximation, statistical learning,
parameter efficiency, and optimization.
Complete protocols, numerical tables, selected configurations, and
computational measurements are reported in
Appendix~\ref{app:experimental-details}.
Proofs of the main theoretical results are collected in
Section~\ref{app}, while additional notation, the sharpness analysis of the
Brownian profile interpolation estimate, and the supporting auxiliary lemmas
are provided in Appendices~\ref{appendix:notation},
\ref{sec:sharpness-brownian-interpolation}, and
\ref{app:auxi_lems}, respectively.

\begin{table}[t]
\centering
\begin{tabular}{@{}lll@{}}
\toprule
Result & Content & Page \\
\midrule

Theorem~\ref{thm:analytical-properties-variation-bkl}
&
Analytical properties of the VBKL spaces
&
page~\pageref{thm:analytical-properties-variation-bkl}
\\

Theorem~\ref{thm:variation-bkl-main-generalization}
&
Architecture-dependent Rademacher bound
&
page~\pageref{thm:variation-bkl-main-generalization}
\\

Corollary~\ref{cor:vbkl-generalization}
&
Finite-architecture generalization guarantee
&
page~\pageref{cor:vbkl-generalization}
\\
Theorem~\ref{thm:path-bkl-approx}
&
Finite-atomic approximation of the full VBKL space
&
page~\pageref{thm:path-bkl-approx}
\\

Theorem~\ref{thm:two-stage-path-profile-bkl}
&
Two-stage finite-atomic and outer-profile approximation
&
page~\pageref{thm:two-stage-path-profile-bkl}
\\

Proposition~\ref{thm:profile-bkl-approximation}
&
Brownian profile interpolation
&
page~\pageref{thm:profile-bkl-approximation}
\\

\bottomrule
\end{tabular}

\caption{Overview of the main theoretical results. Their logical dependencies are illustrated in Fig.~\ref{fig:dependency_graph}; supporting results are summarized in Table~\ref{tab:auxiliary-results}.}
\label{tab:main-results}
\end{table}

\section{Notation}
\label{sec:notations}
We use the following notation throughout the paper.
The set of natural numbers is denoted by
$\N=\{1,2,\ldots\}$,
$\R$ denotes the set of real numbers, and
$[n]=\{1,\ldots,n\}$.
Throughout the paper,
$\mX\subseteq\R^d$
denotes a compact input domain equipped with a probability measure $\nu$, and
$\Omega\subseteq\S^{d-1}$
denotes the admissible set of first-layer directions. The Brownian kernel is
$
\kB(x,x')
=
(|x|+|x'|-|x-x'|)/2,
$
whose associated RKHS is
$
\mH_{\kB}
=
\{
g:
g(0)=0,\;
g
\text{ absolutely continuous},\;
g'\in L^2(\R)
\},
$
equipped with the norm
$
\|g\|_{\mH_{\kB}}^2
=
\int_{\R}|g'(t)|^2\,dt.
$
For a measurable space
$\Theta$,
we write
$\mathcal M(\Theta)$
for the space of finite signed measures equipped with the total variation norm
$
\|\mu\|_{\mathrm{TV}}
=
|\mu|(\Theta).
$
Additional notation is collected in Appendix~\ref{appendix:notation}.

\subsection{Atomic Brownian Kernel Ladder}
\label{subsec:atomic-bkl}

The depth-$L$ VBKL space is generated by a recursively constructed
Brownian dictionary. In BKL
\citep{MohammadigohariDiFattaNicosiaPardalos2026BKL}, a recursive level is
formed by averaging Brownian pullback kernels over a probability measure on
lower-level supports: if $\mS$ is a support class and
$\mu\in\mP_1(\mS)$, then
\begin{align}
k[\mS,\mu]\left(\b x,\b x'\right)
&=
\int_{\mS}
\kB\left(u(\b x),u(\b x')\right)
\,\d\mu(u),
\qquad
\b x,\b x'\in\mX.
\label{eq:atomic-bkl-integral-kernel}
\end{align}
VBKL uses the corresponding Dirac specialization. For a single support
$u$, the recursive kernel is
\begin{align}\label{eq:pullbackk}
k_u\left(\b x,\b x'\right)
&=
\kB\left(u(\b x),u(\b x')\right),
\qquad
\b x,\b x'\in\mX.
\end{align}
Thus each recursive step propagates one support rather than averaging over a
support family. The recursion below constructs only the nonlinear dictionary;
the signed-measure variation hull is introduced after the target depth is
reached.

\paragraph{First layer.}
Let $\Omega\subseteq\S^{d-1}$ denote the admissible set of directions. The first
recursive generator class consists of linear projections,
$ \mU_1 : \allowbreak= \left\{ \b x\mapsto\bomega^\top\b x : \bomega\in\Omega \right\}. $

\paragraph{Atomic BKL construction.}

For each
$l\in\left\{2,\ldots,L\right\}$,
assume that the lower-level dictionary
$\mU_{l-1}$
has been defined.
The next dictionary is obtained by composing a normalized Brownian
profile with a lower-level atom.
The exact pullback-RKHS identification is established in
\Cref{lem:vbkl-union-rkhs}, and gives
\begin{align}
\mU_l
&:=
\left\{
\b x
\mapsto
g\left(
u\left(
\b x
\right)
\right)
:
u\in\mU_{l-1},
\;
g\in\mH_{\kB},
\;
\left\|g\right\|_{\mH_{\kB}}
\le
1
\right\}
\notag\\
&=
\bigcup_{u\in\mU_{l-1}}
\left\{
f\in\mH_{k_u}
:
\left\|f\right\|_{\mH_{k_u}}
\le
1
\right\}.
\end{align}
Consequently, every
$f\in\mU_L$
admits a representation
$ f\left( \b x \right) = \left( g_{L-1} \circ \cdots \circ g_1 \right) \left( \bomega^\top\b x \right), $\allowbreak
$ \bomega\in\Omega$, $\left\|g_j\right\|_{\mH_{\kB}} \le 1$, $j\in\left[L-1\right]. $
Thus, a depth-$L$ atom consists of one linear projection followed by
exactly
$L-1$
normalized Brownian profiles.
The collection of all such atoms forms the recursive Brownian
dictionary
$\mU_L$
used in the variation-space construction below.

\begin{remark}
\label{rem:atomic-bkl-canonical-supports}
Let
$u:\mX\rightarrow\mathbb R$
be a support function and let
$c\in\mathbb R\setminus\left\{0\right\}$.
The positive homogeneity of the Brownian kernel gives
\begin{align}
k_{cu}\left(
\b x,\b x'
\right)
&\stackrel{(a)}{=}
\kB\left(
cu\left(\b x\right),
cu\left(\b x'\right)
\right)
\stackrel{(b)}{=}
\left|c\right|
\kB\left(
u\left(\b x\right),
u\left(\b x'\right)
\right)
 \stackrel{(c)}{=}
\left|c\right|
k_u\left(
\b x,\b x'
\right),
\qquad
\b x,\b x'\in\mX.
\end{align}
Here,
(a) follows from the definition of the pullback kernel,
(b) uses the positive homogeneity of
$\kB$,
and
(c) applies the definition of
$k_u$. Since multiplying a positive-definite kernel by a positive scalar does
not change its RKHS as a set, but rescales its norm, one obtains
\begin{align}
\mH_{k_{cu}}
&=
\mH_{k_u},
\\
\left\|f\right\|_{\mH_{k_{cu}}}
&=
\left|c\right|^{-1/2}
\left\|f\right\|_{\mH_{k_u}},
\qquad
f\in\mH_{k_u}.
\end{align}
Therefore, scaling a support function preserves the underlying RKHS as
a set but does not preserve its unit ball.
Accordingly, intermediate support functions are not identified with
their normalized versions unless such a normalization is imposed
explicitly as part of a separate finite architecture.
\end{remark}

\paragraph{VBKL space.}

Each $\mU_l$ is viewed as a subspace of $C(\mX)$ and is equipped with the
Borel $\sigma$-algebra induced by the supremum norm. The completed dictionary
$\mU_L$ generates the depth-$L$ VBKL space
\begin{align}
\VL
:=
\left\{
F\in L^2(\nu)
:
F
=
\int_{\mU_L}
u\,\d\mu(u),
\;
\mu\in\mathcal M(\mU_L),
\;
\|\mu\|_{\mathrm{TV}}<\infty
\right\}.
\label{eq:variation-bkl-space}
\end{align}
Here $\mathcal M(\mU_L)$ is the space of finite signed Borel measures on
$\mU_L$, and the integral is a Bochner integral in $L^2(\nu)$. Every
absolutely summable atomic expansion corresponds to a discrete measure.
When $\mU_L$ is compact in $L^2(\nu)$, the canonical map
$u\mapsto u$ is continuous, and a uniformly
$\|\cdot\|_{\mathrm{TV}}$-bounded sequence of representing measures has a
weak-* convergent subsequence in
$\mathcal M(\mU_L)=C(\mU_L)^*$. Under these conditions,
\Cref{prop:variation-gauge-attainment} shows that the limiting measure retains
the barycentric representation. Unlike BKL, which
averages pullback kernels recursively, VBKL fixes the completed path
dictionary and applies signed-measure superposition only at the outer level.

\paragraph{Variation complexity.}

For $F\in\VL$, define the intrinsic variation complexity as the infimal total
variation of a representing measure:
\begin{align}
\CLv(F)
:=
\inf
\left\{
\|\mu\|_{\mathrm{TV}}
:
F
=
\int_{\mU_L}
u\,\d\mu(u)
\right\}.
\label{eq:variation-bkl-complexity}
\end{align}
The infimum is over finite signed Borel measures
$\mu\in\mathcal M(\mU_L)$ representing $F$ through
\eqref{eq:variation-bkl-space}. This formulation is convenient for both the
statistical analysis and the approximation results, where finite-support
measures arise as approximants of general representations.

\paragraph{Finite-support VBKL class.}

For $m\ge1$, let
\begin{align}
\mV_m^{(L)}
:=
\left\{
F\in\VL
:
F
=
\int_{\mU_L}
u\,\d\mu(u),
\;
\left|\operatorname{supp}(\mu)\right|
\le
m
\right\},
\end{align}
where $\mu$ is a finite signed Borel measure on $\mU_L$. Equivalently,
$\mV_m^{(L)}$ contains precisely the functions admitting a representation by
at most $m$ atoms from $\mU_L$.

\begin{remark}
\label{rem:variation-bkl-intuition}
The construction has three complementary interpretations.
\begin{enumerate}
\item
\textnormal{Well-definedness and non-degeneracy.}
By \Cref{lem:path-atomic-space-well-defined}, every finite signed measure of
finite total variation defines an element of $L^2(\nu)$, so $\VL$ is a
well-defined linear subspace. Moreover,
\Cref{cor:path-gauge-nondegenerate} gives
\begin{align}
\|F\|_{L^2(\nu)}
\le
R_L\CLv(F),
\qquad
F\in\VL,
\end{align}
where $R_L$ is defined in \eqref{eq:atomic-generator-l2-radius}; hence $\CLv(F)=0$ implies $F=0$ in $L^2(\nu)$.

\item
\textnormal{Variation-space viewpoint.}
The pair $(\VL,\CLv)$ is the measure-generated variation space associated
with the recursive dictionary
\citep{Barron1993,Bach2017,parhi2021,siegel2023,nakhleh2026deep}. Recursion
constructs the nonlinear atoms; only the final signed measure introduces
linear superposition.

\item
\textnormal{Relation to BKL.}
Each $\mU_l$ is a union of Brownian pullback RKHS unit balls. VBKL therefore
retains the recursive Brownian geometry of BKL while replacing layerwise
kernel averaging by an outer variation hull over individual compositional
paths.
\end{enumerate}
\end{remark}

\subsection{Finite Lower-Support Architectures}
\label{subsec:finite-parametric-bkl}

The recursive dictionaries introduced in
\Cref{subsec:atomic-bkl}
are infinite-dimensional because their Brownian profiles belong to
$\mH_{\kB}$.
For the architecture-dependent statistical analysis, we therefore
introduce an explicit finite-parametric class of lower-level support
functions.
This class is defined separately from the infinite-dimensional dictionary
$\mU_{L-1}$;
the relationship between the two classes is stated at the end of this
subsection. Fix
$m\ge1$
and
$G\ge2$.
Set
\begin{align}
R_{\mX}
&:=
\sup_{\b x\in\mX}
\left\|
\b x
\right\|_2,
&
A_{\mX}
&:=
\max
\left\{
1,
R_{\mX}
\right\}.
\label{eq:finite-architecture-input-radius}
\end{align}
Since
$\mX$
is compact,
one has
$A_{\mX}<\infty$.

\paragraph{Finite Brownian profiles.}

Let
\begin{align}
-A_{\mX}
=
t_0
<
t_1
<
\cdots
<
t_{j_0}
=
0
<
\cdots
<
t_G
=
A_{\mX}
\label{eq:finite-profile-grid}
\end{align}
be a fixed interpolation grid containing the origin, and let
$\psi_0,\ldots,\psi_G$
denote the associated continuous piecewise-linear hat functions on
$\left[-A_{\mX},A_{\mX}\right]$.
For
$g\in\mH_{\kB}$,
define
\begin{align}
\left(
\Pi_Gg
\right)
\left(
t
\right)
:=
\sum_{j=0}^{G}
g\left(
t_j
\right)
\psi_j\left(
t
\right),
\qquad
t\in
\left[
-A_{\mX},
A_{\mX}
\right].
\label{eq:finite-profile-interpolant}
\end{align}
Equivalently, for every
$j\in\left\{0,\ldots,G-1\right\}$
and every
$t\in\left[t_j,t_{j+1}\right]$,
\begin{align}
\left(
\Pi_Gg
\right)
\left(
t
\right)
&=
\frac{
t_{j+1}-t
}{
t_{j+1}-t_j
}
g\left(
t_j
\right)
+
\frac{
t-t_j
}{
t_{j+1}-t_j
}
g\left(
t_{j+1}
\right).
\label{eq:finite-profile-interpolant-piecewise}
\end{align}
We extend
$\Pi_Gg$
constantly outside the interpolation interval by setting
\begin{align}
\left(
\Pi_Gg
\right)
\left(
t
\right)
&:=
\left(
\Pi_Gg
\right)
\left(
-A_{\mX}
\right),
\qquad
t<
-A_{\mX},
\label{eq:finite-profile-left-extension}
\\
\left(
\Pi_Gg
\right)
\left(
t
\right)
&:=
\left(
\Pi_Gg
\right)
\left(
A_{\mX}
\right),
\qquad
t>
A_{\mX}.
\label{eq:finite-profile-right-extension}
\end{align}
Define the finite Brownian profile class by
\begin{align}
\mathcal Q_{G,A_{\mX}}
:=
\left\{
\Pi_Gg
:
g\in\mH_{\kB},
\;
\left\|
g
\right\|_{\mH_{\kB}}
\le
1
\right\}.
\label{eq:finite-brownian-profile-class}
\end{align}
Every
$q\in\mathcal Q_{G,A_{\mX}}$
is determined by its nodal coefficient vector $\b c_q
:=
\left(
q\left(
t_0
\right),
\ldots,
q\left(
t_G
\right)
\right)
\in
\mathbb R^{G+1}$. Moreover, the nodal coefficients satisfy $q\left(
0
\right)
=
0$, 
\begin{align}
\left\|
q
\right\|_{\mH_{\kB}}^2
&=
\sum_{j=0}^{G-1}
\frac{
\left(
q\left(
t_{j+1}
\right)
-
q\left(
t_j
\right)
\right)^2
}{
t_{j+1}-t_j
}
\le
1.
\label{eq:finite-profile-discrete-energy}
\end{align}
The exact correspondence between
\eqref{eq:finite-brownian-profile-class}
and
\eqref{eq:finite-profile-discrete-energy}
is established in the appendix.

\paragraph{Admissible mixing coefficients.}

Let
\begin{align}
\mathfrak M_m
:=
\left\{
\b W\in\mathbb R^{m\times m}
:
\max_{j\in\left[m\right]}
\sum_{r=1}^{m}
\left|
W_{jr}
\right|
\le
1
\right\},
\label{eq:finite-mixing-matrix-class}
\end{align}
and let
\begin{align}
\mathfrak B_m
:=
\left\{
\b \beta\in\mathbb R^m
:
\sum_{j=1}^{m}
\left|
\beta_j
\right|
\le
1
\right\}.
\label{eq:finite-readout-class}
\end{align}
The row-sum constraint in
\eqref{eq:finite-mixing-matrix-class}
and the readout constraint in
\eqref{eq:finite-readout-class}
ensure that the intermediate supports remain uniformly bounded.

\paragraph{Finite lower-support architecture.}

For
$L=2$,
define
\begin{align}
\mathcal A_{1}^{m,G}
:=
\mU_1.
\label{eq:finite-lower-support-class-depth-one}
\end{align}
Suppose now that
$L\ge3$,
and set $K
:=
L-2$. An admissible parameter tuple is of the form
\begin{align}
\btheta
:=
\left(
\left(
\bomega_j
\right)_{j\in\left[m\right]},
\left(
q_j^{(r)}
\right)_{
r\in\left[K\right],
\,
j\in\left[m\right]
},
\left(
\b W^{(r)}
\right)_{r=2}^{K},
\b \beta
\right),
\label{eq:finite-lower-support-parameter-tuple}
\end{align}
where
\begin{align}
\bomega_j
&\in
\Omega,
\qquad
j\in\left[m\right],
\label{eq:finite-lower-support-direction-constraint}
\\
q_j^{(r)}
&\in
\mathcal Q_{G,A_{\mX}},
\qquad
r\in\left[K\right],
\quad
j\in\left[m\right],
\label{eq:finite-lower-support-profile-constraint}
\\
\b W^{(r)}
&\in
\mathfrak M_m,
\qquad
r\in
\left\{
2,\ldots,K
\right\},
\label{eq:finite-lower-support-mixing-constraint}
\\
\b \beta
&\in
\mathfrak B_m.
\label{eq:finite-lower-support-readout-constraint}
\end{align}
When
$K=1$,
the family of mixing matrices in
\eqref{eq:finite-lower-support-parameter-tuple}
is empty. For every
$j\in\left[m\right]$,
define the first linear coordinates by
\begin{align}
z_j^{(0)}
\left(
\b x
\right)
:=
\bomega_j^\top\b x,
\qquad
\b x\in\mX.
\label{eq:finite-lower-support-linear-layer}
\end{align}
The first Brownian profile layer is
\begin{align}
z_j^{(1)}
\left(
\b x
\right)
:=
q_j^{(1)}
\left(
z_j^{(0)}
\left(
\b x
\right)
\right),
\qquad
j\in\left[m\right].
\label{eq:finite-lower-support-first-profile-layer}
\end{align}
For every
$r\in\left\{2,\ldots,K\right\}$
and
$j\in\left[m\right]$,
define recursively
\begin{align}
s_j^{(r)}
\left(
\b x
\right)
&:=
\sum_{k=1}^{m}
W_{jk}^{(r)}
z_k^{(r-1)}
\left(
\b x
\right),
\label{eq:finite-lower-support-preactivation}
\\
z_j^{(r)}
\left(
\b x
\right)
&:=
q_j^{(r)}
\left(
s_j^{(r)}
\left(
\b x
\right)
\right).
\label{eq:finite-lower-support-hidden-layer}
\end{align}
The scalar lower-support output associated with
$\btheta$
is
\begin{align}
a_{\btheta}
\left(
\b x
\right)
:=
\sum_{j=1}^{m}
\beta_j
z_j^{(K)}
\left(
\b x
\right),
\qquad
\b x\in\mX.
\label{eq:finite-lower-support-output}
\end{align}
The resulting depth-$\left(L-1\right)$ finite lower-support class is
\begin{align}
\mathcal A_{L-1}^{m,G}
&:=
\left\{
a_{\btheta}
:
\btheta \text{ satisfies }
\eqref{eq:finite-lower-support-direction-constraint}
\text{--}
\eqref{eq:finite-lower-support-readout-constraint}
\right\}.
\end{align}
By
\Cref{lem:finite-bkl-parameter-count}
\ref{lem:finite-architecture-range-control},
the normalization constraints yield the sharper uniform range estimate
\begin{align}
\sup_{a\in\mathcal A_{L-1}^{m,G}}
\sup_{\b x\in\mX}
\left|
a\left(
\b x
\right)
\right|
\le
A_{\mX}^{\,2^{-(L-2)}}.
\label{eq:finite-lower-support-uniform-range}
\end{align}
By
\Cref{lem:finite-bkl-parameter-count}
\ref{lem:finite-architecture-parameter-count},
for
$L\ge3$,
every element of
$\mathcal A_{L-1}^{m,G}$
is represented by at most
\begin{align}
P_{L-1,m,G}
&:=
md
+
\left(
L-2
\right)
m
\left(
G+1
\right)
+
\left(
L-3
\right)
m^2
+
m
\label{eq:finite-bkl-explicit-parameter-count-main}
\end{align}
real parameters.
The four terms in
\eqref{eq:finite-bkl-explicit-parameter-count-main}
correspond, respectively, to the first-layer directions, the nodal
profile coefficients, the internal mixing matrices, and the final
readout coefficients.
For
$L=2$,
we set
\begin{align}
P_{1,m,G}
:=
d.
\end{align}

\paragraph{Associated outer Brownian dictionary.}

For every
$a\in\mathcal A_{L-1}^{m,G}$,
define
\begin{align}
k_a
\left(
\b x,
\b x'
\right)
:=
\kB
\left(
a\left(
\b x
\right),
a\left(
\b x'
\right)
\right),
\qquad
\b x,\b x'\in\mX.
\label{eq:finite-architecture-pullback-kernel}
\end{align}
The associated outer Brownian dictionary is
\begin{align}
\mathcal D_L^{m,G}
:=
\bigcup_{a\in\mathcal A_{L-1}^{m,G}}
\left\{
f\in\mH_{k_a}
:
\left\|
f
\right\|_{\mH_{k_a}}
\le
1
\right\}.
\label{eq:finite-architecture-outer-dictionary}
\end{align}
Finally, for
$R>0$,
define the corresponding finite-architecture Brownian variation ball by
\begin{align}
\mathcal W_R^{L;m,G}
:=
\left\{
F_\mu
:
F_\mu
\left(
\b x
\right)
=
\int_{\mathcal D_L^{m,G}}
f\left(
\b x
\right)
\,\d\mu\left(
f
\right),
\;
\mu\in
\mathcal M
\left(
\mathcal D_L^{m,G}
\right),
\;
\left\|
\mu
\right\|_{\mathrm{TV}}
\le
R
\right\}.
\label{eq:finite-architecture-variation-ball}
\end{align}
By
\eqref{eq:finite-lower-support-uniform-range}
and the reproducing property,
every
$f\in\mathcal D_L^{m,G}$
satisfies
$ \left| f\left( \b x \right) \right| \le A_{\mX}^{1/2}, \qquad \b x\in\mX. $
Consequently, the measure representation in
\eqref{eq:finite-architecture-variation-ball}
is well-defined pointwise and in
$L^2\left(\nu\right)$.

\begin{remark}\textnormal{(Relation to the infinite-dimensional VBKL dictionary)}\ 
\label{rem:finite-architecture-relation-to-vbkl}
The class
$\mathcal A_{L-1}^{m,G}$
is an explicit finite lower-support architecture.
Because
\eqref{eq:finite-lower-support-preactivation}
allows nontrivial linear mixing of lower-level coordinates and
\eqref{eq:finite-lower-support-output}
allows a linear readout,
one does not in general have $\mathcal A_{L-1}^{m,G}
\subseteq
\mU_{L-1}$. Accordingly,
$\mathcal W_R^{L;m,G}$
is treated as an associated finite-architecture Brownian variation
class and is not identified with the full infinite-dimensional ball
$\mV_R^{(L)}$. In the path-only specialization, in which every mixing row selects a
single preceding coordinate and the readout selects a single final
coordinate, every resulting support is a recursively composed
Brownian atom.
For that specialization, $\mathcal A_{L-1,\mathrm{path}}^{m,G}
\subseteq
\mU_{L-1}$,
$\mathcal D_{L,\mathrm{path}}^{m,G}
\subseteq
\mU_L$, and $\mathcal W_{R,\mathrm{path}}^{L;m,G}
\subseteq
\mV_R^{(L)}$.

\end{remark}

\section{Analytical and Statistical Properties of the VBKL spaces}
\label{sec:analytical-properties}

This section develops two complementary parts of the theory.
We first establish the fundamental analytical properties of the full
infinite-dimensional VBKL spaces, showing that variation complexity
controls regularity, pointwise evaluation, and expressive power across
recursive depth.
We then study the statistical complexity of the associated
finite-architecture Brownian variation classes introduced in
\Cref{subsec:finite-parametric-bkl}.
The resulting Rademacher bound explicitly separates the outer
variation radius from the parameter complexity of the finite
lower-support architecture.

\subsection{Analytical Properties of VBKL spaces}
\label{subsec:analytical-properties}

We first ask whether the outer variation complexity controls regularity and pointwise evaluation, and whether the recursive path construction produces a genuine depth hierarchy. The following theorem answers these questions for the full VBKL spaces.

\begin{theorem}\textnormal{(Analytical properties of the VBKL spaces)}
\label{thm:analytical-properties-variation-bkl}
Let $(\mX,\nu)$, $\Omega$, and the recursive dictionaries
$(\mU_l)_{l\ge1}$ be as introduced in \Cref{subsec:atomic-bkl}. Assume that
$\mX\subseteq\R^d$ is compact, and let $L\ge2$. Let $\VL$ and $\CLv$ denote
the depth-$L$ VBKL space and variation complexity defined in
\eqref{eq:variation-bkl-space} and \eqref{eq:variation-bkl-complexity}. Set
\begin{align}
\alpha_L
&:=
2^{-(L-1)},
&
R_{\mX}
&:=
\sup_{\b x\in\mX}\|\b x\|_2.
\end{align}
Then the following statements hold.

\begin{enumerate}[(i)]
\item
\label{thm:variation-holder-pointwise-bound}
\textnormal{H\"older representative and pointwise control.}
Every $F\in\VL$ admits a representative, still denoted by $F$, satisfying
\begin{align}
\left|F(\b x)-F(\b x')\right|
&\le
\|\b x-\b x'\|_2^{\,\alpha_L}\CLv(F),
\qquad
\b x,\b x'\in\mX,
\label{eq:variation-holder-bound}
\\
\left|F(\b x)\right|
&\le
\|\b x\|_2^{\,\alpha_L}\CLv(F),
\qquad
\b x\in\mX.
\label{eq:variation-pointwise-bound}
\end{align}

\item
\label{thm:continuous-embedding}
\textnormal{H\"older-space control.}
Every representative supplied by part~\ref{thm:variation-holder-pointwise-bound}
belongs to $C^{0,\alpha_L}(\mX)$ and satisfies
\begin{align}
\|F\|_{C^{0,\alpha_L}(\mX)}
\le
\left(R_{\mX}^{\,\alpha_L}+1\right)\CLv(F).
\label{eq:variation-holder-embedding-bound}
\end{align}
If $\operatorname{supp}(\nu)=\mX$, this representative is unique, and the
representative map is a continuous linear injection
\begin{align}
\VL
\hookrightarrow
C^{0,\alpha_L}(\mX).
\end{align}

\item
\label{thm:depth-monotonicity}
\textnormal{Depth monotonicity and strictness.}
The consecutive spaces satisfy
\begin{align}
\VL
\subseteq
\mV^{(L+1)}.
\label{eq:vbkl-depth-monotonicity}
\end{align}
Suppose, in addition, that there exist $\b x_0\in\mX$, $\rho>0$, constants
$0<c_1\le C_1<\infty$, a Lipschitz map
$\gamma:[0,\rho]\to\mX$, and a first-layer generator $u_1\in\mU_1$ such
that
\begin{align}
\gamma(0)
&=
\b x_0,
&
\gamma([0,\rho])
&\subseteq
\operatorname{supp}(\nu),
\label{eq:vbkl-strictness-trace-support-proof}
\\
u_1(\b x_0)
&=
0,
&
c_1t
&\le
u_1(\gamma(t))
\le
C_1t,
\qquad
t\in[0,\rho].
\label{eq:vbkl-local-nondegeneracy-assumption}
\end{align}
Then the inclusion in \eqref{eq:vbkl-depth-monotonicity} is strict:
\begin{align}
\VL
\subsetneq
\mV^{(L+1)}.
\label{eq:vbkl-strict-depth-monotonicity}
\end{align}
\end{enumerate}
\end{theorem}

The theorem supplies the regularity and depth properties used by the
statistical and approximation results below.

\begin{remark}
\label{rem:analytical-properties}
Theorem~\ref{thm:analytical-properties-variation-bkl} gives quantitative
H\"older representatives for all VBKL elements, in the spirit of classical RKHS
regularity estimates \citep{Aronszajn1950}. Full support of $\nu$ makes
that representative unique and hence yields the stated embedding into, not
equality with, the H\"older space. The strict-depth conclusion is a separate
claim: it requires both the local non-degeneracy condition and that the
corresponding trace lie in $\operatorname{supp}(\nu)$.
\end{remark}

\subsection{Statistical Properties of Finite-Architecture Brownian Variation Classes}
\label{subsec:statistical-properties}

The analytical results above concern the full infinite-dimensional
VBKL space.
The statistical analysis in this subsection instead concerns the
associated finite-architecture class
$\mathcal W_R^{L;m,G}$
defined in
\eqref{eq:finite-architecture-variation-ball}.
Its lower-level supports belong to the finite-parametric architecture
$\mathcal A_{L-1}^{m,G}$,
while its outer atoms range over the corresponding union of Brownian
pullback RKHS unit balls
$\mathcal D_L^{m,G}$. The proof proceeds by reducing the variation class to its outer
Brownian dictionary, controlling the resulting union of RKHS unit
balls through Brownian quadratic chaos, and estimating that chaos by
the signed threshold entropy of the finite lower-support architecture.
All supporting arguments and the complete proof are deferred to the
appendix.

\subsubsection{Statistical Learning Framework}
\label{subsec:statistical-learning-framework}

Let $(X,Y)\sim\nu_{XY}$ be a random input--output pair on
$\mX\times\mY$, and let
\begin{align}
\mathcal D_n
:=
\left\{
(\b x_i,y_i)
\right\}_{i=1}^{n}
\end{align}
be an independent sample drawn from $\nu_{XY}$. We denote the input marginal
by $\nu_X$ and identify it with the measure $\nu$ introduced in
\Cref{sec:notations}. For a measurable predictor
$F:\mX\rightarrow\mathbb R$ and a loss
$\ell:\mathbb R\times\mY\rightarrow\mathbb R$, define
\begin{align}
\mathcal R(F)
&:=
\mathbb E_{(X,Y)\sim\nu_{XY}}
\left[
\ell(F(X),Y)
\right],
&
\mathcal R_n(F)
&:=
\frac1n
\sum_{i=1}^{n}
\ell(F(\b x_i),y_i).
\end{align}
We work with bounded variation-complexity classes, which play the role of
norm balls in classical RKHS learning.

\paragraph{Finite-architecture hypothesis classes.}

Fix
$L\ge2$,
$m\ge1$,
$G\ge2$,
and
$R>0$.
The hypothesis class considered in this subsection is the
finite-architecture Brownian variation ball
$\mathcal W_R^{L;m,G}$
defined in
\eqref{eq:finite-architecture-variation-ball}.
The radius
$R$
controls the total variation of the outer signed-measure
representation, whereas
$P_{L-1,m,G}$
controls the number of real parameters in the finite lower-support
architecture.

\paragraph{Empirical and expected Rademacher complexities.}

Let
$\mathcal F$
be a class of real-valued functions on
$\mX$,
and let
$\varepsilon_1,\ldots,\varepsilon_n$
be independent Rademacher random variables.
For fixed sample points
$\b x_1,\ldots,\b x_n\in\mX$,
define
\begin{align}
\widehat{\mathfrak R}_n
\left(
\mathcal F
\right)
:=
\mathbb E_{\varepsilon}
\left[
\sup_{f\in\mathcal F}
\frac{1}{n}
\sum_{i=1}^{n}
\varepsilon_i
f\left(
\b x_i
\right)
\right].
\end{align}
Averaging over independent sample points
$X_1,\ldots,X_n\sim\nu$
gives the expected Rademacher complexity
\begin{align}
\mathfrak R_n
\left(
\mathcal F
\right)
:=
\mathbb E_{\b X}
\left[
\widehat{\mathfrak R}_n
\left(
\mathcal F
\right)
\right].
\label{eq:expected-rademacher-complexity}
\end{align}

\begin{theorem}
\textnormal{
(Architecture-dependent Rademacher bound for finite Brownian variation
classes)
}
\label{thm:variation-bkl-main-generalization}
Let
$L,m,G,n\in\mathbb N$
satisfy
$L\ge2$,
$m\ge1$,
$G\ge2$,
and
$n\ge1$,
and let
$R>0$.
Let
$\mathcal A_{L-1}^{m,G}$,
$\mathcal D_L^{m,G}$,
and
$\mathcal W_R^{L;m,G}$
denote, respectively, the finite lower-support class, its associated
outer Brownian dictionary, and the radius-$R$ finite-architecture
Brownian variation ball introduced in
\Cref{subsec:finite-parametric-bkl}. Let
$P_{L-1,m,G}$
denote the parameter-count bound associated with
$\mathcal A_{L-1}^{m,G}$,
as specified in
\Cref{subsec:finite-parametric-bkl}
and established in
\Cref{lem:finite-bkl-parameter-count},
part~\ref{lem:finite-architecture-parameter-count}.
For notational convenience, set
\begin{align}
P
&:=
P_{L-1,m,G},
\label{eq:main-statistical-parameter-count}
\\
\Gamma_{L-1,m,G,n}
&:=
1
+
\left(
P+1
\right)
\ln\left(
e\left(
P+1
\right)
\right)
\ln\left(
en
\right).
\label{eq:finite-architecture-complexity-factor}
\end{align}
For fixed sample points
$\b x_1,\ldots,\b x_n\in\mX$,
define the sample-dependent lower-support envelope
\begin{align}
B_{\b x}
:=
\sup_{a\in\mathcal A_{L-1}^{m,G}}
\max_{i\in\left[n\right]}
\left|
a\left(
\b x_i
\right)
\right|.
\label{eq:main-finite-support-envelope}
\end{align}
Then there exists a constant
$C_L>0$,
depending only on the recursion depth
$L$,
such that
\begin{align}
\widehat{\mathfrak R}_n
\left(
\mathcal W_R^{L;m,G}
\right)
\le
R
B_{\b x}^{1/2}
\min
\left\{
1,
\;
C_L
\left(
\frac{
\Gamma_{L-1,m,G,n}
}{
n
}
\right)^{1/2}
\right\}.
\label{eq:main-empirical-rademacher-bound}
\end{align}
Moreover, the deterministic range estimate
\eqref{eq:finite-lower-support-uniform-range},
established in
\Cref{lem:finite-bkl-parameter-count},
part~\ref{lem:finite-architecture-range-control},
gives
\begin{align}
B_{\b x}
\le
A_{\mX}^{\,2^{-(L-2)}}.
\label{eq:main-finite-support-deterministic-envelope}
\end{align}
Consequently, if
$X_1,\ldots,X_n$
are independent random variables with common distribution
$\nu$,
then the expected Rademacher complexity satisfies
\begin{align}
\mathfrak R_n
\left(
\mathcal W_R^{L;m,G}
\right)
\le
R
A_{\mX}^{\,2^{-(L-1)}}
\min
\left\{
1,
\;
C_L
\left(
\frac{
\Gamma_{L-1,m,G,n}
}{
n
}
\right)^{1/2}
\right\}.
\label{eq:main-rademacher-bound}
\end{align}

\end{theorem}

The preceding result separates three sources of statistical
complexity.
The factor
$R$
is the radius of the outer variation representation,
the exponent
$A_{\mX}^{\,2^{-(L-1)}}$
is induced by the recursive Brownian geometry, and
$P_{L-1,m,G}$
is the parameter complexity of the finite lower-support architecture.
The additional real parameter in
$P_{L-1,m,G}+1$
corresponds to the variable threshold used in the threshold-class
analysis.

\begin{remark}
\label{rem:finite-architecture-statistical-scope}
The bound in
\Cref{thm:variation-bkl-main-generalization}
applies to
$\mathcal W_R^{L;m,G}$,
not to the full infinite-dimensional variation ball
$\mV_R^{(L)}$.
This distinction is necessary because the general finite architecture
allows intermediate linear mixing and a final linear readout, and
therefore need not be contained in the path-atomic dictionary
$\mU_{L-1}$.
For the path-only specialization described in
\Cref{rem:finite-architecture-relation-to-vbkl},
one has
$\mathcal W_{R,\mathrm{path}}^{L;m,G} \subseteq \mV_R^{(L)}$,
so the same bound applies to that restricted VBKL subclass.
\end{remark}

\begin{corollary}\textnormal{(Generalization guarantee for finite Brownian variation classes)}\ 
\label{cor:vbkl-generalization}
Assume the architecture, parameter, and sample-size conditions of
\Cref{thm:variation-bkl-main-generalization}.
Let
$
\mathcal D_n
:=
\left(
\left(
X_i,
Y_i
\right)
\right)_{i=1}^{n}
$
be an independent sample drawn from
$\nu_{XY}$,
and let
$\mathcal R$
and
$\mathcal R_n$
denote, respectively, the population and empirical risks introduced in
\Cref{subsec:statistical-learning-framework}.
Suppose that
$
\ell
:
\mathbb R
\times
\mY
\longrightarrow
\left[
0,1
\right]
$
is
$L_\ell$-Lipschitz
in its first argument.
Write
$
\b X
:=
\left(
X_1,\ldots,X_n
\right),
$
and define the random lower-support envelope
\begin{align}
B_{\b X}
:=
\sup_{a\in\mathcal A_{L-1}^{m,G}}
\max_{i\in\left[
n
\right]}
\left|
a\left(
X_i
\right)
\right|.
\label{eq:generalization-random-support-envelope}
\end{align}
Let
$C_L$
and
$\Gamma_{L-1,m,G,n}$
be the quantities appearing in
\Cref{thm:variation-bkl-main-generalization}. Then, for every
$\delta\in\left(
0,1
\right)$,
with probability at least
$1-\delta$
over the draw of
$\mathcal D_n$,
every
$F\in\mathcal W_R^{L;m,G}$
satisfies
\begin{align}
\mathcal R\left(
F
\right)
&\le
\mathcal R_n\left(
F
\right)
+
2L_\ell
R
B_{\b X}^{1/2}
\min
\left\{
1,
\;
C_L
\left(
\frac{
\Gamma_{L-1,m,G,n}
}{
n
}
\right)^{1/2}
\right\}
+
3
\left(
\frac{
\ln\left(
2/\delta
\right)
}{
2n
}
\right)^{1/2}
\notag\\
&\le
\mathcal R_n\left(
F
\right)
+
2L_\ell
R
A_{\mX}^{\,2^{-(L-1)}}
\min
\left\{
1,
\;
C_L
\left(
\frac{
\Gamma_{L-1,m,G,n}
}{
n
}
\right)^{1/2}
\right\}
+
3
\left(
\frac{
\ln\left(
2/\delta
\right)
}{
2n
}
\right)^{1/2}.
\end{align}
\end{corollary}

\section{Constructive Approximation of VBKL spaces}\label{sec:constructive-approximation}

This section returns to the full infinite-dimensional VBKL space. Starting
from an arbitrary measure-generated element, the first stage replaces the
outer signed measure by a finite atomic measure while leaving the selected
atoms unchanged. The second stage discretizes only the outermost Brownian
profile of each selected atom; its lower-level support remains in
$\mU_{L-1}$. This separation yields distinct atom-discretization and
outer-profile-interpolation errors.

\subsection{Finite Approximation by Recursive Atomic Representations}
\label{subsec:shared-ladder-approximation}

The first stage discretizes only the outer variation representation: the
recursive dictionary and the selected atoms remain unchanged. Its error is
therefore independent of profile interpolation.

\begin{theorem}\textnormal{(Finite-atomic approximation of the full VBKL space)}\ 
\label{thm:path-bkl-approx}
Assume that
$\mX\subseteq\mathbb R^d$
is compact, that
$\nu$
is a Borel probability measure on
$\mX$,
and that
$\Omega\subseteq\mathbb S^{d-1}$
is compact.
Let
$L\ge2$,
and let
$\VL$
and
$\CLv$
denote the depth-$L$ VBKL space and its variation complexity defined in
\eqref{eq:variation-bkl-space}
and
\eqref{eq:variation-bkl-complexity}. For every
$F\in\VL$
and every
$m\in\mathbb N$,
$m\ge1$,
there exists
$
F_m
\in
\mV_m^{(L)}
$,
where
$\mV_m^{(L)}$
is the finite-support VBKL class introduced in
\Cref{subsec:atomic-bkl},
such that
\begin{align}
\left\|
F-F_m
\right\|_{L^2\left(\nu\right)}
\le
\left(
\int_{\mX}
\left\|
\b x
\right\|_2^{\,2^{-(L-2)}}
\,\d\nu\left(
\b x
\right)
\right)^{1/2}
\CLv\left(
F
\right)
m^{-1/2}.
\label{eq:finite-recursive-approximation-bound}
\end{align}
If
$\CLv\left(
F
\right)=0$,
then
$F=0$
in
$L^2\left(\nu\right)$,
and one may take
$F_m=0$.
If
$\CLv\left(
F
\right)>0$,
then there exist atoms
$
u_1,\ldots,u_m
\in
\mU_L
$
and signs
$
\sigma_1,\ldots,\sigma_m
\in
\left\{
-1,1
\right\}
$
such that
\begin{align}
F_m
=
\frac{
\CLv\left(
F
\right)
}{
m
}
\sum_{i=1}^{m}
\sigma_i
u_i.
\label{eq:finite-recursive-equal-weight-representation}
\end{align}
\end{theorem}

Theorem~\ref{thm:path-bkl-approx} controls the finite-atomic stage. The next
result additionally interpolates the selected atoms' outer Brownian profiles.
It does not discretize their lower-level supports, which remain elements of
$\mU_{L-1}$.

\begin{theorem}\textnormal{(Two-stage finite-atomic and outer-profile approximation of the VBKL space)}
\label{thm:two-stage-path-profile-bkl}
Let
$\mX\subseteq\mathbb R^d$
be compact, let
$\nu$
be a Borel probability measure on
$\mX$,
and let
$\Omega\subseteq\mathbb S^{d-1}$
be compact.
Let
$L\ge2$,
let
$\left(
\mU_l
\right)_{l=1}^{L}$
be the recursive atomic dictionaries introduced in
\Cref{subsec:atomic-bkl},
and let
$\VL$
and
$\CLv$
denote, respectively, the depth-$L$ VBKL space and its variation complexity defined in
\eqref{eq:variation-bkl-space}
and
\eqref{eq:variation-bkl-complexity}. For
$F\in\VL$,
define
$
R_{\mX}
:=
\sup_{\b x\in\mX}
\left\|
\b x
\right\|_2.
$
Set
\begin{align}
A_0
:=
R_{\mX}^{\,2^{-(L-2)}}
\label{eq:A0}
\end{align}
and
\begin{align}
R_L
:=
\left(
\int_{\mX}
\left\|
\b x
\right\|_2^{\,2^{-(L-2)}}
\,\d\nu\left(
\b x
\right)
\right)^{1/2}.
\label{eq:RL}
\end{align}
Let
$M,m\in\mathbb N$.
If
$\CLv\left(
F
\right)=0$,
then
$F=0$
in
$L^2\left(
\nu
\right)$,
and one may take
$F_{M,m}:=0$. Suppose now that
$\CLv\left(
F
\right)>0$.
Then
$R_{\mX}>0$
and
$A_0>0$.
Let
$
-A_0=t_0<t_1<\cdots<t_m=A_0
$
be the uniform grid on
$\left[
-A_0,A_0
\right]$,
and let
$\psi_0,\ldots,\psi_m$
denote the associated continuous piecewise-linear hat functions.
For every function
$g\in\mH_{\kB}$,
let
$\Pi_mg$
denote its continuous piecewise-linear interpolant on this grid, that is,
\begin{align}
\left(
\Pi_mg
\right)
\left(
t
\right)
:=
\sum_{k=0}^{m}
g\left(
t_k
\right)
\psi_k\left(
t
\right),
\qquad
t\in
\left[
-A_0,A_0
\right].
\end{align}
Then there exist atoms
$
u_1,\ldots,u_M
\in
\mU_L
$
and signs
$
\sigma_1,\ldots,\sigma_M
\in
\left\{
-1,1
\right\}
$
such that, for every
$j\in\left[
M
\right]$,
there exist
$
a_{u_j}\in\mU_{L-1}
$
and
$
g_{u_j}\in\mH_{\kB}
$
satisfying
\begin{align}
u_j\left(
\b x
\right)
&=
g_{u_j}\left(
a_{u_j}\left(
\b x
\right)
\right),
\qquad
\b x\in\mX,
\\
\left\|
g_{u_j}
\right\|_{\mH_{\kB}}
&\le
1,
\\
a_{u_j}\left(
\mX
\right)
&\subseteq
\left[
-A_0,A_0
\right].
\end{align}
The preceding outer-profile representation follows from
\Cref{lem:vbkl-union-rkhs},
while the range inclusion follows from
\Cref{lem:lower-level-atomic-range-bound}. Define
\begin{align}
F_{M,m}\left(
\b x
\right)
:=
\frac{
\CLv\left(
F
\right)
}{
M
}
\sum_{j=1}^{M}
\sigma_j
\left(
\Pi_mg_{u_j}
\right)
\left(
a_{u_j}\left(
\b x
\right)
\right),
\qquad
\b x\in\mX.
\label{eq:two-stage-approximant}
\end{align}
Then
\begin{align}
\left\|
F
-
F_{M,m}
\right\|_{L^2\left(
\nu
\right)}
\le
\CLv\left(
F
\right)
\left[
R_L
M^{-1/2}
+
\left(
\frac{
A_0
}{
2
}
\right)^{1/2}
m^{-1/2}
\right].
\label{eq:two-stage-final-bound}
\end{align}
Moreover, for every
$\b x\in\mX$
and every
$j\in\left[
M
\right]$,
the nodal representation
\begin{align}
\left(
\Pi_mg_{u_j}
\right)
\left(
a_{u_j}\left(
\b x
\right)
\right)
&=
\sum_{k=0}^{m}
g_{u_j}\left(
t_k
\right)
\psi_k\left(
a_{u_j}\left(
\b x
\right)
\right)
\end{align}
contains at most two nonzero terms.
Consequently, evaluating
$F_{M,m}\left(
\b x
\right)$
requires at most
$2M$
active outer-profile basis contributions, independently of the interpolation resolution
$m$.
In the case
$\CLv\left(
F
\right)=0$,
the zero realization requires no active outer-profile contribution.

\end{theorem}

The resulting approximant has finite outer atomic support and finite
outer-profile interpolation, while preserving the selected lower-level
supports exactly.

\begin{remark}
\label{rem:two-stage-approximation}
The two error terms correspond to different resources. The $M^{-1/2}$ term
comes from replacing the outer measure by at most $M$ atoms; the $m^{-1/2}$
term comes from interpolating the selected outer Brownian profiles. The
lower-level supports are not discretized by
\Cref{thm:two-stage-path-profile-bkl}.
\end{remark}

\begin{corollary}\textnormal{(Balanced two-stage approximation)}\ 
\label{cor:balanced-realization}
Assume the hypotheses and notation of
\Cref{thm:two-stage-path-profile-bkl}.
Let
$N\in\mathbb N$,
$N\ge1$,
and choose the number of selected recursive atoms and the outer-profile interpolation resolution identically:
$
M
=
m
=
N.
$
Then the two-stage approximant
$F_{N,N}$
provided by
\Cref{thm:two-stage-path-profile-bkl}
satisfies
\begin{align}
\left\|
F-F_{N,N}
\right\|_{L^2\left(\nu\right)}
\le
\CLv\left(
F
\right)
\left[
R_L
+
\left(
\frac{
A_0
}{
2
}
\right)^{1/2}
\right]
N^{-1/2}.
\label{eq:balanced-two-stage-approximation-bound}
\end{align}
Consequently, for fixed
$F$,
$R_L$,
and
$A_0$,
one has
$
\left\|
F-F_{N,N}
\right\|_{L^2\left(\nu\right)}
=
O\left(
N^{-1/2}
\right).
$

\end{corollary}

\begin{remark}\textnormal{(Two sources of approximation error)}\ 
The bound in \Cref{thm:two-stage-path-profile-bkl} is
\begin{align}
R_LM^{-1/2}
+
\sqrt{\frac{A_0}{2}}\,m^{-1/2},
\end{align}
separating outer-measure discretization from outer-profile interpolation.
\end{remark}

\subsection{Approximation of Brownian Profile Functions}
\label{sec:profile-bkl-approximation}
We now isolate the outer-profile interpolation step and establish a
quantitative estimate for Brownian RKHS functions.

\begin{prop}\textnormal{(Interpolation of Brownian RKHS functions on a symmetric interval)}\ 
\label{thm:profile-bkl-approximation}
Let
$A>0$
and
$m\in\mathbb N$,
$m\ge1$.
Let
$
-A=t_0<t_1<\cdots<t_m=A
$
be the uniform grid on
$\left[-A,A\right]$,
and let
$\Pi_m$
denote the associated continuous piecewise-linear interpolation operator.
Thus, for
$g\in\mH_{\kB}$,
the function
$\Pi_mg$
is the unique continuous function on
$\left[-A,A\right]$
that is affine on every interval
$\left[t_i,t_{i+1}\right]$
and satisfies
$
\left(
\Pi_mg
\right)
\left(
t_i
\right)
=
g\left(
t_i
\right)$, $i\in
\left\{
0,\ldots,m
\right\}.
$ Then, for every
$g\in\mH_{\kB}$,
\begin{align}
\left\|
g-\Pi_mg
\right\|_{L^\infty\left(\left[-A,A\right]\right)}
\le
\left(
\frac{
A
}{
2
}
\right)^{1/2}
m^{-1/2}
\left\|
g
\right\|_{\mH_{\kB}}.
\label{eq:profile-interpolation-symmetric-bound}
\end{align}

\end{prop}

Together, the results give a two-stage approximation of the full VBKL
space: the outer measure is finitely supported and the selected outer
profiles are piecewise linear, while the lower-level supports may remain
infinite-dimensional elements of $\mU_{L-1}$.

\begin{remark}
The interpolation estimate of
\Cref{thm:profile-bkl-approximation-symmetric}
is optimal.
A sharpness proposition showing that neither the convergence rate
$m^{-1/2}$
nor the constant
$\sqrt{A/2}$
can be improved is proved in
Appendix~\ref{sec:sharpness-brownian-interpolation}.
\end{remark}

\section{Experiments}
\label{sec:experiments}

The experiments are organized around the theoretical mechanisms developed
above rather than around a large benchmark collection. They address four
questions: whether the finite constructions exhibit the predicted
approximation behaviour; whether validation-selected realizations can be
learned effectively from finite data; whether their predictive accuracy is
obtained with an economical recursive representation; and whether the
resulting finite models admit numerically stable optimization. Complete
protocols, validation-selected configurations, numerical tables, and
computational measurements are reported in
Appendix~\ref{app:experimental-details}.

\subsection{Experimental questions and common protocol}
\label{sec:experimental-protocol}

The approximation studies use deterministic teachers whose recursive
representations are known explicitly, so that the effects of signed-measure
discretization and Brownian-profile interpolation can be evaluated without
statistical estimation or training. The supervised studies use a controlled
recursive-teacher problem and the Energy Efficiency regression benchmark.
They compare finite VBKL realizations with Deep Neural Variation Spaces
(DNVS), kernel ridge regression (KRR), and random Fourier features (RBF).
For each training-set size and random seed, the number of recursive Brownian
paths and the optimization horizon of VBKL are selected using validation
data before the selected model is retrained. The Energy Efficiency study
also selects the Brownian-profile resolution by validation. All supervised
results are reported over five independent random seeds using test
mean-squared error. The compared methods use the same data-splitting
protocol within each benchmark.

\subsection{Constructive approximation and interpolation sharpness}
\label{sec:exp-constructive}

We first examine the outer-profile-discretization component of the
two-stage approximation. Three representative unit-norm Brownian RKHS profiles
are considered: a smooth sinusoidal profile, a localized profile, and an
oscillatory profile. Their continuous piecewise-linear interpolants are
evaluated on uniform grids with
$m\in\{8,16,32,64,128,256,512\}$.
Figure~\ref{fig:profile-interpolation}(a) shows that these fixed profiles
converge substantially faster than the uniform worst-case rate. This
behaviour reflects their additional regularity and clarifies that
\Cref{thm:profile-bkl-approximation} is a worst-case guarantee rather than
the asymptotic rate of every fixed profile.

For each resolution, we also construct a normalized tent profile supported
on a single interpolation interval. Its interpolation error coincides with
the theoretical bound at every tested resolution, as shown in
Figure~\ref{fig:profile-interpolation}(b). The experiment therefore
separates typical fixed-profile behaviour from the sharp worst case proved
in Appendix~\ref{sec:sharpness-brownian-interpolation}.

\begin{figure}[t]
    \centering
    \includegraphics[width=\linewidth]{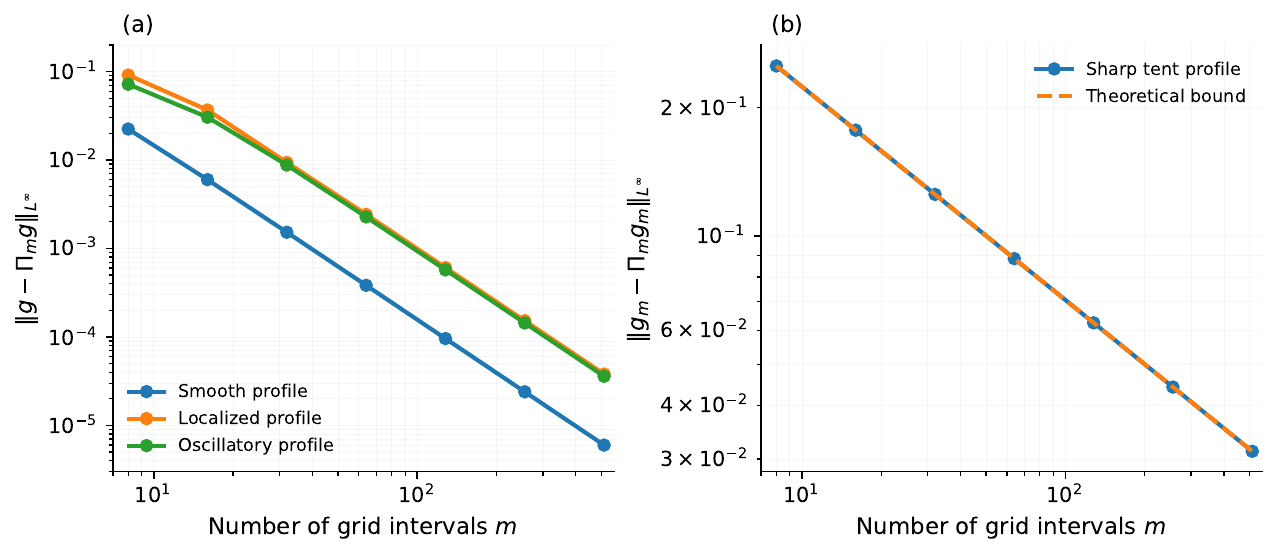}
    \caption{
    \textbf{Brownian profile interpolation.}
    \textbf{(a)} Interpolation errors for three representative unit-norm
    Brownian RKHS profiles.
    \textbf{(b)} Interpolation error of the normalized tent profile together
    with the theoretical interpolation bound. The coincidence of the curves
    demonstrates sharpness of the interpolation estimate.
    }
    \label{fig:profile-interpolation}
\end{figure}

We next test both stages of the construction. In
Figure~\ref{fig:constructive_realization}(a), only the representing signed
measure is discretized while the selected Brownian profiles remain exact;
the error decreases consistently with the $M^{-1/2}$ reference predicted by
\Cref{thm:path-bkl-approx}. Panel~(b) fixes the finite signed-measure
representation and varies only the profile resolution. The deterministic
teacher converges faster than the $m^{-1/2}$ worst-case upper bound, in
agreement with the fixed-profile behaviour above. Panel~(c) refines both
complexities simultaneously with $M=m=N$ and exhibits the balanced
$N^{-1/2}$ behaviour of \Cref{cor:balanced-realization}. Together, the two figures illustrate the separate roles of outer-measure
and outer-profile complexity in \Cref{thm:two-stage-path-profile-bkl}.

\begin{figure}[t]
    \centering
    \includegraphics[width=\linewidth]{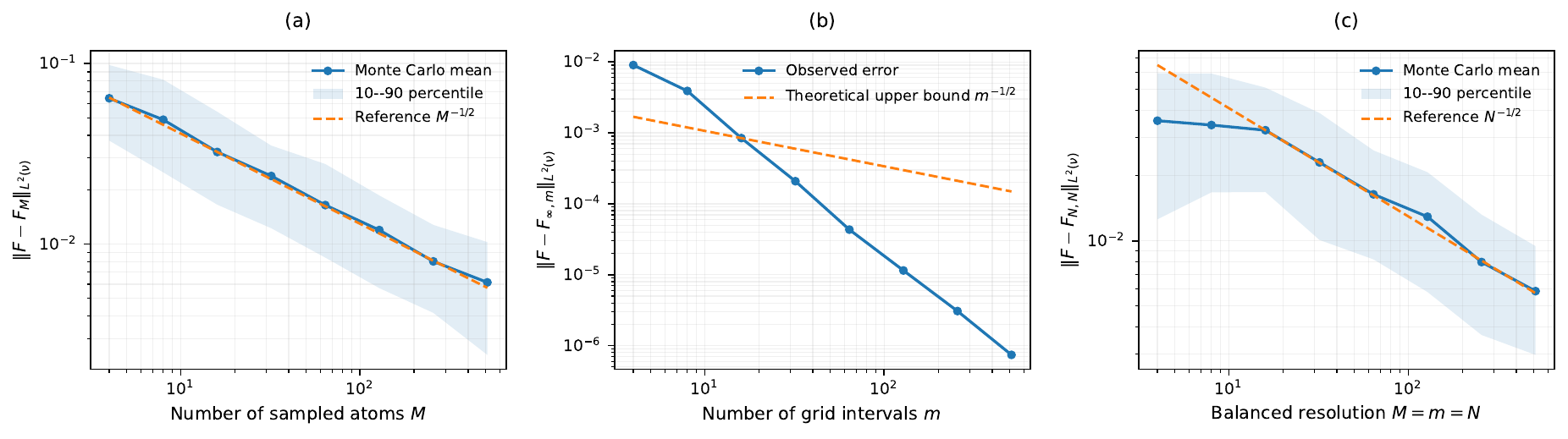}
    \caption{
    \textbf{Two-stage constructive approximation.}
    \textbf{(a)} Finite signed-measure discretization with exact profiles.
    \textbf{(b)} Brownian-profile discretization with a fixed finite
    measure representation, together with the worst-case upper bound.
    \textbf{(c)} Balanced refinement $M=m=N$. The reference slopes correspond
    to the rates in \Cref{thm:path-bkl-approx,cor:balanced-realization}.
    }
    \label{fig:constructive_realization}
\end{figure}

\subsection{Statistical learning and parameter efficiency}
\label{sec:exp-statistical}

Figure~\ref{fig:statistical-validation}(a) reports the recursive-teacher
learning curve. This controlled benchmark matches the hierarchical
Brownian structure of the target to the proposed hypothesis class. The
validation-selected VBKL realizations are particularly effective in the
limited-data regime and remain competitive as the sample size grows; the
performance differences narrow at larger sample sizes, where DNVS attains
the lowest mean error at the largest sizes considered. This behaviour is
consistent with using validation to adapt the explicit finite
representation complexity rather than fixing one architecture throughout
the learning curve.

The Energy Efficiency benchmark provides a complementary real-data test
whose generating mechanism is unknown. Figure~\ref{fig:statistical-validation}(b)
shows that, at $n=100$, VBKL is comparable to the strongest baselines and
has lower mean error than DNVS\@. At $n=250$ and $n=500$, VBKL does not lead
the predictive comparison: DNVS, KRR, and RBF obtain lower mean errors.
Thus, the empirical claim is not universal dominance, but a distinct
limited-data regime in which the recursive Brownian realization is most
competitive.

Representation size provides the second part of this comparison.
Figure~\ref{fig:statistical-validation}(c) and
Table~\ref{tab:energy-efficiency} show that VBKL uses substantially fewer
parameters than DNVS at every Energy Efficiency training size. At $n=100$,
VBKL wins four of the five matched seed-wise comparisons while DNVS uses
approximately $4.6\times$ as many parameters. The parameter ratio grows to
$13.7\times$ and $18.3\times$ at $n=250$ and $n=500$, respectively, even
though DNVS wins more of the corresponding accuracy comparisons. The
combined evidence therefore supports a favourable accuracy--complexity
trade-off, with the clearest empirical advantage in the smallest-data
regime.

\begin{figure}[t]
    \centering
    \subfloat[Recursive-teacher learning curve.\label{fig:recursive_learning}]{%
        \includegraphics[width=.86\linewidth]{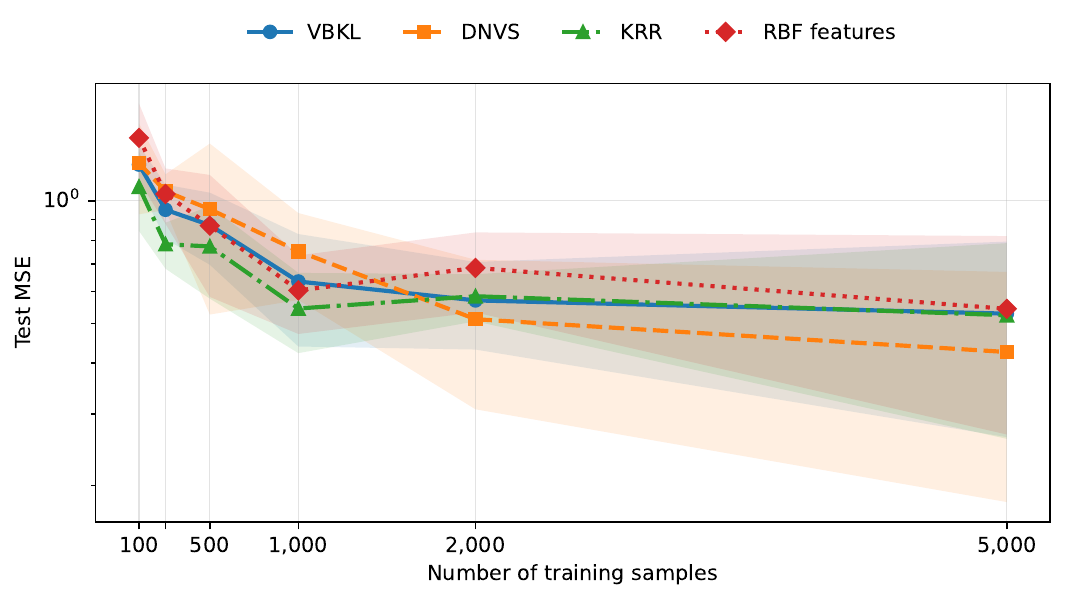}%
    }\\[1.0ex]
    \subfloat[Energy Efficiency learning curve.\label{fig:energy_learning_curve}]{%
        \includegraphics[width=.48\linewidth]{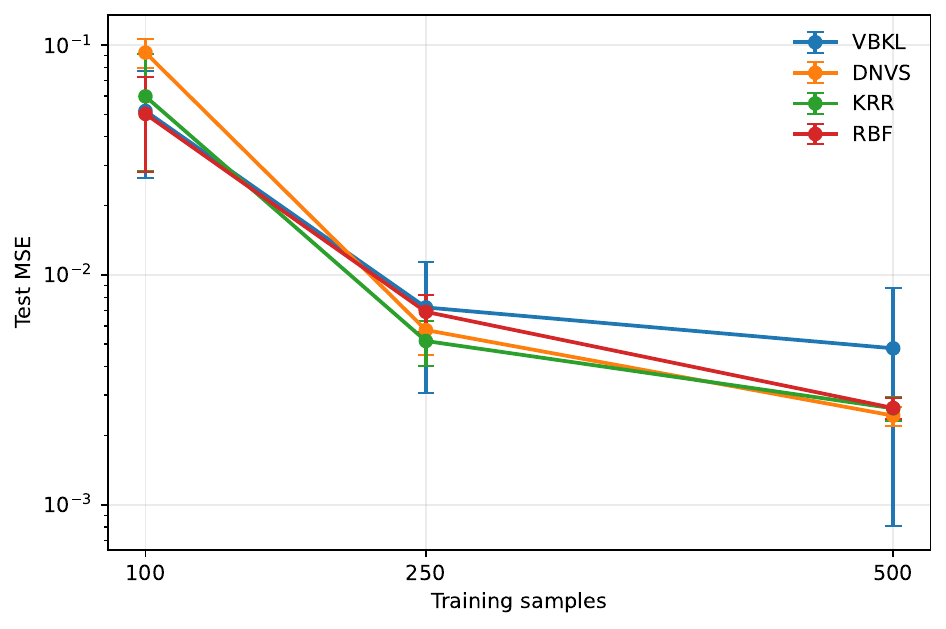}%
    }\hfill
    \subfloat[Accuracy--complexity comparison with DNVS\@.\label{fig:energy_parameter_efficiency}]{%
        \includegraphics[width=.48\linewidth]{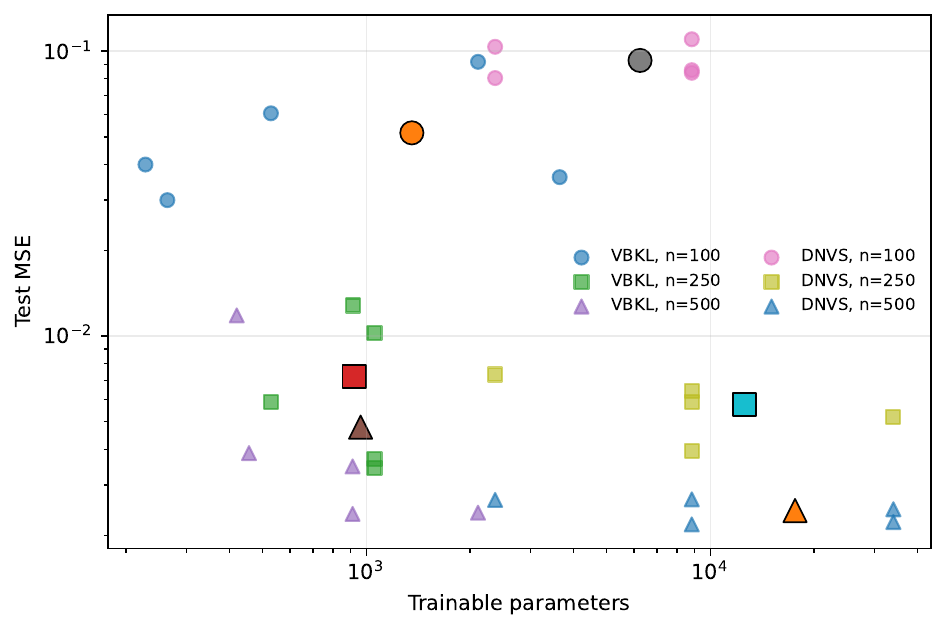}%
    }
    \caption{
    \textbf{Statistical learning and representation efficiency.}
    Test mean-squared error is averaged over five independent random seeds.
    VBKL complexity is selected independently by validation for each
    training-set size and seed. Panel~\textbf{(c)} uses logarithmic axes;
    small markers represent individual seeds and outlined markers represent
    means.
    }
    \label{fig:statistical-validation}
\end{figure}

\begin{table}[t]
\centering
\begin{tabular}{ccc}
\toprule
$n$ & VBKL wins vs.\ DNVS & Parameter ratio \\
\midrule
100 & 4/5 & 4.6$\times$ \\
250 & 2/5 & 13.7$\times$ \\
500 & 2/5 & 18.3$\times$ \\
\bottomrule
\end{tabular}

\caption{Seed-wise and parameter-efficiency comparison between VBKL and DNVS\@. The parameter ratio is DNVS/VBKL.}
\label{tab:energy-efficiency}
\end{table}
 
\subsection{Optimization stability and practical realizability}
\label{sec:exp-optimization}

The constructive theory produces finite recursive models whose parameters
must be optimized numerically. Figure~\ref{fig:optimization_validation}(a)
compares finite-difference directional estimates with the corresponding
automatic-differentiation directional derivatives. The relative error falls
as the interaction scale decreases from $10^{-1}$ to $10^{-3}$ and rises
modestly at $10^{-4}$, displaying the expected finite-difference trade-off
at the smallest scale rather than monotone improvement for arbitrarily small
steps. Panel~(b) evaluates Monte Carlo directional averaging: increasing the
number of sampled directions progressively reduces the estimator standard
deviation.

These controlled tests verify the implementation of the directional
estimator and the variance reduction obtained from averaging. They do not,
by themselves, establish a global optimization-convergence theorem; rather,
they show that the finite VBKL realizations can be optimized with a stable
numerical estimator under the tested conditions. Detailed computational
measurements are reported in Appendix~\ref{app:computational-details}.

\begin{figure}[t]
    \centering
    \includegraphics[width=\linewidth]{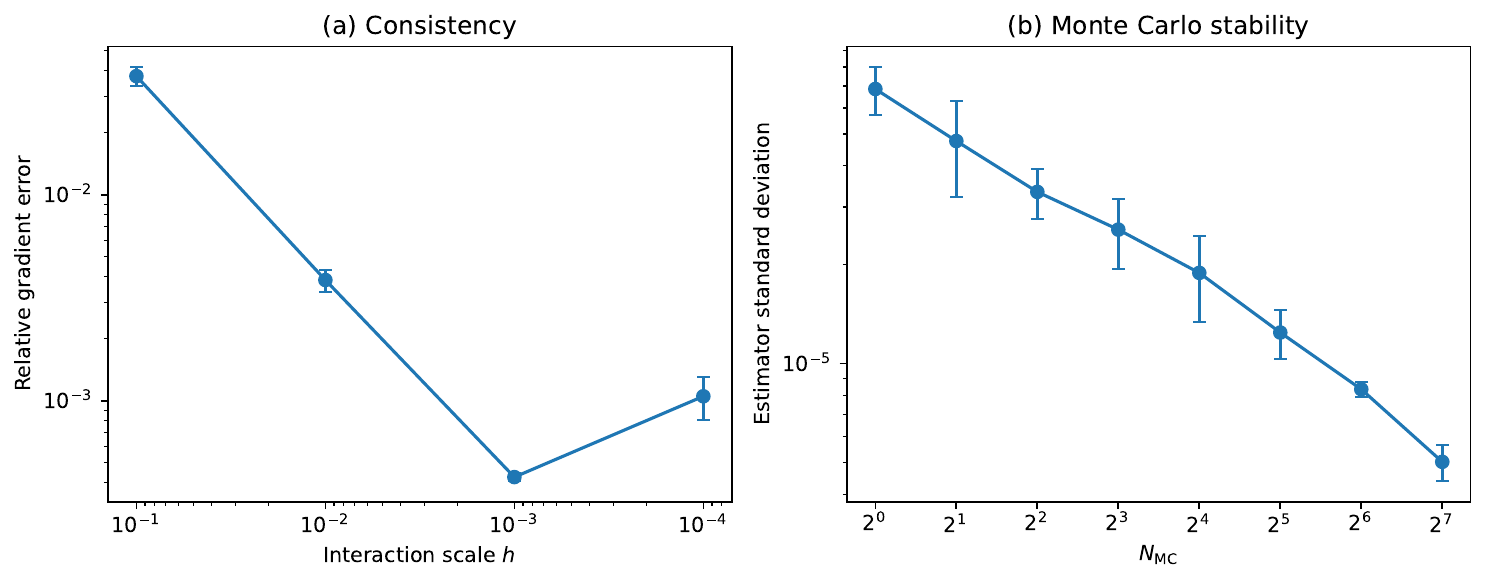}
    \caption{
    \textbf{Optimization validation of finite VBKL realizations.}
    \textbf{(a)} Relative error between automatic-differentiation
    directional derivatives and the averaged directional estimator across
    interaction scales.
    \textbf{(b)} Estimator standard deviation under Monte Carlo directional
    averaging. Increasing the number of sampled directions reduces the
    variability of the estimate.
    }
    \label{fig:optimization_validation}
\end{figure}

\FloatBarrier

\section{Conclusion}
\label{sec:conclusion}

This paper introduced the \emph{Variation Brownian Kernel Ladder} (VBKL), a
path-atomic function-space framework in which recursive Brownian dictionary
construction precedes outer variation superposition. The recursive
dictionaries are unions of Brownian pullback RKHS unit balls. Their variation
hulls admit quantitative pointwise and H\"older representatives; under full
support these representatives define a continuous H\"older embedding, and
under the stated local non-degeneracy and trace-support conditions the spaces
form a strict depth hierarchy.

For associated finite lower-support architectures, we derived Rademacher and
generalization bounds that separate the outer variation radius from the
architecture parameter count. For the full VBKL space, the constructive
results separate two approximation resources: $M$ controls finite support of
the outer measure, and $m$ controls interpolation of the selected outer
Brownian profiles. The lower-level supports remain in $\mU_{L-1}$. The error
is the sum of $M^{-1/2}$ and $m^{-1/2}$ terms, the profile-interpolation
constant is sharp, and each evaluation uses at most $2M$ active outer-profile
basis contributions.

The experiments illustrate these mechanisms and indicate a favorable
trade-off between accuracy and complexity in limited-data regimes, without claiming
universal predictive dominance or global optimization convergence. Future
work includes intrinsic full-space statistical bounds, recursive
discretization of lower-level supports, and optimization theory for explicit
finite architectures.

\section{Proofs}\label{app}
This section contains the proofs of the main analytical, statistical, and
approximation results.

\subsection{Proof of \texorpdfstring{\Cref{thm:analytical-properties-variation-bkl}\ref{thm:variation-holder-pointwise-bound}}{}}

\begin{noheadproof}
Fix $F\in\VL$. For each $n\ge1$, choose a finite signed measure $\mu_n$ on
$\mU_L$ representing $F$ in $L^2(\nu)$ and satisfying
\begin{align}
\|\mu_n\|_{\mathrm{TV}}
\le
\CLv(F)+n^{-1}.
\end{align}
The induced Borel structure on $\mU_L\subset C(\mX)$, separability of
$C(\mX)$, and the uniform atom bound in \Cref{lem:path-atom-bound} imply
that the canonical inclusion is Bochner integrable in $C(\mX)$ with respect
to each $\mu_n$. Define the corresponding $C(\mX)$-valued barycentre by
\begin{align}
F_n(\b x)
:=
\int_{\mU_L}
u(\b x)\,\d\mu_n(u),
\qquad
\b x\in\mX.
\end{align}
By \Cref{lem:path-atom-bound}, every $u\in\mU_L$ satisfies the uniform
$\alpha_L$-H\"older and pointwise bounds. Hence
\begin{align}
|F_n(\b x)-F_n(\b x')|
&\le
\|\b x-\b x'\|_2^{\alpha_L}\|\mu_n\|_{\mathrm{TV}},
\notag\\
|F_n(\b x)|
&\le
\|\b x\|_2^{\alpha_L}\|\mu_n\|_{\mathrm{TV}}.
\end{align}
The sequence $(F_n)_{n\ge1}$ is uniformly bounded and equicontinuous on the
compact set $\mX$. By Arzel\`a--Ascoli, a subsequence converges uniformly to
some $F_\star\in C(\mX)$. The continuous inclusion $C(\mX)\hookrightarrow L^2(\nu)$ commutes with
Bochner integration and shows that each $F_n$ represents the same element $F$ of $L^2(\nu)$. Uniform convergence
therefore implies convergence in $L^2(\nu)$, and hence
$F_\star=F$ in $L^2(\nu)$. Passing to the limit in the preceding bounds and using
$\|\mu_n\|_{\mathrm{TV}}\to\CLv(F)$ along the selected near-minimizing
sequence gives
\begin{align}
|F_\star(\b x)-F_\star(\b x')|
&\le
\|\b x-\b x'\|_2^{\alpha_L}\CLv(F),
\notag\\
|F_\star(\b x)|
&\le
\|\b x\|_2^{\alpha_L}\CLv(F).
\end{align}
Identifying $F$ with this representative proves
\eqref{eq:variation-holder-bound} and \eqref{eq:variation-pointwise-bound}.
\end{noheadproof}
\hfill\BlackBox

\subsection{Proof of \texorpdfstring{\Cref{thm:analytical-properties-variation-bkl}\ref{thm:continuous-embedding}}{}}

\begin{noheadproof}
The bounds in \eqref{eq:variation-holder-bound} and
\eqref{eq:variation-pointwise-bound} give
\begin{align}
[F]_{C^{0,\alpha_L}(\mX)}
&\le
\CLv(F),
&
\|F\|_{L^\infty(\mX)}
&\le
R_{\mX}^{\alpha_L}\CLv(F).
\end{align}
Adding the two estimates proves
\eqref{eq:variation-holder-embedding-bound}. If
$\operatorname{supp}(\nu)=\mX$, two continuous representatives of the same
$L^2(\nu)$ element agree on $\mX$; hence the representative is unique and
the induced map into $C^{0,\alpha_L}(\mX)$ is linear, injective, and
continuous.
\end{noheadproof}
\hfill\BlackBox

\subsection{Proof of
\texorpdfstring{
\Cref{thm:analytical-properties-variation-bkl}
\ref{thm:depth-monotonicity}
}{}
}

\begin{noheadproof}
We first prove the depth inclusion.
Set
\begin{align}
A_L
:=
R_{\mX}^{\,\alpha_L}.
\label{eq:vbkl-depth-inclusion-atomic-range}
\end{align}
By \Cref{lem:path-atom-bound}, every $u\in\mU_L$ satisfies
\begin{align}
|u(\b x)|
\le
\|\b x\|_2^{\alpha_L}
\le
R_{\mX}^{\alpha_L}
=
A_L,
\qquad
\b x\in\mX.
\label{eq:vbkl-depth-inclusion-uniform-range}
\end{align}
Thus,
\begin{align}
u\left(
\mX
\right)
\subseteq
\left[
-A_L,A_L
\right],
\qquad
u\in\mU_L.
\label{eq:vbkl-depth-inclusion-common-range}
\end{align}
Choose
$\chi_L\in C_c^\infty\left(\mathbb R\right)$
satisfying
$\chi_L\left(t\right)=1$
for
$t\in\left[-A_L,A_L\right]$,
and define
\begin{align}
\phi_L\left(
t
\right)
:=
t\chi_L\left(
t
\right),
\qquad
t\in\mathbb R.
\label{eq:vbkl-depth-inclusion-identity-profile}
\end{align}
The function
$\phi_L$
is absolutely continuous, vanishes at the origin, and has a compactly
supported derivative in
$L^2\left(\mathbb R\right)$.
Hence, by the Brownian RKHS characterization
\citep[Lemma~C9]{
MohammadigohariDiFattaNicosiaPardalos2026BKL},
\begin{align}
\phi_L
&\in
\mH_{\kB},
&
c_L
&:=
\left\|
\phi_L
\right\|_{\mH_{\kB}}
\in
\left(
0,\infty
\right).
\end{align}
Define
\begin{align}
S_L
:
\mU_L
&\longrightarrow
\mU_{L+1},
&
S_L\left(
u
\right)
&:=
\frac{
\phi_L
}{
c_L
}
\circ u.
\label{eq:vbkl-depth-inclusion-map}
\end{align}
Since
$\left\|\phi_L/c_L\right\|_{\mH_{\kB}}=1$,
the recursive definition of
$\mU_{L+1}$
gives
$S_L\left(u\right)\in\mU_{L+1}$.
Moreover, since $\phi_L$ is globally Lipschitz,
\begin{align}
\left\|S_L(u)-S_L(v)\right\|_{L^\infty(\mX)}
\le
\frac{\operatorname{Lip}(\phi_L)}{c_L}
\left\|u-v\right\|_{L^\infty(\mX)},
\qquad u,v\in\mU_L.
\end{align}
Thus $S_L$ is continuous for the induced supremum-norm topologies and hence
Borel measurable.
Since
$\phi_L$
agrees with the identity on
$\left[-A_L,A_L\right]$,
\eqref{eq:vbkl-depth-inclusion-common-range}
gives
\begin{align}
u
=
c_L
S_L\left(
u
\right),
\qquad
u\in\mU_L.
\label{eq:depth-atom-inclusion-scaled}
\end{align}
Fix
$F\in\VL$,
and choose a finite signed Borel measure
$\mu\in\mathcal M\left(\mU_L\right)$
such that
\begin{align}
F
&=
\int_{\mU_L}
u
\,\d\mu\left(
u
\right)
\qquad
\text{in }
L^2\left(
\nu
\right),
&
\left\|
\mu
\right\|_{\mathrm{TV}}
&<
\infty.
\end{align}
Define
\begin{align}
\widetilde\mu
:=
c_L
\left(
S_L
\right)_{\#}
\mu.
\label{eq:vbkl-depth-inclusion-pushforward-measure}
\end{align}
Therefore,
\begin{align}
\int_{\mU_{L+1}}
v
\,\d\widetilde\mu\left(
v
\right)
&\stackrel{(a)}{=}
c_L
\int_{\mU_L}
S_L\left(
u
\right)
\,\d\mu\left(
u
\right)
\stackrel{(b)}{=}
\int_{\mU_L}
u
\,\d\mu\left(
u
\right)
\stackrel{(c)}{=}
F.
\label{eq:vbkl-depth-inclusion-pushforward-representation}
\end{align}
Here,
(a) applies the push-forward change-of-variables identity,
(b) uses
\eqref{eq:depth-atom-inclusion-scaled},
and
(c) uses the selected representation of
$F$.
Furthermore,
\begin{align}
\left\|
\widetilde\mu
\right\|_{\mathrm{TV}}
\le
c_L
\left\|
\mu
\right\|_{\mathrm{TV}}
<
\infty.
\end{align}
Hence,
$F\in\mV^{(L+1)}$.
Since
$F\in\VL$
was arbitrary,
\begin{align}
\VL
\subseteq
\mV^{(L+1)}.
\end{align}
We now prove strictness under the local non-degeneracy and trace-support
conditions in the theorem statement.
Recall that
$\alpha_L=2^{-(L-1)}$,
and choose
\begin{align}
a
\in
\left(
\frac{1}{2},
\alpha_L^{1/L}
\right)
=
\left(
\frac{1}{2},
2^{-(L-1)/L}
\right).
\label{eq:vbkl-strictness-fractional-exponent}
\end{align}
This interval is nonempty because
$2^{-(L-1)/L}=2^{-1+1/L}>1/2$.
Following the fractional-power construction in
\citet[Theorem~1(iii)]{
MohammadigohariDiFattaNicosiaPardalos2026BKL},
choose
$\eta\in C_c^\infty\left(\mathbb R\right)$
such that
$0\le\eta\le1$
and
$\eta\left(t\right)=1$
for
$\left|t\right|\le1/2$,
and define
\begin{align}
g_a\left(
t
\right)
:=
\eta\left(
t
\right)
\left(
t_+
\right)^a,
\qquad
t_+
:=
\max
\left\{
t,0
\right\}.
\label{eq:vbkl-strictness-raw-profile}
\end{align}
The calculation in the cited proof shows that
$a>1/2$
implies
$g_a\in\mH_{\kB}$.
Set
\begin{align}
\kappa_a
&:=
\left\|
g_a
\right\|_{\mH_{\kB}},
&
h_a
&:=
\frac{
g_a
}{
\kappa_a
}.
\label{eq:vbkl-strictness-normalized-profile}
\end{align}
Then
\begin{align}
\left\|
h_a
\right\|_{\mH_{\kB}}
&=
1,
&
h_a\left(
s
\right)
&=
\kappa_a^{-1}
s^a,
\qquad
0\le s\le\frac{1}{2}.
\label{eq:vbkl-strictness-local-profile-form}
\end{align}
Starting from the first-layer generator
$u_1$
in the theorem statement, define
\begin{align}
v_1
&:=
u_1,
&
v_{l+1}
&:=
h_a\circ v_l,
\qquad
l\in\left[L\right].
\label{eq:vbkl-strictness-path-recursion}
\end{align}
Since
$\left\|h_a\right\|_{\mH_{\kB}}=1$,
the recursive dictionary construction gives
\begin{align}
v_l
&\in
\mU_l,
\qquad
l\in
\left\{
1,\ldots,L+1
\right\},
\label{eq:vbkl-strictness-path-membership}
\\
v_{L+1}
&\in
\mU_{L+1}
\subseteq
\mV^{(L+1)}.
\label{eq:vbkl-strictness-upper-space-membership}
\end{align}
We claim that, for every
$l\in\left\{1,\ldots,L+1\right\}$,
there exist
$\delta_l\in\left(0,\rho\right]$
and
$c_l,C_l\in\left(0,\infty\right)$
such that
\begin{align}
c_l
t^{a^{l-1}}
\le
v_l\left(
\gamma\left(
t
\right)
\right)
\le
C_l
t^{a^{l-1}},
\qquad
t\in
\left[
0,\delta_l
\right].
\label{eq:vbkl-strictness-recursive-trace-bound}
\end{align}
For
$l=1$,
this is precisely the local non-degeneracy assumption, with
$\delta_1:=\rho$.
Suppose that the claim holds at some
$l\in\left[L\right]$,
and choose
$\delta_{l+1}\in\left(0,\delta_l\right]$
such that
\begin{align}
C_l
\delta_{l+1}^{\,a^{l-1}}
\le
\frac{1}{2}.
\label{eq:vbkl-strictness-localization-radius}
\end{align}
Therefore, for
$t\in\left[0,\delta_{l+1}\right]$,
\begin{align}
0
&\stackrel{(a)}{\le}
v_l\left(
\gamma\left(
t
\right)
\right)
\stackrel{(b)}{\le}
C_l
t^{a^{l-1}}
\stackrel{(c)}{\le}
\frac{1}{2},
\notag\\
v_{l+1}\left(
\gamma\left(
t
\right)
\right)
&\stackrel{(d)}{=}
h_a\left(
v_l\left(
\gamma\left(
t
\right)
\right)
\right)
\stackrel{(e)}{=}
\kappa_a^{-1}
v_l\left(
\gamma\left(
t
\right)
\right)^a.
\end{align}
Here,
(a) and
(b) use the induction bounds,
(c) applies
\eqref{eq:vbkl-strictness-localization-radius},
(d) uses
\eqref{eq:vbkl-strictness-path-recursion},
and
(e) applies
\eqref{eq:vbkl-strictness-local-profile-form}.
Raising the induction bounds to the positive power
$a$
therefore gives
\begin{align}
\kappa_a^{-1}
c_l^a
t^{a^l}
\le
v_{l+1}\left(
\gamma\left(
t
\right)
\right)
\le
\kappa_a^{-1}
C_l^a
t^{a^l},
\qquad
t\in
\left[
0,\delta_{l+1}
\right].
\end{align}
Thus the induction continues with
$c_{l+1}:=\kappa_a^{-1}c_l^a$
and
$C_{l+1}:=\kappa_a^{-1}C_l^a$. Set
$\delta:=\delta_{L+1}$.
Then
\begin{align}
c_{L+1}
t^{a^L}
\le
v_{L+1}\left(
\gamma\left(
t
\right)
\right)
\le
C_{L+1}
t^{a^L},
\qquad
t\in
\left[
0,\delta
\right].
\label{eq:vbkl-strictness-final-trace-bound}
\end{align}
Moreover,
$u_1\left(\b x_0\right)=0$
and
$h_a\left(0\right)=0$
imply recursively that
\begin{align}
v_l\left(
\b x_0
\right)
=
0,
\qquad
l\in
\left\{
1,\ldots,L+1
\right\}.
\label{eq:vbkl-strictness-anchor-values}
\end{align}
Suppose, toward a contradiction, that
$v_{L+1}\in\VL$.
By
\Cref{thm:analytical-properties-variation-bkl}
\ref{thm:variation-holder-pointwise-bound},
its
$L^2\left(\nu\right)$
equivalence class admits an
$\alpha_L$-Hölder representative
$\widetilde v_{L+1}$.
Both
$v_{L+1}$
and
$\widetilde v_{L+1}$
are continuous and agree
$\nu$-almost everywhere.
Consequently, they agree throughout
$\operatorname{supp}\left(\nu\right)$:
\begin{align}
v_{L+1}\left(
\b x
\right)
=
\widetilde v_{L+1}\left(
\b x
\right),
\qquad
\b x\in
\operatorname{supp}\left(
\nu
\right).
\label{eq:vbkl-strictness-representative-agreement}
\end{align}
Let
$L_\gamma<\infty$
be a Lipschitz constant of
$\gamma$.
By
\eqref{eq:vbkl-strictness-trace-support-proof},
\eqref{eq:vbkl-strictness-final-trace-bound},
\eqref{eq:vbkl-strictness-anchor-values},
and
\eqref{eq:vbkl-strictness-representative-agreement},
there exists
$C_\star<\infty$
such that, for every
$t\in\left(0,\delta\right]$,
\begin{align}
c_{L+1}
t^{a^L}
&\stackrel{(a)}{\le}
v_{L+1}\left(
\gamma\left(
t
\right)
\right)
\stackrel{(b)}{=}
\left|
\widetilde v_{L+1}\left(
\gamma\left(
t
\right)
\right)
-
\widetilde v_{L+1}\left(
\gamma\left(
0
\right)
\right)
\right|
\notag\\
&\stackrel{(c)}{\le}
C_\star
\left\|
\gamma\left(
t
\right)
-
\gamma\left(
0
\right)
\right\|_2^{\,\alpha_L}
\stackrel{(d)}{\le}
C_\star
L_\gamma^{\,\alpha_L}
t^{\alpha_L}.
\label{eq:vbkl-strictness-holder-upper-bound}
\end{align}
Here,
(a) applies the lower trace bound,
(b) uses the anchor identity and the agreement of the two representatives
on the trace,
(c) applies the Hölder estimate,
and
(d) uses the Lipschitz continuity of
$\gamma$.
Therefore,
\begin{align}
c_{L+1}
\le
C_\star
L_\gamma^{\,\alpha_L}
t^{\alpha_L-a^L}.
\label{eq:vbkl-strictness-final-contradiction}
\end{align}
By
\eqref{eq:vbkl-strictness-fractional-exponent},
one has
$a^L<\alpha_L$.
Hence the right-hand side of
\eqref{eq:vbkl-strictness-final-contradiction}
converges to zero as
$t\downarrow0$,
whereas
$c_{L+1}>0$.
This contradiction proves
$v_{L+1}\notin\VL$.
Together with
\eqref{eq:vbkl-strictness-upper-space-membership},
we conclude that $\VL
\subsetneq
\mV^{(L+1)}$.
\end{noheadproof}
\hfill\BlackBox

\subsection{Proof of \texorpdfstring{\Cref{thm:variation-bkl-main-generalization}}{}}
\label{app:proof-finite-architecture-rademacher}

\begin{noheadproof}
Fix $L\ge2$, $m\ge1$, $G\ge2$, $R>0$, $n\ge1$, and sample points $\b x_1,\ldots,\b x_n\in\mX$. For notational convenience, set $\mathcal A:=\mathcal A_{L-1}^{m,G}$, $\mathcal D:=\mathcal D_L^{m,G}$, $\mathcal W:=\mathcal W_R^{L;m,G}$,
\begin{align}
P&:=P_{L-1,m,G},
\label{eq:main-proof-parameter-count}
\\
\Gamma&:=\Gamma_{L-1,m,G,n}
=1+\left(P+1\right)\ln\left(e\left(P+1\right)\right)\ln\left(en\right),
\label{eq:main-proof-complexity-factor}
\\
B_{\b x}&:=\sup_{a\in\mathcal A}\max_{i\in\left[n\right]}
\left|a\left(\b x_i\right)\right|.
\label{eq:main-proof-support-envelope}
\end{align}
The quantities in \eqref{eq:main-proof-parameter-count}, \eqref{eq:main-proof-complexity-factor}, and \eqref{eq:main-proof-support-envelope} agree with the quantities introduced in \eqref{eq:main-statistical-parameter-count}, \eqref{eq:finite-architecture-complexity-factor}, and \eqref{eq:main-finite-support-envelope}, respectively. We divide the proof into five steps. By \eqref{eq:finite-architecture-variation-ball} and \eqref{eq:finite-architecture-outer-dictionary}, the class $\mathcal W$ is the variation hull of radius $R$ generated by the symmetric dictionary $\mathcal D$. Therefore, \eqref{eq:variation-reduction-finite-architecture} of \Cref{lem:variation-reduction} gives
\begin{align}
\widehat{\mathfrak R}_n\left(\mathcal W\right)
= R\widehat{\mathfrak R}_n\left(\mathcal D\right).
\label{eq:main-proof-exact-variation-reduction}
\end{align}
The second inequality in \eqref{eq:variation-reduction-finite-architecture} in \Cref{lem:variation-reduction} gives the elementary envelope estimate
\begin{align}
\widehat{\mathfrak R}_n\left(\mathcal W\right)
\le RB_{\b x}^{1/2}.
\label{eq:main-proof-trivial-envelope-bound}
\end{align} 
This estimate will provide the first term in the minimum appearing in \eqref{eq:main-empirical-rademacher-bound}. By \eqref{eq:finite-architecture-dictionary-as-rkhs-union} of \Cref{lem:brownian-pullback-rademacher}, the outer dictionary satisfies
\begin{align}
\mathcal D = \mathcal D_1\left(\mathcal A\right).
\label{eq:main-proof-dictionary-identification}
\end{align}
For every $a\in\mathcal A$ and $i\in\left[n\right]$, the Brownian pullback kernel satisfies
\begin{align}
k_a\left(\b x_i,\b x_i\right)
&\stackrel{(a)}{=}
\kB\left(a\left(\b x_i\right),a\left(\b x_i\right)\right)
\stackrel{(b)}{=}
\left|a\left(\b x_i\right)\right|
\stackrel{(c)}{\le}
B_{\b x}.
\label{eq:main-proof-diagonal-control}
\end{align}
Here, (a) applies \eqref{eq:finite-architecture-pullback-kernel}, (b) uses $\kB\left(s,s\right)=\left|s\right|$, $s\in\mathbb R$, and (c) applies \eqref{eq:main-proof-support-envelope}. Taking the maximum over $i\in\left[n\right]$ and then the supremum over $a\in\mathcal A$ in \eqref{eq:main-proof-diagonal-control} gives
\begin{align}
B_{\b x}\left(\mathcal A\right)
&\stackrel{(a)}{=}
\sup_{a\in\mathcal A}\max_{i\in\left[n\right]}
k_a\left(\b x_i,\b x_i\right)
\stackrel{(b)}{=}
B_{\b x}.
\label{eq:main-proof-diagonal-envelope-identity}
\end{align}
Here, (a) applies \eqref{eq:brownian-pullback-diagonal-envelope}, while (b) follows from the pointwise identity in \eqref{eq:main-proof-diagonal-control}. Applying \eqref{eq:brownian-pullback-rademacher-bound} of \Cref{lem:brownian-pullback-rademacher} with $r=1$, $\mathcal A=\mathcal A_{L-1}^{m,G}$, and using \eqref{eq:main-proof-dictionary-identification} and \eqref{eq:main-proof-diagonal-envelope-identity}, we obtain
\begin{align}
\widehat{\mathfrak R}_n\left(\mathcal D\right)
\le
\frac{1}{\sqrt n}
\left(
2\widehat{\mathfrak C}_{n,\b x}^{(B)}
\left(\mathcal A\right)
+B_{\b x}
\right)^{1/2}.
\label{eq:main-proof-dictionary-chaos-bound}
\end{align}
By \eqref{eq:brownian-chaos-from-entropy-bound} of \Cref{lem:brownian-chaos-from-entropy} and the abbreviations \eqref{eq:main-proof-complexity-factor} and \eqref{eq:main-proof-support-envelope}, there exists a constant $\widetilde C_L>0$, depending only on $L$, such that
\begin{align}
\widehat{\mathfrak C}_{n,\b x}^{(B)}\left(\mathcal A\right)
\le\widetilde C_LB_{\b x}\Gamma.
\label{eq:main-proof-chaos-entropy-bound}
\end{align}
Substituting \eqref{eq:main-proof-chaos-entropy-bound} into
\eqref{eq:main-proof-dictionary-chaos-bound} and factoring the nonnegative
quantity $B_{\b x}$ gives
\begin{align}
\widehat{\mathfrak R}_n\left(\mathcal D\right)
&\le
\frac{1}{\sqrt n}
\left(
2\widetilde C_LB_{\b x}\Gamma+B_{\b x}
\right)^{1/2}
=
\frac{B_{\b x}^{1/2}}{\sqrt n}
\left(
2\widetilde C_L\Gamma+1
\right)^{1/2}.
\label{eq:main-proof-dictionary-after-chaos}
\end{align}
The identity remains valid when $B_{\b x}=0$. By
\eqref{eq:main-proof-complexity-factor},
\begin{align}
\Gamma&\ge1.
\label{eq:main-proof-gamma-lower-bound}
\end{align}
Hence
\begin{align}
2\widetilde C_L\Gamma+1
&\le
\left(
2\widetilde C_L+1
\right)\Gamma.
\label{eq:main-proof-absorb-diagonal}
\end{align}
Define
\begin{align}
C_L&:=\left(2\widetilde C_L+1\right)^{1/2}.
\label{eq:main-proof-depth-constant}
\end{align}
Combining the preceding estimates yields
\begin{align}
\widehat{\mathfrak R}_n\left(\mathcal D\right)
&\le
C_LB_{\b x}^{1/2}
\left(
\frac{\Gamma}{n}
\right)^{1/2}.
\label{eq:main-proof-refined-dictionary-bound}
\end{align}
The exact variation reduction
\eqref{eq:main-proof-exact-variation-reduction} therefore gives
\begin{align}
\widehat{\mathfrak R}_n\left(\mathcal W\right)
&=
R\widehat{\mathfrak R}_n\left(\mathcal D\right)
\le
C_LRB_{\b x}^{1/2}
\left(
\frac{\Gamma}{n}
\right)^{1/2}.
\label{eq:main-proof-refined-variation-bound}
\end{align}
We have therefore established both
\begin{align}
\widehat{\mathfrak R}_n\left(\mathcal W\right)
&\le RB_{\b x}^{1/2}
\label{eq:main-proof-first-competing-bound}
\end{align}
from \eqref{eq:main-proof-trivial-envelope-bound}, and
\begin{align}
\widehat{\mathfrak R}_n\left(\mathcal W\right)
&\le
C_LRB_{\b x}^{1/2}
\left(
\frac{\Gamma}{n}
\right)^{1/2}
\label{eq:main-proof-second-competing-bound}
\end{align}
from \eqref{eq:main-proof-refined-variation-bound}. Combining \eqref{eq:main-proof-first-competing-bound} and
\eqref{eq:main-proof-second-competing-bound}, and factoring the common
nonnegative term, gives
\begin{align}
\widehat{\mathfrak R}_n\left(\mathcal W\right)
&\le
RB_{\b x}^{1/2}
\min\left\{
1,
\;
C_L
\left(
\frac{\Gamma}{n}
\right)^{1/2}
\right\}.
\label{eq:main-proof-empirical-minimum-bound}
\end{align} Substituting $\mathcal W=\mathcal W_R^{L;m,G}$, $\Gamma=\Gamma_{L-1,m,G,n}$ in \eqref{eq:main-proof-empirical-minimum-bound} proves \eqref{eq:main-empirical-rademacher-bound}. By \Cref{lem:finite-bkl-parameter-count}\ref{lem:finite-architecture-range-control}, every $a\in\mathcal A$ satisfies
\begin{align}
\left|a\left(\b x\right)\right|
&\le
A_{\mX}^{\,2^{-(L-2)}},
\qquad
\b x\in\mX.
\label{eq:main-proof-finite-support-global-bound}
\end{align}
Consequently,
\begin{align}
B_{\b x}
&\stackrel{(a)}{=}
\sup_{a\in\mathcal A}
\max_{i\in\left[n\right]}
\left|a\left(\b x_i\right)\right|
\stackrel{(b)}{\le}
A_{\mX}^{\,2^{-(L-2)}}.
\label{eq:main-proof-deterministic-support-envelope}
\end{align}
Here, (a) applies \eqref{eq:main-proof-support-envelope}, and (b) applies \eqref{eq:main-proof-finite-support-global-bound} to every sample point. This proves \eqref{eq:main-finite-support-deterministic-envelope}. Since $A_{\mX}\ge1$, and therefore $A_{\mX}>0$, taking square roots in \eqref{eq:main-proof-deterministic-support-envelope} gives
\begin{align}
B_{\b x}^{1/2}
&\stackrel{(a)}{\le}
\left(
A_{\mX}^{\,2^{-(L-2)}}
\right)^{1/2}
\stackrel{(b)}{=}
A_{\mX}^{\,2^{-(L-1)}}.
\label{eq:main-proof-deterministic-outer-envelope}
\end{align}
Here, (a) uses the monotonicity of the square-\allowbreak root function on $\left[0,\infty\right)$, while (b) uses $\frac{1}{2}2^{-(L-2)}\allowbreak=\allowbreak 2^{-(L-1)}$. Let $X_1,\ldots,X_n$ be independent random variables with common distribution $\nu$. Applying \eqref{eq:main-proof-empirical-minimum-bound} to the realized sample and then using \eqref{eq:main-proof-deterministic-outer-envelope} gives
\begin{align}
\widehat{\mathfrak R}_n\left(\mathcal W_R^{L;m,G}\right)
&\stackrel{(a)}{\le}
RB_{\b X}^{1/2}
\min\left\{
1,
\;
C_L
\left(
\frac{\Gamma_{L-1,m,G,n}}{n}
\right)^{1/2}
\right\}
\notag\\
&\stackrel{(b)}{\le}
RA_{\mX}^{\,2^{-(L-1)}}
\min\left\{
1,
\;
C_L
\left(
\frac{\Gamma_{L-1,m,G,n}}{n}
\right)^{1/2}
\right\}.
\label{eq:main-proof-random-sample-deterministic-bound}
\end{align}
Here, (a) applies \eqref{eq:main-proof-empirical-minimum-bound} to the random sample, and (b) applies \eqref{eq:main-proof-deterministic-outer-envelope}. The right-hand side of \eqref{eq:main-proof-random-sample-deterministic-bound} is deterministic. Taking expectation with respect to $X_1,\ldots,X_n$ and applying \eqref{eq:expected-rademacher-complexity} therefore yields
\begin{align}
\mathfrak R_n\left(\mathcal W_R^{L;m,G}\right)
&\stackrel{(a)}{=}
\mathbb E_{\b X}
\left[
\widehat{\mathfrak R}_n
\left(\mathcal W_R^{L;m,G}\right)
\right]
 \stackrel{(b)}{\le}
\mathbb E_{\b X}
\left[
RA_{\mX}^{\,2^{-(L-1)}}
\min\left\{
1,
\;
C_L
\left(
\frac{\Gamma_{L-1,m,G,n}}{n}
\right)^{1/2}
\right\}
\right]
\notag\\
&\stackrel{(c)}{=}
RA_{\mX}^{\,2^{-(L-1)}}
\min\left\{
1,
\;
C_L
\left(
\frac{\Gamma_{L-1,m,G,n}}{n}
\right)^{1/2}
\right\}.
\label{eq:main-proof-expected-rademacher-final}
\end{align}
Here, (a) applies the definition of expected Rademacher complexity, (b) applies \eqref{eq:main-proof-random-sample-deterministic-bound}, and (c) uses the fact that every quantity inside the expectation is deterministic. Equation \eqref{eq:main-proof-expected-rademacher-final} is precisely \eqref{eq:main-rademacher-bound}. This completes the proof.

\end{noheadproof}
\hfill\BlackBox

\subsection{Proof of
\texorpdfstring{\Cref{thm:path-bkl-approx}}{}}
\begin{noheadproof}
Set
\begin{align}
R_L
:=
\left(
\int_{\mX}
\left\|
\b x
\right\|_2^{\,2^{-(L-2)}}
\,\d\nu\left(
\b x
\right)
\right)^{1/2}.
\end{align}
By
\Cref{lem:path-atom-bound},
every atom
$u\in\mU_L$
satisfies
\begin{align}
\left\|
u
\right\|_{L^2\left(\nu\right)}
\le
R_L.
\label{eq:main-approx-atom-radius}
\end{align}
Moreover,
\Cref{lem:atomic-bkl-dictionary-compact}
shows that
$\mU_L$
is compact in
$L^2\left(\nu\right)$.
Therefore,
\Cref{prop:variation-gauge-attainment}
provides a finite signed Radon measure
$\mu_F$
on
$\mU_L$
such that
\begin{align}
F
&=
\int_{\mU_L}
u
\,\d\mu_F\left(
u
\right)
\qquad
\text{in }
L^2\left(
\nu
\right),
\label{eq:main-approx-attaining-representation}
\\
\left\|
\mu_F
\right\|_{\mathrm{TV}}
&=
\CLv\left(
F
\right).
\label{eq:main-approx-attaining-norm}
\end{align}
Suppose first that
$\CLv\left( F \right) = 0$.
Then
\eqref{eq:main-approx-attaining-norm}
gives
$\left\| \mu_F \right\|_{\mathrm{TV}} = 0$,
and hence
$\mu_F = 0$.
Substituting this identity into
\eqref{eq:main-approx-attaining-representation}
gives
$F = 0$ in $L^2\left( \nu \right)$.
Choosing
$F_m := 0$
gives
$F_m \in \mV_m^{(L)}$
through the zero measure, whose support is empty, and proves the
claimed estimate in this case. We may therefore assume that
\begin{align}
\rho
:=
\left\|
\mu_F
\right\|_{\mathrm{TV}}
=
\CLv\left(
F
\right)
>
0.
\label{eq:main-approx-positive-radius}
\end{align}
Let
$\left|\mu_F\right|$
denote the total-variation measure of
$\mu_F$.
By the polar decomposition of finite real signed measures, there
exists a measurable function
\begin{align}
\sigma
:
\mU_L
\longrightarrow
\left\{
-1,1
\right\}
\label{eq:main-approx-polar-sign}
\end{align}
such that
\begin{align}
\d\mu_F\left(
u
\right)
=
\sigma\left(
u
\right)
\,\d\left|\mu_F\right|\left(
u
\right).
\label{eq:main-approx-polar-decomposition}
\end{align}
The function
$\sigma$
may be defined arbitrarily on a
$\left|\mu_F\right|$-null set so that
\eqref{eq:main-approx-polar-sign}
holds on all of
$\mU_L$. Define
\begin{align}
P_F
:=
\frac{
\left|\mu_F\right|
}{
\rho
}.
\label{eq:main-approx-probability-measure}
\end{align}
Then
$P_F$
is a probability measure on
$\mU_L$ because
\begin{align}
P_F\left(
\mU_L
\right)
&\stackrel{(a)}{=}
\frac{
\left|\mu_F\right|\left(
\mU_L
\right)
}{
\rho
}
\stackrel{(b)}{=}
\frac{
\left\|
\mu_F
\right\|_{\mathrm{TV}}
}{
\rho
}
\stackrel{(c)}{=}
1.
\end{align}
Here,
(a) applies
\eqref{eq:main-approx-probability-measure},
(b) applies the definition of the total-variation norm, and
(c) applies
\eqref{eq:main-approx-positive-radius}. Let
$U_1,\ldots,U_m$
be independent
$\mU_L$-valued random variables with common law
$P_F$,
and define
\begin{align}
Z_i
:=
\sigma\left(
U_i
\right)
U_i,
\qquad
i\in\left[m\right].
\label{eq:main-approx-random-signed-atoms}
\end{align}
Since
$\mU_L$
is compact in
$L^2\left(\nu\right)$,
it is separable.
The map
$u \longmapsto \sigma\left( u \right)u$
is measurable and has its range in the separable set
$\mU_L \cup \left( -\mU_L \right)$.
Consequently,
each
$Z_i$
is strongly measurable as an
$L^2\left(\nu\right)$-valued random variable.
Moreover,
\eqref{eq:main-approx-atom-radius}
gives
\begin{align}
\left\|
Z_i
\right\|_{L^2\left(\nu\right)}
&\stackrel{(a)}{=}
\left|
\sigma\left(
U_i
\right)
\right|
\left\|
U_i
\right\|_{L^2\left(\nu\right)}
 \stackrel{(b)}{=}
\left\|
U_i
\right\|_{L^2\left(\nu\right)}
\stackrel{(c)}{\le}
R_L
\qquad
\text{almost surely}.
\label{eq:main-approx-Z-bound}
\end{align}
Here,
(a) applies the absolute homogeneity of the
$L^2\left(\nu\right)$
norm,
(b) uses
$\left|\sigma\left(U_i\right)\right|=1$,
and
(c) applies
\eqref{eq:main-approx-atom-radius}.
Thus,
$Z_i$
is Bochner integrable and square-integrable. The Bochner expectation of
$Z_i$
satisfies
\begin{align}
\mathbb E
\left[
Z_i
\right]
&\stackrel{(a)}{=}
\int_{\mU_L}
\sigma\left(
u
\right)u
\,\d P_F\left(
u
\right)
\stackrel{(b)}{=}
\frac{1}{\rho}
\int_{\mU_L}
\sigma\left(
u
\right)u
\,\d\left|\mu_F\right|\left(
u
\right)
\stackrel{(c)}{=}
\frac{1}{\rho}
\int_{\mU_L}
u
\,\d\mu_F\left(
u
\right)
\notag\\
&\stackrel{(d)}{=}
\frac{
F
}{
\rho
}
\qquad
\text{in }
L^2\left(
\nu
\right).
\label{eq:main-approx-Z-mean}
\end{align}
Here,
(a) uses the law of
$U_i$
and
\eqref{eq:main-approx-random-signed-atoms},
(b) substitutes
\eqref{eq:main-approx-probability-measure},
(c) applies
\eqref{eq:main-approx-polar-decomposition},
and
(d) applies
\eqref{eq:main-approx-attaining-representation}. Define the random finite approximation
\begin{align}
\widetilde F_m
:=
\frac{
\rho
}{
m
}
\sum_{i=1}^{m}
Z_i.
\label{eq:main-approx-random-approximant}
\end{align}
By
\eqref{eq:main-approx-Z-mean},
\begin{align}
\mathbb E
\left[
\widetilde F_m
\right]
&\stackrel{(a)}{=}
\frac{
\rho
}{
m
}
\sum_{i=1}^{m}
\mathbb E
\left[
Z_i
\right]
 \stackrel{(b)}{=}
\frac{
\rho
}{
m
}
\sum_{i=1}^{m}
\frac{
F
}{
\rho
}
\stackrel{(c)}{=}
F
\qquad
\text{in }
L^2\left(
\nu
\right).
\end{align}
Here,
(a) uses the linearity of the Bochner expectation,
(b) applies
\eqref{eq:main-approx-Z-mean},
and
(c) evaluates the sum. Set
\begin{align}
Y_i
:=
Z_i
-
\frac{
F
}{
\rho
},
\qquad
i\in\left[m\right].
\label{eq:main-approx-centered-variables}
\end{align}
Then
$Y_1,\ldots,Y_m$
are independent, square-integrable, and centered:
\begin{align}
\mathbb E
\left[
Y_i
\right]
&\stackrel{(a)}{=}
\mathbb E
\left[
Z_i
\right]
-
\frac{
F
}{
\rho
}
\stackrel{(b)}{=}
0.
\label{eq:main-approx-centered-means}
\end{align}
Here,
(a) applies
\eqref{eq:main-approx-centered-variables},
and
(b) applies
\eqref{eq:main-approx-Z-mean}. Moreover,
\begin{align}
F
-
\widetilde F_m
&\stackrel{(a)}{=}
F
-
\frac{
\rho
}{
m
}
\sum_{i=1}^{m}
Z_i
 \stackrel{(b)}{=}
-\frac{
\rho
}{
m
}
\sum_{i=1}^{m}
\left(
Z_i
-
\frac{
F
}{
\rho
}
\right)
 \stackrel{(c)}{=}
-\frac{
\rho
}{
m
}
\sum_{i=1}^{m}
Y_i.
\label{eq:main-approx-centered-error}
\end{align}
Here,
(a) applies
\eqref{eq:main-approx-random-approximant},
(b) uses
$F = \frac{ \rho }{ m } \sum_{i=1}^{m} \frac{ F }{ \rho }$,
and
(c) applies
\eqref{eq:main-approx-centered-variables}. For
$i\neq j$,
independence and
\eqref{eq:main-approx-centered-means}
give
\begin{align}
\mathbb E
\left[
\left\langle
Y_i,
Y_j
\right\rangle_{L^2\left(\nu\right)}
\right]
&\stackrel{(a)}{=}
\left\langle
\mathbb E
\left[
Y_i
\right],
\mathbb E
\left[
Y_j
\right]
\right\rangle_{L^2\left(\nu\right)}
\stackrel{(b)}{=}
0.
\label{eq:main-approx-cross-terms}
\end{align}
Here,
(a) follows from independence and Fubini's theorem for the
square-integrable Hilbert-space-valued variables, and
(b) applies
\eqref{eq:main-approx-centered-means}. Using
\eqref{eq:main-approx-centered-error},
we therefore obtain
\begin{align}
&
\mathbb E
\left[
\left\|
F
-
\widetilde F_m
\right\|_{L^2\left(\nu\right)}^2
\right]
\stackrel{(a)}{=}
\frac{
\rho^2
}{
m^2
}
\mathbb E
\left[
\left\|
\sum_{i=1}^{m}
Y_i
\right\|_{L^2\left(\nu\right)}^2
\right]
\notag\\
&\stackrel{(b)}{=}
\frac{
\rho^2
}{
m^2
}
\left[
\sum_{i=1}^{m}
\mathbb E
\left[
\left\|
Y_i
\right\|_{L^2\left(\nu\right)}^2
\right]
+
2
\sum_{1\le i<j\le m}
\mathbb E
\left[
\left\langle
Y_i,
Y_j
\right\rangle_{L^2\left(\nu\right)}
\right]
\right]
\stackrel{(c)}{=}
\frac{
\rho^2
}{
m^2
}
\sum_{i=1}^{m}
\mathbb E
\left[
\left\|
Y_i
\right\|_{L^2\left(\nu\right)}^2
\right]
\notag\\
&\stackrel{(d)}{=}
\frac{
\rho^2
}{
m^2
}
\sum_{i=1}^{m}
\left[
\mathbb E
\left[
\left\|
Z_i
\right\|_{L^2\left(\nu\right)}^2
\right]
-
\left\|
\mathbb E
\left[
Z_i
\right]
\right\|_{L^2\left(\nu\right)}^2
\right]
\stackrel{(e)}{\le}
\frac{
\rho^2
}{
m^2
}
\sum_{i=1}^{m}
\mathbb E
\left[
\left\|
Z_i
\right\|_{L^2\left(\nu\right)}^2
\right]
\notag\\
&\stackrel{(f)}{\le}
\frac{
\rho^2
}{
m^2
}
\sum_{i=1}^{m}
R_L^2
\stackrel{(g)}{=}
\frac{
\rho^2R_L^2
}{
m
}.
\label{eq:main-approx-expected-error}
\end{align}
Here,
(a) applies
\eqref{eq:main-approx-centered-error},
(b) expands the squared Hilbert-space norm,
(c) applies
\eqref{eq:main-approx-cross-terms},
(d) applies the Hilbert-space variance identity to
\eqref{eq:main-approx-centered-variables},
(e) discards the nonnegative squared-mean term,
(f) applies
\eqref{eq:main-approx-Z-bound},
and
(g) evaluates the sum. Since the random variable
$\left\| F - \widetilde F_m \right\|_{L^2\left(\nu\right)}^2$
is nonnegative, the expectation estimate
\eqref{eq:main-approx-expected-error}
implies that there exists a realization\linebreak
$u_1,\ldots,u_m \allowbreak\in \mU_L$
of
$U_1,\ldots,U_m$
such that
\begin{align}
\left\|
F
-
F_m
\right\|_{L^2\left(\nu\right)}^2
\le
\frac{
\rho^2R_L^2
}{
m
},
\label{eq:main-approx-deterministic-squared-error}
\end{align}
where
\begin{align}
F_m
:=
\frac{
\rho
}{
m
}
\sum_{i=1}^{m}
\sigma\left(
u_i
\right)
u_i.
\end{align}
Indeed, if
\eqref{eq:main-approx-deterministic-squared-error}
failed for every realization, then the expectation in
\eqref{eq:main-approx-expected-error}
would be strictly larger than
$\rho^2R_L^2/m$. Define the finite signed discrete measure
\begin{align}
\mu_m
:=
\frac{
\rho
}{
m
}
\sum_{i=1}^{m}
\sigma\left(
u_i
\right)
\delta_{u_i}.
\label{eq:main-approx-discrete-measure}
\end{align}
By the definition of $\mu_m$,
\begin{align}
F_m
&=
\int_{\mU_L}
u\,\d\mu_m(u),
&
\left|\operatorname{supp}(\mu_m)\right|
&\le
m,
\end{align}
where repeated sampled atoms can only reduce the number of distinct support
points.
Consequently, $F_m
\in
\mV_m^{(L)}$. Taking square roots in
\eqref{eq:main-approx-deterministic-squared-error} and using
\eqref{eq:main-approx-positive-radius} gives
\begin{align}
\left\|F-F_m\right\|_{L^2(\nu)}
\le
R_L\CLv(F)m^{-1/2}.
\label{eq:main-approx-final-bound-RL}
\end{align}
Finally, substituting the definition of $R_L$ into
\eqref{eq:main-approx-final-bound-RL} gives
\begin{align}
\left\|
F
-
F_m
\right\|_{L^2\left(\nu\right)}
\le
\left(
\int_{\mX}
\left\|
\b x
\right\|_2^{\,2^{-(L-2)}}
\,\d\nu\left(
\b x
\right)
\right)^{1/2}
\CLv\left(
F
\right)
m^{-1/2},
\end{align}
which proves the claim.
\end{noheadproof}
\hfill\BlackBox

\subsection{Proof of
\texorpdfstring{\Cref{thm:two-stage-path-profile-bkl}}{}}

\begin{noheadproof}
Set $\rho
:=
\CLv\left(
F
\right)$. If
$\rho=0$,
then
\Cref{cor:path-gauge-nondegenerate}
gives
$F = 0 \qquad \text{in } L^2\left( \nu \right)$.
In this case, the zero realization proves the error estimate and the
local sparsity assertion.
We may therefore assume that $\rho
>
0$. We first construct a finite atomic approximation.
By
\Cref{thm:path-bkl-approx},
applied with the number of atoms equal to
$M$,
there exists
$F_M \in \mV_M^{(L)}$
such that
\begin{align}
\left\|
F
-
F_M
\right\|_{L^2\left(\nu\right)}
\le
R_L
\rho
M^{-1/2}.
\label{eq:two-stage-path-error}
\end{align}
Moreover, the explicit construction in the proof of
\Cref{thm:path-bkl-approx}
shows that
$F_M$
may be chosen in the form
\begin{align}
F_M\left(
\b x
\right)
=
\frac{
\rho
}{
M
}
\sum_{j=1}^{M}
\sigma_j
u_j\left(
\b x
\right),
\qquad
\b x\in\mX,
\label{eq:two-stage-path-approximant}
\end{align}
where
$u_1,\ldots,u_M
\in
\mU_L$ and
$\sigma_1,\ldots,\sigma_M
\in
\left\{
-1,1
\right\}$. Fix
$j\in\left[M\right]$.
By the recursive unit-ball characterization in
\Cref{lem:vbkl-union-rkhs},
there exist $a_{u_j}
\in
\mU_{L-1}$ and $g_{u_j}
\in
\mH_{\kB}$ such that
\begin{align}
u_j\left(
\b x
\right)
&=
g_{u_j}
\left(
a_{u_j}\left(
\b x
\right)
\right),
\qquad
\b x\in\mX,
\label{eq:two-stage-selected-atom-representation}
\\
\left\|
g_{u_j}
\right\|_{\mH_{\kB}}
&\le
1.
\label{eq:two-stage-selected-profile-norm}
\end{align}
By
\Cref{lem:lower-level-atomic-range-bound},
\begin{align}
a_{u_j}\left(
\mX
\right)
\subseteq
\left[
-A_0,
A_0
\right].
\label{eq:two-stage-selected-support-range}
\end{align}
Since
$\rho>0$,
the degenerate case
$R_{\mX}=0$
cannot occur.
Indeed, if
$R_{\mX}=0$,
then
$\mX\subseteq\left\{\boldsymbol 0\right\}$,
and
\Cref{lem:path-atom-bound}
would imply that every atom in
$\mU_L$
vanishes on
$\mX$.
This would imply
$F=0$
and therefore $\rho=0$,
contrary to the assumption that
$\rho>0$.
Hence, $A_0>0$. Define the interpolated selected atom by
\begin{align}
u_{j,m}\left(
\b x
\right)
:=
\left(
\Pi_mg_{u_j}
\right)
\left(
a_{u_j}\left(
\b x
\right)
\right),
\qquad
\b x\in\mX.
\label{eq:two-stage-interpolated-atom}
\end{align}
Applying
\Cref{lem:profile-interpolation-composition}
with
$A = A_0, a = a_{u_j}, g = g_{u_j}$
gives
\begin{align}
\left\|
u_j
-
u_{j,m}
\right\|_{L^2\left(\nu\right)}
&\stackrel{(a)}{\le}
\left(
\frac{
A_0
}{
2
}
\right)^{1/2}
m^{-1/2}
\left\|
g_{u_j}
\right\|_{\mH_{\kB}}
 \stackrel{(b)}{\le}
\left(
\frac{
A_0
}{
2
}
\right)^{1/2}
m^{-1/2}.
\label{eq:two-stage-single-path-l2}
\end{align}
Here,
(a) applies
\eqref{eq:profile-interpolation-composition-bound},
using
\eqref{eq:two-stage-selected-support-range},
while
(b) applies
\eqref{eq:two-stage-selected-profile-norm}. The two-stage approximant from the theorem statement can be written as
\begin{align}
F_{M,m}
\left(
\b x
\right)
&\stackrel{(a)}{=}
\frac{
\rho
}{
M
}
\sum_{j=1}^{M}
\sigma_j
\left(
\Pi_mg_{u_j}
\right)
\left(
a_{u_j}\left(
\b x
\right)
\right)
 \stackrel{(b)}{=}
\frac{
\rho
}{
M
}
\sum_{j=1}^{M}
\sigma_j
u_{j,m}\left(
\b x
\right).
\label{eq:two-stage-approximant-proof}
\end{align}
Here,
(a) applies
\eqref{eq:two-stage-approximant},
and
(b) applies
\eqref{eq:two-stage-interpolated-atom}. Subtracting
\eqref{eq:two-stage-approximant-proof}
from
\eqref{eq:two-stage-path-approximant}
gives
\begin{align}
F_M
-
F_{M,m}
=
\frac{
\rho
}{
M
}
\sum_{j=1}^{M}
\sigma_j
\left(
u_j
-
u_{j,m}
\right).
\label{eq:two-stage-profile-difference}
\end{align}
Therefore,
\begin{align}
&\left\|
F_M
-
F_{M,m}
\right\|_{L^2\left(\nu\right)}
\stackrel{(a)}{=}
\left\|
\frac{
\rho
}{
M
}
\sum_{j=1}^{M}
\sigma_j
\left(
u_j
-
u_{j,m}
\right)
\right\|_{L^2\left(\nu\right)}
\stackrel{(b)}{\le}
\frac{
\rho
}{
M
}
\sum_{j=1}^{M}
\left|
\sigma_j
\right|
\left\|
u_j
-
u_{j,m}
\right\|_{L^2\left(\nu\right)}
\notag\\
&\stackrel{(c)}{=}
\frac{
\rho
}{
M
}
\sum_{j=1}^{M}
\left\|
u_j
-
u_{j,m}
\right\|_{L^2\left(\nu\right)}
\stackrel{(d)}{\le}
\frac{
\rho
}{
M
}
\sum_{j=1}^{M}
\left(
\frac{
A_0
}{
2
}
\right)^{1/2}
m^{-1/2}
\stackrel{(e)}{=}
\left(
\frac{
A_0
}{
2
}
\right)^{1/2}
\rho
m^{-1/2}.
\label{eq:two-stage-profile-error}
\end{align}
Here,
(a) applies
\eqref{eq:two-stage-profile-difference},
(b) applies the triangle inequality and absolute homogeneity in
$L^2\left(\nu\right)$,
(c) uses
$\left|\sigma_j\right|=1$,
(d) applies
\eqref{eq:two-stage-single-path-l2},
and
(e) evaluates the sum of
$M$
identical terms. The triangle inequality now gives
\begin{align}
&\left\|
F
-
F_{M,m}
\right\|_{L^2\left(\nu\right)}
\stackrel{(a)}{\le}
\left\|
F
-
F_M
\right\|_{L^2\left(\nu\right)}
+
\left\|
F_M
-
F_{M,m}
\right\|_{L^2\left(\nu\right)}
\notag\\
&\stackrel{(b)}{\le}
R_L
\rho
M^{-1/2}
+
\left(
\frac{
A_0
}{
2
}
\right)^{1/2}
\rho
m^{-1/2}
\stackrel{(c)}{=}
\CLv\left(
F
\right)
\left[
R_L
M^{-1/2}
+
\left(
\frac{
A_0
}{
2
}
\right)^{1/2}
m^{-1/2}
\right].
\end{align}
Here,
(a) applies the triangle inequality,
(b) combines
\eqref{eq:two-stage-path-error}
and
\eqref{eq:two-stage-profile-error},
and
(c) applies
definition of $\rho$.
This proves
\eqref{eq:two-stage-final-bound}. It remains to prove the local evaluation sparsity assertion.
Let $-A_0
=
t_0
<
t_1
<
\cdots
<
t_m
=
A_0$ be the uniform interpolation grid underlying
$\Pi_m$,
and let
$\psi_0,\ldots,\psi_m$
denote the corresponding continuous piecewise-linear hat basis
functions. Fix
$\b x\in\mX$
and
$j\in\left[M\right]$.
By
\eqref{eq:two-stage-selected-support-range},
$a_{u_j}\left( \b x \right) \in \left[ -A_0, A_0 \right]$.
Hence, there exists an index
$q_j\left( \b x \right) \in \left\{ 0,\ldots,m-1 \right\}$
such that
\begin{align}
a_{u_j}\left(
\b x
\right)
\in
\left[
t_{q_j\left(\b x\right)},
t_{q_j\left(\b x\right)+1}
\right].
\label{eq:two-stage-active-cell}
\end{align}
The nodal representation of the piecewise-linear interpolant gives
\begin{align}
\left(
\Pi_mg_{u_j}
\right)
\left(
a_{u_j}\left(
\b x
\right)
\right)
=
\sum_{k=0}^{m}
g_{u_j}\left(
t_k
\right)
\psi_k
\left(
a_{u_j}\left(
\b x
\right)
\right).
\label{eq:two-stage-nodal-interpolant}
\end{align}
Since each hat function
$\psi_k$
is supported only on the grid intervals adjacent to
$t_k$,
\eqref{eq:two-stage-active-cell}
implies
\begin{align}
\psi_k
\left(
a_{u_j}\left(
\b x
\right)
\right)
=
0,
\qquad
k
\notin
\left\{
q_j\left(
\b x
\right),
q_j\left(
\b x
\right)+1
\right\}.
\label{eq:two-stage-inactive-hat-functions}
\end{align}
This statement remains valid when
$a_{u_j}\left(\b x\right)$
is a grid knot: one may select either adjacent cell, and the number of
nonzero hat functions can only decrease. Substituting
\eqref{eq:two-stage-inactive-hat-functions}
into
\eqref{eq:two-stage-nodal-interpolant}
gives
\begin{align}
&
\left(
\Pi_mg_{u_j}
\right)
\left(
a_{u_j}\left(
\b x
\right)
\right)
\stackrel{(a)}{=}
g_{u_j}
\left(
t_{q_j\left(\b x\right)}
\right)
\psi_{q_j\left(\b x\right)}
\left(
a_{u_j}\left(
\b x
\right)
\right)
+
g_{u_j}
\left(
t_{q_j\left(\b x\right)+1}
\right)
\psi_{q_j\left(\b x\right)+1}
\left(
a_{u_j}\left(
\b x
\right)
\right).
\end{align}
Here,
(a) removes from
\eqref{eq:two-stage-nodal-interpolant}
all terms that vanish by
\eqref{eq:two-stage-inactive-hat-functions}. Therefore, for each fixed
$\b x\in\mX$
and
$j\in\left[M\right]$,
the evaluation
$\left( \Pi_mg_{u_j} \right) \left( a_{u_j}\left( \b x \right) \right)$
uses at most two active hat functions and at most two corresponding
nodal coefficients. Finally,
\eqref{eq:two-stage-approximant-proof}
shows that
\begin{align}
F_{M,m}\left(
\b x
\right)
=
\frac{
\CLv\left(
F
\right)
}{
M
}
\sum_{j=1}^{M}
\sigma_j
\left(
\Pi_mg_{u_j}
\right)
\left(
a_{u_j}\left(
\b x
\right)
\right).
\end{align}
Since each of the
$M$
summands contributes at most two active interpolation basis
functions, evaluating
$F_{M,m}\left(\b x\right)$
requires at most
$2M$
active outer-profile basis contributions.
This number is independent of the interpolation resolution
$m$,
which proves the local evaluation sparsity assertion.
\end{noheadproof}
\hfill\BlackBox

\subsection{Proof of
\texorpdfstring{\Cref{thm:profile-bkl-approximation}}{}}

\begin{noheadproof}
For every
$g\in\mH_{\kB}$
one has
\begin{align}
\left\|
g
-
\Pi_mg
\right\|_{L^\infty\left(\left[-A,A\right]\right)}
&\stackrel{(a)}{\le}
\left(
\frac{
A
}{
2m
}
\right)^{1/2}
\left\|
g
\right\|_{\mH_{\kB}}
 \stackrel{(b)}{=}
\left(
\frac{
A
}{
2
}
\right)^{1/2}
m^{-1/2}
\left\|
g
\right\|_{\mH_{\kB}}.
\end{align}
Here,
(a) applies
\eqref{eq:symmetric-profile-interpolation-bound} of \Cref{thm:profile-bkl-approximation-symmetric},
and
(b) uses $\left(
A/2m
\right)^{1/2}
=
\left(
A/2
\right)^{1/2}
m^{-1/2}$. This is precisely
\eqref{eq:profile-interpolation-symmetric-bound},
and therefore proves the proposition.

\end{noheadproof}
\hfill\BlackBox

\appendix
% Keep both the PDF page and the dependency graph upright, exactly like the
% surrounding JMLR pages.  The exact supplied graph is scaled uniformly to the
% normal text width; only invisible clickable regions are added.
\begin{figure}[H]
\centering
\input{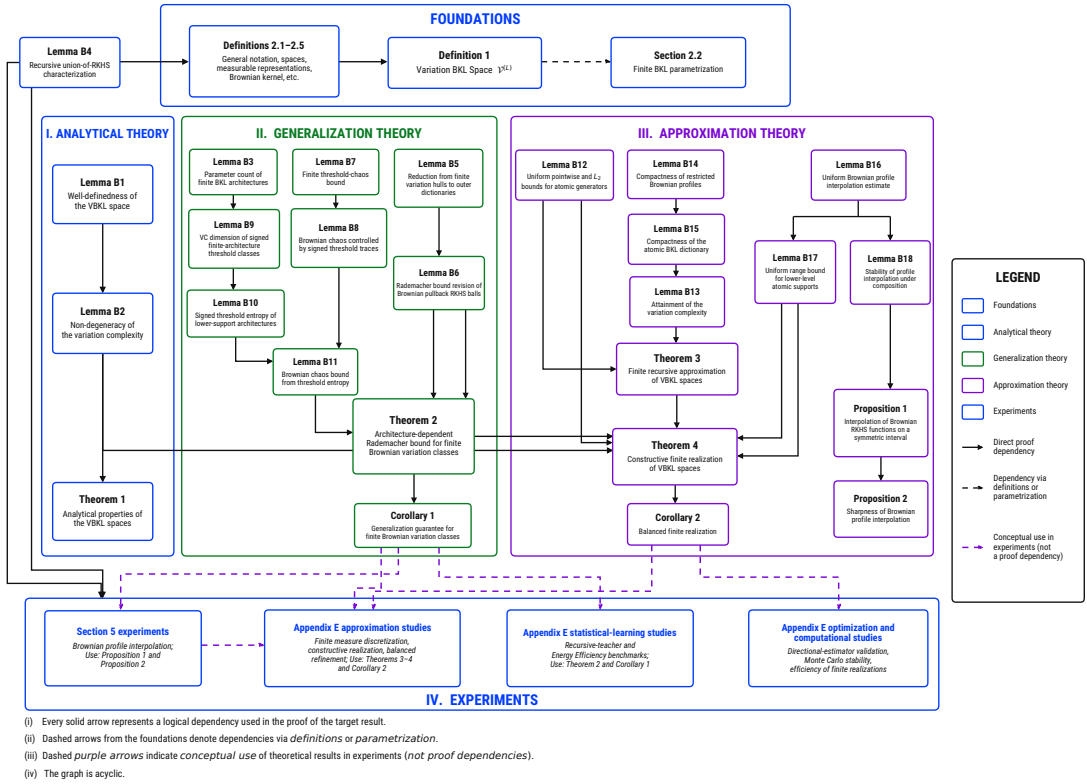}
\vspace{-0.35em}
\caption{
Proof-dependency structure of the theoretical results.
The displayed graph is the exact supplied dependency graph; every panel title, result node, and experiment node is clickable and jumps to the corresponding statement, section, or protocol.}
\label{fig:dependency_graph}
\end{figure}

\section{Experimental Protocols and Extended Results}
\label{app:experimental-details}
\label{app:additional-experiments}

This appendix provides the protocols and numerical details supporting
Section~\ref{sec:experiments}. The principal empirical findings and all
figures needed to assess them are retained in the main body; the material
below records the controlled constructions, validation-selected
configurations, exact numerical comparisons, and computational
measurements.

\subsection{Constructive-approximation protocol}
\label{app:constructive-protocol}

The constructive studies use a deterministic teacher
$F\in\mathcal{V}^{(L)}$ represented by an explicit finite signed measure on
recursive Brownian atoms. Each selected atom has an exact recursive support
and a deterministic unit-energy Brownian profile. Because this reference
representation is known, the approximation errors are computed directly
and do not contain training or statistical-estimation error. The three panels in Figure~\ref{fig:constructive_realization} isolate the
components of the construction. For signed-measure discretization, the
profiles are kept exact and only the number $M$ of sampled atoms is varied.
For profile discretization, the finite signed measure is held fixed and the
selected profiles are replaced by continuous piecewise-linear interpolants
of resolution $m$. For balanced refinement, both resolutions are set equal,
$M=m=N$. Monte Carlo means and percentile bands are used for the sampling
experiments, whereas the profile-discretization curve is deterministic.
The separate profile study in Figure~\ref{fig:profile-interpolation} uses
$m\in\{8,16,32,64,128,256,512\}$ and reports both typical fixed-profile
behaviour and the normalized tent profile that attains the sharp uniform
bound.

\subsection{Statistical-learning protocol and extended results}
\label{app:statistical-details}

The supervised studies comprise the recursive-teacher benchmark and the
Energy Efficiency benchmark. In each study, VBKL architecture selection is
performed separately for every training-set size and random seed using the
validation set. The selected configuration is then retrained and evaluated
on the held-out test set. Results are aggregated over five independent
seeds. DNVS, KRR, and RBF are evaluated under the same split protocol within
each benchmark.

\subsubsection{Recursive-teacher configurations}
\label{app:recursive-teacher-details}

The recursive-teacher target is generated using the same hierarchical
Brownian mechanism as the proposed model, providing a controlled matched
setting. Validation jointly selects the number of recursive Brownian paths
and the optimization horizon. Table~\ref{tab:vbkl_selected_architectures}
reports these choices in seed order. The variation across sample sizes and
seeds is intentional: it records the adaptive model-selection procedure
used to obtain the learning curve in
Figure~\ref{fig:recursive_learning}.

\begin{table}[t]
\centering

\small
\setlength{\tabcolsep}{5pt}

\begin{tabular}{rcccc}
\toprule
$n_{\mathrm{train}}$
&
Paths by seed
&
Median paths
&
Epochs by seed
&
Median epochs
\\
\midrule
100   & 32, 8, 2, 8, 16      &  8  & 100, 190, 600, 430, 30      & 190 \\
250   & 2, 64, 64, 32, 8      & 32  & 450, 70, 80, 190, 110        & 110 \\
500   & 64, 16, 16, 64, 32    & 32  & 140, 280, 90, 60, 110        & 110 \\
1,000 & 16, 64, 16, 4, 2      & 16  & 170, 70, 150, 380, 380       & 170 \\
2,000 & 4, 8, 8, 64, 32       &  8  & 1,990, 990, 440, 90, 110     & 440 \\
5,000 & 32, 4, 64, 16, 64     & 32  & 170, 510, 90, 570, 80        & 170 \\
\bottomrule
\end{tabular}

\caption{Validation-selected VBKL configurations. Path counts and training epochs are listed in seed order. The architecture and optimization horizon are selected independently for each training size and random seed using the validation set.}
\label{tab:vbkl_selected_architectures}
\end{table}
 
\subsubsection{Energy Efficiency numerical results}
\label{app:energy-details}

The Energy Efficiency experiment uses the publicly available regression
benchmark and the same train/validation/test protocol for all compared
methods. For VBKL, validation selects the number of paths, the Brownian
profile resolution, and the optimization horizon independently for every
training size and seed. Table~\ref{tab:energy-main} gives the exact means and
standard deviations underlying
Figure~\ref{fig:energy_learning_curve}. At $n=100$, the mean VBKL error is
close to the best reported value and lower than the DNVS mean. At the two
larger training sizes, VBKL has the largest mean error among the four
methods. The parameter-efficiency comparison in the main body should
therefore be read together with, rather than in place of, this predictive
comparison.

\begin{table}[t]
\centering
\small
\renewcommand{\arraystretch}{1.22}
\setlength{\tabcolsep}{7pt}
\begin{tabular}{rcccc}
\toprule
$n$ & VBKL & DNVS & KRR & RBF \\
\midrule
100
& \shortstack{0.051656\\(0.025140)}
& \shortstack{0.092731\\(0.013174)}
& \shortstack{0.059793\\(0.031508)}
& \shortstack{0.050172\\(0.022180)} \\
250
& \shortstack{0.007220\\(0.004147)}
& \shortstack{0.005756\\(0.001281)}
& \shortstack{0.005158\\(0.001125)}
& \shortstack{0.006908\\(0.001264)} \\
500
& \shortstack{0.004793\\(0.003983)}
& \shortstack{0.002442\\(0.000232)}
& \shortstack{0.002627\\(0.000309)}
& \shortstack{0.002636\\(0.000276)} \\
\bottomrule
\end{tabular}

\caption{Test MSE on the Energy Efficiency benchmark. Values are means, with standard deviations in parentheses, over five random seeds.}
\label{tab:energy-main}
\end{table}
 
\subsection{Optimization protocol}
\label{app:optimization-details}

The consistency experiment uses automatic differentiation as the reference
for directional derivatives and compares it with the finite-difference
averaged directional estimator at interaction scales
$h\in\{10^{-1},10^{-2},10^{-3},10^{-4}\}$. The Monte Carlo stability study
varies the number of sampled directions from $2^0$ to $2^7$ and records the
standard deviation of the resulting estimator. As shown in
Figure~\ref{fig:optimization_validation}, the first experiment exhibits the
usual finite-difference scale trade-off, while the second shows systematic
variance reduction as the number of directions increases.

\subsection{Computational characteristics}
\label{app:computational-details}

The recursive-learning benchmark also records training time, inference time
for the complete held-out test set, and fitted or trainable parameter count.
Table~\ref{tab:computational_characteristics} summarizes these quantities
over all training sizes and seeds. Tables~\ref{tab:recursive-learning-training-time},
\ref{tab:recursive-learning-inference-time}, and
\ref{tab:recursive-learning-parameters} provide the corresponding
training-size-specific results. Times are reported as means with standard
deviations over five seeds. Parameter counts are reported in the same format
when they vary across selected configurations.

The measurements show that the finite VBKL models remain computationally
tractable at the benchmark scale, although they are not uniformly the
fastest method. DNVS has lower inference time throughout this comparison,
and the feature and kernel baselines have lower training time for the tested
sample sizes. The principal computational distinction of VBKL is instead
that its recursive representation remains explicit and, relative to DNVS,
substantially more compact in the limited-data Energy Efficiency comparison.

\begin{table}[H]
\centering

\small
\setlength{\tabcolsep}{6pt}

\begin{tabular}{lccc}
\toprule
Model & Training time (s) & Inference time (ms) & Parameters \\
\midrule
VBKL & 2.797 $\pm$ 2.189 & 13.104 $\pm$ 11.204 & 3,248 \\
DNVS & 1.338 $\pm$ 0.232 & 0.882 $\pm$ 0.523 & 8,960 \\
KRR  & 0.103 $\pm$ 0.180 & 48.811 $\pm$ 60.318 &   750 \\
RBF  & 0.007 $\pm$ 0.009 &  8.413 $\pm$  4.517 &     0 \\
\bottomrule
\end{tabular}

\caption{Computational characteristics on the recursive-learning benchmark. Training and inference times are reported as mean $\pm$ standard deviation over all training sizes and seeds. Parameter counts are reported by their median because the validation-selected VBKL architecture may vary across runs.}
\label{tab:computational_characteristics}
\end{table}
\begin{table}[H]
\centering
\footnotesize
\renewcommand{\arraystretch}{1.35}
\setlength{\tabcolsep}{6.5pt}
\begin{tabular}{lcccccc}
\toprule
Model & \multicolumn{6}{c}{Training samples} \\
\cmidrule(lr){2-7}
 & $100$ & $250$ & $500$ & $1000$ & $2000$ & $5000$ \\
\midrule
\textbf{VBKL} & \shortstack[c]{1.086\\{\footnotesize(0.0698)}} & \shortstack[c]{1.912\\{\footnotesize(0.0308)}} & \shortstack[c]{2.551\\{\footnotesize(0.0543)}} & \shortstack[c]{1.471\\{\footnotesize(0.0092)}} & \shortstack[c]{3.878\\{\footnotesize(0.0442)}} & \shortstack[c]{5.882\\{\footnotesize(0.0064)}} \\
\textbf{DNVS} & \shortstack[c]{1.123\\{\footnotesize(0.0057)}} & \shortstack[c]{1.181\\{\footnotesize(0.0096)}} & \shortstack[c]{1.353\\{\footnotesize(0.0142)}} & \shortstack[c]{1.413\\{\footnotesize(0.0271)}} & \shortstack[c]{1.435\\{\footnotesize(0.0364)}} & \shortstack[c]{1.523\\{\footnotesize(0.0003)}} \\
\textbf{KRR} & \underline{\shortstack[c]{0.015\\{\footnotesize(0.0023)}}} & \underline{\shortstack[c]{0.005\\{\footnotesize(0)}}} & \underline{\shortstack[c]{0.007\\{\footnotesize(0.0001)}}} & \underline{\shortstack[c]{0.018\\{\footnotesize(0.0005)}}} & \underline{\shortstack[c]{0.078\\{\footnotesize(0.0034)}}} & \underline{\shortstack[c]{0.492\\{\footnotesize(0.0002)}}} \\
\textbf{RBF} & \textbf{\shortstack[c]{0.002\\{\footnotesize(0.0001)}}} & \textbf{\shortstack[c]{0.002\\{\footnotesize(0)}}} & \textbf{\shortstack[c]{0.003\\{\footnotesize(0.0001)}}} & \textbf{\shortstack[c]{0.005\\{\footnotesize(0.0005)}}} & \textbf{\shortstack[c]{0.009\\{\footnotesize(0.0007)}}} & \textbf{\shortstack[c]{0.020\\{\footnotesize(0.0013)}}} \\
\bottomrule
\end{tabular}
\vspace{1mm}
\begin{minipage}{0.98\textwidth}
\footnotesize Values are mean, with standard deviation shown in parentheses over five random seeds. The best result in each column is shown in bold and the second-best result is underlined.
\end{minipage}

\caption{Training time in seconds on the recursive stochastic teacher benchmark. Lower values are better.}
\label{tab:recursive-learning-training-time}
\end{table}
\begin{table}[H]
\centering
\footnotesize
\renewcommand{\arraystretch}{1.35}
\setlength{\tabcolsep}{6.5pt}
\begin{tabular}{lcccccc}
\toprule
Model & \multicolumn{6}{c}{Training samples} \\
\cmidrule(lr){2-7}
 & $100$ & $250$ & $500$ & $1000$ & $2000$ & $5000$ \\
\midrule
\textbf{VBKL} & \shortstack[c]{0.0069\\{\footnotesize(0.0062)}} & \shortstack[c]{0.0161\\{\footnotesize(0.0137)}} & \shortstack[c]{0.0183\\{\footnotesize(0.0111)}} & \shortstack[c]{0.0096\\{\footnotesize(0.0117)}} & \shortstack[c]{0.0110\\{\footnotesize(0.0117)}} & \shortstack[c]{0.0167\\{\footnotesize(0.0125)}} \\
\textbf{DNVS} & \textbf{\shortstack[c]{0.0010\\{\footnotesize(0.0004)}}} & \textbf{\shortstack[c]{0.0011\\{\footnotesize(0.0004)}}} & \textbf{\shortstack[c]{0.0013\\{\footnotesize(0.0005)}}} & \textbf{\shortstack[c]{0.0011\\{\footnotesize(0.0005)}}} & \textbf{\shortstack[c]{0.0005\\{\footnotesize(0.0004)}}} & \textbf{\shortstack[c]{0.0003\\{\footnotesize(0)}}} \\
\textbf{KRR} & \underline{\shortstack[c]{0.0028\\{\footnotesize(0.0003)}}} & \underline{\shortstack[c]{0.0062\\{\footnotesize(0.0002)}}} & \shortstack[c]{0.0126\\{\footnotesize(0.0003)}} & \shortstack[c]{0.0314\\{\footnotesize(0.0068)}} & \shortstack[c]{0.0688\\{\footnotesize(0.0104)}} & \shortstack[c]{0.1711\\{\footnotesize(0.0101)}} \\
\textbf{RBF} & \shortstack[c]{0.0073\\{\footnotesize(0.0035)}} & \shortstack[c]{0.0108\\{\footnotesize(0.0038)}} & \underline{\shortstack[c]{0.0070\\{\footnotesize(0.0044)}}} & \underline{\shortstack[c]{0.0072\\{\footnotesize(0.0048)}}} & \underline{\shortstack[c]{0.0085\\{\footnotesize(0.0054)}}} & \underline{\shortstack[c]{0.0097\\{\footnotesize(0.0059)}}} \\
\bottomrule
\end{tabular}
\vspace{1mm}
\begin{minipage}{0.98\textwidth}
\footnotesize Values are mean, with standard deviation shown in parentheses over five random seeds. The best result in each column is shown in bold and the second-best result is underlined.
\end{minipage}

\caption{Inference time in seconds for the complete held-out test set. Lower values are better.}
\label{tab:recursive-learning-inference-time}
\end{table}
\begin{table}[H]
\centering
\footnotesize
\renewcommand{\arraystretch}{1.35}
\setlength{\tabcolsep}{6.5pt}
\begin{tabular}{lcccccc}
\toprule
Model & \multicolumn{6}{c}{Training samples} \\
\cmidrule(lr){2-7}
 & $100$ & $250$ & $500$ & $1000$ & $2000$ & $5000$ \\
\midrule
\textbf{VBKL} & \shortstack[c]{2,680\\{\footnotesize(2,360)}} & \shortstack[c]{6,902\\{\footnotesize(6,008)}} & \shortstack[c]{7,795\\{\footnotesize(4,926)}} & \shortstack[c]{4,141\\{\footnotesize(5,123)}} & \shortstack[c]{4,710\\{\footnotesize(5,148)}} & \shortstack[c]{7,308\\{\footnotesize(5,567)}} \\
\textbf{DNVS} & \shortstack[c]{12,723\\{\footnotesize(12,391)}} & \shortstack[c]{14,029\\{\footnotesize(11,334)}} & \shortstack[c]{24,166\\{\footnotesize(13,881)}} & \shortstack[c]{21,555\\{\footnotesize(17,457)}} & \shortstack[c]{8,806\\{\footnotesize(14,254)}} & \underline{2,432} \\
\textbf{KRR} & \underline{100} & \underline{250} & \underline{500} & \underline{1,000} & \underline{2,000} & 5,000 \\
\textbf{RBF} & \textbf{0} & \textbf{0} & \textbf{0} & \textbf{0} & \textbf{0} & \textbf{0} \\
\bottomrule
\end{tabular}
\vspace{1mm}
\begin{minipage}{0.98\textwidth}
\footnotesize Values are mean, with standard deviation shown in parentheses over five random seeds. The best result in each column is shown in bold and the second-best result is underlined.
\end{minipage}

\caption{Number of trainable or fitted parameters used by each method. Lower values indicate a more compact representation.}
\label{tab:recursive-learning-parameters}
\end{table}
 
\FloatBarrier

\section{Additional notation}
\label{appendix:notation}
The following standard notation is used throughout the proofs. For a normed vector space
$(Z,\|\cdot\|_Z)$
and
$r>0$,
let
$
B_r(Z)
:=
\{
z\in Z:
\|z\|_Z\le r
\},
\qquad
S_r(Z)
:=
\{
z\in Z:
\|z\|_Z=r
\}.
$
In particular,
$
B^d
:=
B_1(\R^d,\|\cdot\|_2)$, $S^{d-1}
:=
S_1(\R^d,\|\cdot\|_2).
$ For
$\b x\in\R^d$
and
$p\in[1,\infty)$,
we write
$
\|\b x\|_p
=
\left(
\sum_{i=1}^{d}|x_i|^p
\right)^{1/p}$, $\|\b x\|_\infty
=
\max_{i\in[d]}|x_i|,
$
and denote the transpose of
$\b x$
by
$\b x^\top$. We write
$\mC(\mX)$
for the space of continuous functions on
$\mX$
equipped with the supremum norm
$
\|f\|_\infty
=
\sup_{\b x\in\mX}|f(\b x)|.
$ For
$\alpha\in(0,1]$,
let
$C^{0,\alpha}(\Omega)$
denote the space of $\alpha$-Hölder continuous functions on $\Omega$, equipped with the norm
$
\|f\|_{C^{0,\alpha}(\Omega)}
=
\|f\|_\infty
+
[f]_{C^{0,\alpha}(\Omega)},
$
where
$
[f]_{C^{0,\alpha}(\Omega)}
=
\sup_{x\neq y}
\frac{|f(x)-f(y)|}
{\|x-y\|_2^\alpha}.
$ For a set
$A$,
we write
$
L^\infty(A)
$
for the vector space of bounded real-valued functions on
$A$,
equipped with the norm
$
\|f\|_{L^\infty(A)}
:=
\sup_{x\in A}|f(x)|
$. We write
$
L^2(\nu)
=
\left\{
F:\mX\to\R:
\int_{\mX}|F(\b x)|^2\,d\nu(\b x)<\infty
\right\}
$
for the Hilbert space of
$\nu$-square-integrable functions, equipped with the inner product
$
\langle F,G\rangle_{L^2(\nu)}
=
\int_{\mX}
F(\b x)G(\b x)\,d\nu(\b x),
$
and induced norm
$
\|F\|_{L^2(\nu)}
=
\langle F,F\rangle_{L^2(\nu)}^{1/2}.
$ Unless stated otherwise, vector-valued integrals are understood in the Bochner sense.
For a statement
$E$,
we write
$\mathbf1_E$
for its indicator function.

\section{Sharpness of the Brownian profile interpolation estimate}\label{sec:sharpness-brownian-interpolation}

The preceding proposition establishes the uniform interpolation estimate used
throughout the constructive approximation theory. The following theorem
shows that this estimate is sharp: neither the convergence rate nor the
constant can be improved.

\begin{prop}\textnormal{(Sharpness of Brownian profile interpolation)}\ 
\label{thm:sharp-brownian-interpolation}
Let
$A>0$
and
$m\in\mathbb N$,
$m\ge1$.
Let $-A
=
t_0
<
t_1
<
\cdots
<
t_m
=
A$
be the uniform interpolation grid on
$[-A,A]$,
and let
$\Pi_m$
denote the associated continuous piecewise-linear interpolation operator.
Then
\begin{align}
\sup_{
\substack{
g\in\mathcal H_{k^{(B)}}\\
\|g\|_{\mathcal H_{k^{(B)}}}\le1
}}
\|g-\Pi_mg\|_{L^\infty([-A,A])}
=
\left(
\frac{A}{2m}
\right)^{1/2}.
\label{eq:sharp-profile-interpolation}
\end{align}
Consequently, the interpolation estimate of
\Cref{thm:profile-bkl-approximation-symmetric}
is optimal:
neither the exponent
$1/2$
nor the constant
$\sqrt{A/2}$
can be improved. Moreover, for every
$m\ge1$,
there exists
$g_m\in\mathcal H_{k^{(B)}}$
such that
\begin{align}
\|g_m\|_{\mathcal H_{k^{(B)}}}
=
1,
\qquad
\|g_m-\Pi_mg_m\|_{L^\infty([-A,A])}
=
\left(
\frac{A}{2}
\right)^{1/2}
m^{-1/2}.
\end{align}

\end{prop}
\begin{proof}
By Proposition~1, it suffices to establish the reverse inequality. Let
$
h
:=
\frac{2A}{m}
$
denote the grid spacing.
Fix an interpolation interval
$[t_i,t_{i+1}]$,
let
$
s_i
:=
\frac{t_i+t_{i+1}}{2},
$
and define
$g_m$
through its weak derivative by
\begin{align}
g_m'(r)
:=
\frac{1}{\sqrt h}
\mathbf 1_{[t_i,s_i]}(r)
-
\frac{1}{\sqrt h}
\mathbf 1_{[s_i,t_{i+1}]}(r).
\end{align}
Finally, define
$g_m(t) := \int_0^t g_m'(r) \,\d r$.
Since
$g_m$
is absolutely continuous,
satisfies
$g_m(0)=0$,
and
$g_m'\in L^2(\mathbb R)$,
we obtain
$
g_m
\in
\mathcal H_{k^{(B)}}.
$
Moreover,
\begin{align}
\|g_m\|_{\mathcal H_{k^{(B)}}}^2
&\stackrel{(a)}{=}
\int_{\mathbb R}
|g_m'(r)|^2
\,\d r
\stackrel{(b)}{=}
\frac1h
\left(
|[t_i,s_i]|
+
|[s_i,t_{i+1}]|
\right)
\stackrel{(c)}{=}
\frac1h
\left(
\frac h2
+
\frac h2
\right)
=
1.
\end{align}
Here,
(a) is the definition of the Brownian RKHS norm,
(b) follows from the definition of
$g_m'$,
and
(c) follows from the definition of the midpoint
$s_i$. Next,
\begin{align}
g_m(t_{i+1})
-
g_m(t_i)
&\stackrel{(a)}{=}
\int_{t_i}^{t_{i+1}}
g_m'(r)
\,\d r
\stackrel{(b)}{=}
\frac1{\sqrt h}
\frac h2
-
\frac1{\sqrt h}
\frac h2
=
0.
\end{align}
Here,
(a) is the fundamental theorem of calculus,
while
(b) follows from the definition of
$g_m'$.
Hence,
$
g_m(t_i)
=
g_m(t_{i+1}).
$ Therefore,
the piecewise-linear interpolant is constant on
$[t_i,t_{i+1}]$,
that is,
\begin{align}
(\Pi_m g_m)(t)
=
g_m(t_i),
\qquad
t\in[t_i,t_{i+1}].
\label{eq:sharp-interpolant-constant}
\end{align}
Evaluating at the midpoint
$s_i$,
we obtain
\begin{align}
g_m(s_i)
-
(\Pi_m g_m)(s_i)
&\stackrel{(a)}{=}
g_m(s_i)
-
g_m(t_i)
\stackrel{(b)}{=}
\int_{t_i}^{s_i}
g_m'(r)
\,\d r
\stackrel{(c)}{=}
\frac1{\sqrt h}
\frac h2
=
\frac{\sqrt h}{2}.
\end{align}
Here,
(a) follows from
\eqref{eq:sharp-interpolant-constant},
(b) is the fundamental theorem of calculus,
and
(c) follows from the definition of
$g_m'$. Since
$h=2A/m$,
we conclude that
\begin{align}
\left|
g_m(s_i)
-
(\Pi_m g_m)(s_i)
\right|
=
\frac12
\left(
\frac{2A}{m}
\right)^{1/2}
=
\left(
\frac{A}{2m}
\right)^{1/2}.
\end{align}
Therefore,
\begin{align}
\|g_m-\Pi_mg_m\|_{L^\infty([-A,A])}
\ge
\left(
\frac{A}{2m}
\right)^{1/2}.
\label{eq:sharp-lower-bound}
\end{align}
On the other hand,
since
$\|g_m\|_{\mathcal H_{k^{(B)}}}=1$,
Proposition~1 yields
\begin{align}
\|g_m-\Pi_mg_m\|_{L^\infty([-A,A])}
\le
\left(
\frac{A}{2m}
\right)^{1/2}.
\label{eq:sharp-upper-bound}
\end{align}
Combining
\eqref{eq:sharp-lower-bound}
and
\eqref{eq:sharp-upper-bound},
we obtain
$\|g_m-\Pi_mg_m\|_{L^\infty([-A,A])} = \left( \frac{A}{2m} \right)^{1/2}$.
Taking the supremum over all
$g\in\mathcal H_{k^{(B)}}$
with
$\|g\|_{\mathcal H_{k^{(B)}}}\le1$
establishes the reverse inequality.
Together with Proposition~1,
this proves the theorem.
\end{proof}

\begin{table}[t]
\centering
\begin{tabular}{@{}lll@{}}
\toprule
Result & Content & Page \\
\midrule

Lemma~\ref{lem:path-atomic-space-well-defined}
&
Well-definedness of the VBKL space
&
page~\pageref{lem:path-atomic-space-well-defined}
\\

Lemma~\ref{cor:path-gauge-nondegenerate}
&
Non-degeneracy of the variation complexity
&
page~\pageref{cor:path-gauge-nondegenerate}
\\

Lemma~\ref{lem:finite-bkl-parameter-count}
&
Finite-profile characterization and finite-architecture bounds
&
page~\pageref{lem:finite-bkl-parameter-count}
\\

Lemma~\ref{lem:vbkl-union-rkhs}
&
Recursive union-of-RKHS characterization
&
page~\pageref{lem:vbkl-union-rkhs}
\\

Lemma~\ref{lem:variation-reduction}
&
Exact reduction from finite variation hulls to outer dictionaries
&
page~\pageref{lem:variation-reduction}
\\

Lemma~\ref{lem:brownian-pullback-rademacher}
&
Rademacher bound for unions of Brownian pullback RKHS balls
&
page~\pageref{lem:brownian-pullback-rademacher}
\\

Lemma~\ref{lem:self-contained-threshold-chaos}
&
Finite threshold-chaos bound
&
page~\pageref{lem:self-contained-threshold-chaos}
\\

Lemma~\ref{lem:threshold-traces-control-chaos}
&
Brownian chaos controlled by signed threshold traces
&
page~\pageref{lem:threshold-traces-control-chaos}
\\
Lemma~\ref{lem:finite-bkl-vc-dimension}
&
VC dimension of signed finite-architecture threshold classes
&
page~\pageref{lem:finite-bkl-vc-dimension}
\\
Lemma~\ref{lem:finite-bkl-threshold-entropy}
&
Signed threshold entropy of finite lower-support architectures
&
page~\pageref{lem:finite-bkl-threshold-entropy}
\\

Lemma~\ref{lem:brownian-chaos-from-entropy}
&
Brownian chaos bound from finite-architecture threshold entropy
&
page~\pageref{lem:brownian-chaos-from-entropy}
\\

Lemma~\ref{lem:path-atom-bound}
&
Uniform H\"older, pointwise, and $L^2$ bounds for atomic generators
&
page~\pageref{lem:path-atom-bound}
\\

Lemma~\ref{prop:variation-gauge-attainment}
&
Attainment of the variation complexity
&
page~\pageref{prop:variation-gauge-attainment}
\\

Lemma~\ref{lem:brownian-profile-compactness}
&
Compactness of restricted Brownian profiles
&
page~\pageref{lem:brownian-profile-compactness}
\\

Lemma~\ref{lem:atomic-bkl-dictionary-compact}
&
Compactness of the atomic VBKL dictionary
&
page~\pageref{lem:atomic-bkl-dictionary-compact}
\\

Lemma~\ref{thm:profile-bkl-approximation-symmetric}
&
Uniform Brownian profile interpolation estimate
&
page~\pageref{thm:profile-bkl-approximation-symmetric}
\\

Lemma~\ref{lem:lower-level-atomic-range-bound}
&
Uniform range bound for lower-level atomic supports
&
page~\pageref{lem:lower-level-atomic-range-bound}
\\

Lemma~\ref{lem:profile-interpolation-composition}
&
Stability of profile interpolation under composition
&
page~\pageref{lem:profile-interpolation-composition}
\\

\bottomrule
\end{tabular}

\caption{Overview of the supporting theoretical results. Their logical dependencies are illustrated in Fig.~\ref{fig:dependency_graph}; the main theoretical results are summarized in Table~\ref{tab:main-results}.}
\label{tab:auxiliary-results}
\end{table}

\section{Internal Lemmas}\label{app:auxi_lems}
\setcounter{lemma}{0}
\renewcommand{\thelemma}{B\arabic{lemma}}

In this section, we present our own lemmas used in \Cref{app}.

\begin{lemma}\textnormal{(Well-definedness of the measure-valued VBKL space)}\ 
\label{lem:path-atomic-space-well-defined}
Fix
$L\ge2$.
Let
$\left(
\mU_l
\right)_{l\ge1}$
be the recursive atomic dictionaries introduced in
\Cref{subsec:atomic-bkl},
and let
$\VL$
and
$\CLv$
be defined in
\eqref{eq:variation-bkl-space}
and
\eqref{eq:variation-bkl-complexity}. Assume that
$\mU_L\subseteq L^2\left(\nu\right)$
is equipped with a
$\sigma$-algebra for which the canonical inclusion
$
\iota
:
\mU_L
\longrightarrow
L^2\left(
\nu
\right)$, $\iota\left(
u
\right)
:=
u,
$
is strongly measurable.
Assume further that there exists
$R_L<\infty$
such that
\begin{align}
\sup_{u\in\mU_L}
\left\|
u
\right\|_{L^2\left(\nu\right)}
\le
R_L.
\label{eq:path-atom-uniform-bound}
\end{align}
Let
$\mu\in\mathcal M\left(\mU_L\right)$
be a finite signed measure satisfying
$
\left\|
\mu
\right\|_{\mathrm{TV}}
<
\infty.
$
Then the Bochner integral
\begin{align}
F
:=
\int_{\mU_L}
u
\,\d\mu\left(
u
\right)
\label{eq:variation-measure-representation}
\end{align}
is a well-defined element of
$L^2\left(\nu\right)$
and satisfies
\begin{align}
\left\|
F
\right\|_{L^2\left(\nu\right)}
\le
R_L
\left\|
\mu
\right\|_{\mathrm{TV}}.
\label{eq:path-atomic-l2-bound}
\end{align}
Consequently,
$\VL$
is a well-defined subset of
$L^2\left(\nu\right)$.
Furthermore, for every
$F\in\VL$,
the value
$
\CLv\left(
F
\right)
$
is a well-defined finite nonnegative real number.

\end{lemma}

\begin{proof}
By assumption, the canonical inclusion $\iota:\mU_L\to L^2(\nu)$ is strongly measurable. Moreover, the uniform bound \eqref{eq:path-atom-uniform-bound} implies that
\begin{align}
\int_{\mU_L}
\|u\|_{L^2(\nu)}
\d|\mu|(u)
\le
R_L|\mu|(\mU_L)
=
R_L\|\mu\|_{\mathrm{TV}}
<
\infty.
\end{align}
Thus, the map $u\mapsto u$ is strongly measurable and integrable in norm with respect to the total-variation measure $|\mu|$. It is therefore Bochner integrable with respect to the finite signed measure $\mu$, and the integral in \eqref{eq:variation-measure-representation} defines an element $F\in L^2(\nu)$. Then
\begin{align}
\|F\|_{L^2(\nu)}
&\stackrel{(a)}{=}
\left\|
\int_{\mU_L}
u
\,\d\mu(u)
\right\|_{L^2(\nu)}
\stackrel{(b)}{\le}
\int_{\mU_L}
\|u\|_{L^2(\nu)}
\,\d|\mu|(u)
\stackrel{(c)}{\le}
R_L
\|\mu\|_{\mathrm{TV}}.
\end{align}
Here,
(a) is the representation of
$F$,
(b) is the norm inequality for Bochner integrals with respect to finite signed measures,
and
(c) follows from the uniform bound
\eqref{eq:path-atom-uniform-bound}. This proves \eqref{eq:path-atomic-l2-bound}. Every function appearing in the defining set for $\VL$ is therefore an element of $L^2(\nu)$. Hence $\VL$ is a well-defined subset of $L^2(\nu)$. Finally, fix $F\in\VL$. By the definition of $\VL$, there exists at least one finite signed measure $\mu$ on $\mU_L$ such that
$F=\int_{\mU_L}u\d\mu(u)$ and $\|\mu\|_{\mathrm{TV}}<\infty$. Consequently, the set
\begin{align}
\left\{
\|\mu\|_{\mathrm{TV}}:
F
=
\int_{\mU_L}
u
\d\mu(u),
\ \mu \text{ is a finite signed measure on }\mU_L
\right\}
\end{align}
is nonempty, is contained in $[0,\infty)$, and contains at least one finite value. Its infimum therefore exists as a finite nonnegative real number. Hence $\CLv(F)$ is well-defined.
\end{proof}

\begin{lemma}\textnormal{(Non-degeneracy of the variation complexity)}\leavevmode\\
\label{cor:path-gauge-nondegenerate}
Assume the hypotheses of
\Cref{lem:path-atomic-space-well-defined},
and let
$R_L$
be the uniform atomic bound appearing in
\eqref{eq:path-atom-uniform-bound}.
Then, for every
$F\in\VL$,
\begin{align}
\left\|
F
\right\|_{L^2\left(\nu\right)}
\le
R_L
\CLv\left(
F
\right).
\label{eq:path-gauge-controls-l2}
\end{align}
Consequently,
\begin{align}
\CLv\left(
F
\right)
=
0
\quad
\Longrightarrow
\quad
F
=
0
\qquad
\text{in }
L^2\left(
\nu
\right).
\label{eq:path-gauge-nondegenerate}
\end{align}
Hence,
$\CLv$
is non-degenerate on
$\VL$.

\end{lemma}

\begin{proof}
Fix $F\in\VL$ and let $\varepsilon>0$ be arbitrary. By the definition of $\CLv(F)$, there exists a finite signed measure $\mu$ on $\mU_L$ satisfying
$F = \int_{\mU_L} u\,\d\mu(u)$
and
$ \|\mu\|_{\mathrm{TV}} \le \CLv(F)+\varepsilon. $
Applying \Cref{lem:path-atomic-space-well-defined} to this representation yields
$ \|F\|_{L^2(\nu)} \le R_L\|\mu\|_{\mathrm{TV}} \le R_L\left(\CLv(F)+\varepsilon\right). $
Since the above estimate holds for every $\varepsilon>0$, letting $\varepsilon\downarrow0$ gives
$\|F\|_{L^2(\nu)} \le R_L\CLv(F)$,
which proves \eqref{eq:path-gauge-controls-l2}. Now suppose that $\CLv(F)=0$. Then \eqref{eq:path-gauge-controls-l2} immediately implies
$\|F\|_{L^2(\nu)}=0$.
Since $L^2(\nu)$ is a normed vector space, its norm is positive definite. Therefore $F=0$ in $L^2(\nu)$, proving \eqref{eq:path-gauge-nondegenerate}. The final assertion follows directly from this implication.

\end{proof}

\begin{lemma}\textnormal{(Finite-profile characterization and finite-architecture bounds)}\ 
\label{lem:finite-bkl-parameter-count}
Fix
$L\ge2$,
$m\ge1$,
and
$G\ge2$,
and consider the finite lower-support architecture introduced in
\Cref{subsec:finite-parametric-bkl}.
Let
$-A_{\mX} = t_0 < t_1 < \cdots < t_{j_0} = 0 < \cdots < t_G = A_{\mX}$
be the grid from
\eqref{eq:finite-profile-grid}.
Then the following statements hold.

\begin{enumerate}[(i)]

\item
\label{lem:finite-profile-exact-characterization}
A function
$q:\mathbb R\rightarrow\mathbb R$
belongs to
$\mathcal Q_{G,A_{\mX}}$
if and only if

\begin{enumerate}[(a)]

\item
$q$
is continuous on
$\mathbb R$;

\item
$q$
is affine on every interval
$\left[t_j,t_{j+1}\right]$,
where
$j\in\left\{0,\ldots,G-1\right\}$;

\item
$q$
is constant on
$\left(-\infty,-A_{\mX}\right]$
and on
$\left[A_{\mX},\infty\right)$;

\item
$q\left(0\right)=0$;

\item
the nodal coefficients satisfy
\begin{align}
\sum_{j=0}^{G-1}
\frac{
\left(
q\left(
t_{j+1}
\right)
-
q\left(
t_j
\right)
\right)^2
}{
t_{j+1}-t_j
}
\le
1.
\label{eq:finite-profile-energy-characterization}
\end{align}

\end{enumerate}

Moreover, every such function belongs to
$\mH_{\kB}$,
and its Brownian RKHS norm is given exactly by
\begin{align}
\left\|
q
\right\|_{\mH_{\kB}}^2
=
\sum_{j=0}^{G-1}
\frac{
\left(
q\left(
t_{j+1}
\right)
-
q\left(
t_j
\right)
\right)^2
}{
t_{j+1}-t_j
}.
\label{eq:finite-profile-exact-energy}
\end{align}

\item
\label{lem:finite-architecture-range-control}
Every
$a\in\mathcal A_{L-1}^{m,G}$
satisfies
\begin{align}
\left|
a\left(
\b x
\right)
\right|
\le
A_{\mX}^{\,2^{-(L-2)}},
\qquad
\b x\in\mX.
\label{eq:finite-architecture-sharp-range}
\end{align}
Consequently,
\begin{align}
\sup_{a\in\mathcal A_{L-1}^{m,G}}
\left\|
a
\right\|_{L^\infty\left(\mX\right)}
\le
A_{\mX}^{\,2^{-(L-2)}}.
\label{eq:finite-architecture-sharp-uniform-range}
\end{align}

Furthermore, every
$f\in\mathcal D_L^{m,G}$
satisfies
\begin{align}
\left|
f\left(
\b x
\right)
\right|
\le
A_{\mX}^{\,2^{-(L-1)}},
\qquad
\b x\in\mX.
\label{eq:finite-architecture-sharp-outer-range}
\end{align}

\item
\label{lem:finite-architecture-parameter-count}
For
$L\ge3$,
every element of
$\mathcal A_{L-1}^{m,G}$
is represented by at most
\begin{align}
P_{L-1,m,G}
&:=
md
+
\left(
L-2
\right)
m
\left(
G+1
\right)
+
\left(
L-3
\right)
m^2
+
m
\label{eq:finite-bkl-explicit-parameter-count}
\end{align}
real parameters.
For
$L=2$,
every element of
$\mathcal A_1^{m,G}=\mU_1$
is represented by at most
\begin{align}
P_{1,m,G}
:=
d
\label{eq:finite-bkl-explicit-parameter-count-base}
\end{align}
real parameters.

\end{enumerate}

\end{lemma}

\begin{proof}
We prove the three assertions separately.

\paragraph{Proof of
\ref{lem:finite-profile-exact-characterization}.}
We first prove the forward implication.
Let
$q\in\mathcal Q_{G,A_{\mX}}$.
By
\eqref{eq:finite-brownian-profile-class},
there exists
$g\in\mH_{\kB}$
such that
\begin{align}
q
=
\Pi_Gg,
\qquad
\left\|
g
\right\|_{\mH_{\kB}}
\le
1.
\end{align}
By the construction of
$\Pi_G$
in
\eqref{eq:finite-profile-interpolant}--\eqref{eq:finite-profile-right-extension},
the function
$q$
is continuous on
$\mathbb R$,
is affine on every grid interval, and is constant outside
$\left[-A_{\mX},A_{\mX}\right]$.
Since the grid contains the origin and the Brownian RKHS is anchored at zero,
\begin{align}
q\left(
0
\right)
&\stackrel{(a)}{=}
\left(
\Pi_Gg
\right)\left(
0
\right)
\stackrel{(b)}{=}
g\left(
0
\right)
\stackrel{(c)}{=}
0.
\end{align}
Here,
(a) uses the representation
$q=\Pi_Gg$,
(b) follows from the nodal interpolation property at
$t_{j_0}=0$
in
\eqref{eq:finite-profile-interpolant-piecewise},
and
(c) uses the defining anchor condition of
$\mH_{\kB}$
given in
\Cref{sec:notations}. Fix
$j\in\left\{0,\ldots,G-1\right\}$.
For almost every
$t\in\left(t_j,t_{j+1}\right)$,
differentiating the affine interpolation formula
\eqref{eq:finite-profile-interpolant-piecewise}
gives
\begin{align}
q'\left(
t
\right)
&\stackrel{(a)}{=}
\frac{
g\left(
t_{j+1}
\right)
-
g\left(
t_j
\right)
}{
t_{j+1}-t_j
}
\stackrel{(b)}{=}
\frac{
q\left(
t_{j+1}
\right)
-
q\left(
t_j
\right)
}{
t_{j+1}-t_j
}.
\end{align}
Here,
(a) differentiates
\eqref{eq:finite-profile-interpolant-piecewise},
while
(b) uses the nodal interpolation identities
$q\left(t_j\right)=g\left(t_j\right)$
and
$q\left(t_{j+1}\right)=g\left(t_{j+1}\right)$.
Moreover,
$q'\left(t\right)=0$
for almost every
$t\notin\left[-A_{\mX},A_{\mX}\right]$
because
$q$
is constant outside the interpolation interval.
Thus,
$q$
is absolutely continuous and
\begin{align}
\int_{\mathbb R}
\left|
q'\left(
t
\right)
\right|^2
\,\d t
&\stackrel{(a)}{=}
\sum_{j=0}^{G-1}
\int_{t_j}^{t_{j+1}}
\left|
q'\left(
t
\right)
\right|^2
\,\d t
\stackrel{(b)}{=}
\sum_{j=0}^{G-1}
\int_{t_j}^{t_{j+1}}
\left|
\frac{
q\left(
t_{j+1}
\right)
-
q\left(
t_j
\right)
}{
t_{j+1}-t_j
}
\right|^2
\,\d t
\\
&\stackrel{(c)}{=}
\sum_{j=0}^{G-1}
\frac{
\left(
q\left(
t_{j+1}
\right)
-
q\left(
t_j
\right)
\right)^2
}{
t_{j+1}-t_j
}.
\end{align}
Here,
(a) uses the vanishing of
$q'$
outside
$\left[-A_{\mX},A_{\mX}\right]$
and the fact that the grid intervals partition this interval,
(b) substitutes the derivative formula above,
and
(c) evaluates each integral over an interval of length
$t_{j+1}-t_j$. We next control this energy by the Brownian RKHS norm of
$g$.
Since
$g$
is absolutely continuous, the fundamental theorem of calculus and the
Cauchy--Schwarz inequality give
\begin{align}
\left|
g\left(
t_{j+1}
\right)
-
g\left(
t_j
\right)
\right|^2
&\stackrel{(a)}{=}
\left|
\int_{t_j}^{t_{j+1}}
g'\left(
t
\right)
\,\d t
\right|^2
\stackrel{(b)}{\le}
\left(
t_{j+1}-t_j
\right)
\int_{t_j}^{t_{j+1}}
\left|
g'\left(
t
\right)
\right|^2
\,\d t.
\end{align}
Here,
(a) applies the fundamental theorem of calculus for absolutely continuous
functions,
while
(b) applies the Cauchy--Schwarz inequality on
$\left[t_j,t_{j+1}\right]$.
Dividing by
$t_{j+1}-t_j$
and summing over the grid intervals yields
\begin{align}
&\sum_{j=0}^{G-1}
\frac{
\left(
g\left(
t_{j+1}
\right)
-
g\left(
t_j
\right)
\right)^2
}{
t_{j+1}-t_j
}
\stackrel{(a)}{\le}
\sum_{j=0}^{G-1}
\int_{t_j}^{t_{j+1}}
\left|
g'\left(
t
\right)
\right|^2
\,\d t
\stackrel{(b)}{=}
\int_{-A_{\mX}}^{A_{\mX}}
\left|
g'\left(
t
\right)
\right|^2
\,\d t
\\
&\stackrel{(c)}{\le}
\int_{\mathbb R}
\left|
g'\left(
t
\right)
\right|^2
\,\d t
\stackrel{(d)}{=}
\left\|
g
\right\|_{\mH_{\kB}}^2
\stackrel{(e)}{\le}
1.
\end{align}
Here,
(a) sums the preceding intervalwise estimates,
(b) uses the fact that the grid intervals partition
$\left[-A_{\mX},A_{\mX}\right]$,
(c) enlarges the integration domain,
(d) applies the Brownian RKHS norm characterization in
\Cref{sec:notations},
and
(e) uses
$\left\|g\right\|_{\mH_{\kB}}\le1$.
Since
$q\left(t_j\right)=g\left(t_j\right)$
at every grid node, the preceding estimate proves
\eqref{eq:finite-profile-energy-characterization}.
Together with the anchor condition and the derivative-energy identity
established above, the Brownian RKHS characterization in
\Cref{sec:notations}
shows that
$q\in\mH_{\kB}$
and yields the exact norm formula
\eqref{eq:finite-profile-exact-energy}. We now prove the converse implication.
Let
$q:\mathbb R\rightarrow\mathbb R$
be continuous, affine on every grid interval, constant outside
$\left[-A_{\mX},A_{\mX}\right]$, anchored at zero, and satisfying
\eqref{eq:finite-profile-energy-characterization}.
Then
$q$
is absolutely continuous, and, for every
$j\in\left\{0,\ldots,G-1\right\}$,
\begin{align}
q'\left(
t
\right)
=
\frac{
q\left(
t_{j+1}
\right)
-
q\left(
t_j
\right)
}{
t_{j+1}-t_j
}
\end{align}
for almost every
$t\in\left(t_j,t_{j+1}\right)$,
whereas
$q'\left(t\right)=0$
for almost every
$t\notin\left[-A_{\mX},A_{\mX}\right]$.
Consequently, the Brownian RKHS characterization in
\Cref{sec:notations}
gives
$q\in\mH_{\kB}$
and
\begin{align}
\left\|
q
\right\|_{\mH_{\kB}}^2
&\stackrel{(a)}{=}
\int_{\mathbb R}
\left|
q'\left(
t
\right)
\right|^2
\,\d t
\stackrel{(b)}{=}
\sum_{j=0}^{G-1}
\frac{
\left(
q\left(
t_{j+1}
\right)
-
q\left(
t_j
\right)
\right)^2
}{
t_{j+1}-t_j
}
\stackrel{(c)}{\le}
1.
\end{align}
Here,
(a) applies the Brownian RKHS norm formula in
\Cref{sec:notations},
(b) substitutes the derivative representation above and evaluates the
integrals over the grid intervals,
and
(c) uses
\eqref{eq:finite-profile-energy-characterization}.
This calculation also proves
\eqref{eq:finite-profile-exact-energy}
for every function satisfying the stated conditions. It remains to verify that
$q$
is reproduced by the interpolation operator.
For every
$j\in\left\{0,\ldots,G-1\right\}$
and every
$t\in\left[t_j,t_{j+1}\right]$,
the affinity of
$q$
gives
\begin{align}
q\left(
t
\right)
&\stackrel{(a)}{=}
\frac{
t_{j+1}-t
}{
t_{j+1}-t_j
}
q\left(
t_j
\right)
+
\frac{
t-t_j
}{
t_{j+1}-t_j
}
q\left(
t_{j+1}
\right)
\stackrel{(b)}{=}
\left(
\Pi_Gq
\right)\left(
t
\right).
\end{align}
Here,
(a) is the affine interpolation identity,
while
(b) applies
\eqref{eq:finite-profile-interpolant-piecewise}
with
$g=q$.
Outside
$\left[-A_{\mX},A_{\mX}\right]$,
both
$q$
and
$\Pi_Gq$
are constant with the same endpoint values by
\eqref{eq:finite-profile-left-extension}
and
\eqref{eq:finite-profile-right-extension}.
Hence, $\Pi_Gq
=
q$ on $\mathbb R$. Since
$q\in\mH_{\kB}$
and
$\left\|q\right\|_{\mH_{\kB}}\le1$,
the definition
\eqref{eq:finite-brownian-profile-class}
implies
$q\in\mathcal Q_{G,A_{\mX}}$.
This proves
\Cref{lem:finite-profile-exact-characterization}.

\paragraph{Proof of
\ref{lem:finite-architecture-range-control}.}

We first derive the pointwise Brownian estimate for the finite profile
class.
Let
$q\in\mathcal Q_{G,A_{\mX}}$
and
$t\in\mathbb R$.
Define
$I_t := \left[ \min\left\{ 0,t \right\}, \max\left\{ 0,t \right\} \right]$.
Since
$q\left(0\right)=0$,
$q$
is absolutely continuous, and
$\left\|q\right\|_{\mH_{\kB}}\le1$,
we have
\begin{align}
\left|
q\left(
t
\right)
\right|
&\stackrel{(a)}{=}
\left|
q\left(
t
\right)
-
q\left(
0
\right)
\right|
\stackrel{(b)}{=}
\left|
\int_0^t
q'\left(
s
\right)
\,\d s
\right|
\stackrel{(c)}{\le}
\int_{I_t}
\left|
q'\left(
s
\right)
\right|
\,\d s
\stackrel{(d)}{\le}
\left|
I_t
\right|^{1/2}
\left(
\int_{I_t}
\left|
q'\left(
s
\right)
\right|^2
\,\d s
\right)^{1/2}
\notag\\
&\stackrel{(e)}{\le}
\left|
t
\right|^{1/2}
\left\|
q
\right\|_{\mH_{\kB}}
\stackrel{(f)}{\le}
\left|
t
\right|^{1/2}.
\label{eq:finite-profile-proof-pointwise-bound}
\end{align}
Here,
(a) uses the anchor condition,
(b) applies the fundamental theorem of calculus,
(c) accounts for both possible orientations of the integral,
(d) applies the Cauchy--Schwarz inequality on
$I_t$,
(e) uses
$\left|I_t\right|=\left|t\right|$
and enlarges the derivative integral to
$\mathbb R$,
and
(f) uses the unit-norm property of
$\mathcal Q_{G,A_{\mX}}$. We first consider
$L=2$.
By
\eqref{eq:finite-lower-support-class-depth-one},
every
$a\in\mathcal A_1^{m,G}$
has the form
$a\left( \b x \right) = \bomega^\top\b x$
for some
$\bomega\in\Omega\subseteq\mathbb S^{d-1}$.
Therefore,
\begin{align}
\left|
a\left(
\b x
\right)
\right|
&\stackrel{(a)}{=}
\left|
\bomega^\top\b x
\right|
\stackrel{(b)}{\le}
\left\|
\bomega
\right\|_2
\left\|
\b x
\right\|_2
\stackrel{(c)}{\le}
R_{\mX}
\stackrel{(d)}{\le}
A_{\mX}.
\end{align}
Here,
(a) uses the definition of the first-layer support,
(b) applies the Euclidean Cauchy--Schwarz inequality,
(c) uses
$\left\|\bomega\right\|_2=1$
and the definition of
$R_{\mX}$,
and
(d) uses
$A_{\mX}=\max\left\{1,R_{\mX}\right\}$.
Since
$2^{-(L-2)}=1$
when
$L=2$,
this proves
\eqref{eq:finite-architecture-sharp-range}
in the base case. Suppose now that
$L\ge3$,
and set
$K := L-2$.
Fix an admissible parameter tuple
$\boldsymbol\theta$
and
$\b x\in\mX$.
At the linear layer,
\begin{align}
\max_{j\in\left[m\right]}
\left|
z_j^{(0)}
\left(
\b x
\right)
\right|
&\stackrel{(a)}{=}
\max_{j\in\left[m\right]}
\left|
\bomega_j^\top\b x
\right|
 \stackrel{(b)}{\le}
\max_{j\in\left[m\right]}
\left\|
\bomega_j
\right\|_2
\left\|
\b x
\right\|_2
 \stackrel{(c)}{\le}
R_{\mX}
\stackrel{(d)}{\le}
A_{\mX}.
\label{eq:finite-architecture-proof-linear-range}
\end{align}
Here,
(a) applies
\eqref{eq:finite-lower-support-linear-layer},
(b) applies the Euclidean Cauchy--Schwarz inequality,
(c) uses
$\bomega_j\in\Omega\subseteq\mathbb S^{d-1}$,
and
(d) applies
\eqref{eq:finite-architecture-input-radius}. At the first profile layer,
\begin{align}
\left|
z_j^{(1)}
\left(
\b x
\right)
\right|
&\stackrel{(a)}{=}
\left|
q_j^{(1)}
\left(
z_j^{(0)}
\left(
\b x
\right)
\right)
\right|
 \stackrel{(b)}{\le}
\left|
z_j^{(0)}
\left(
\b x
\right)
\right|^{1/2}
\stackrel{(c)}{\le}
A_{\mX}^{1/2},
\qquad
j\in\left[m\right].
\end{align}
Here,
(a) uses
\eqref{eq:finite-lower-support-first-profile-layer},
(b) applies
\eqref{eq:finite-profile-proof-pointwise-bound},
and
(c) uses
\eqref{eq:finite-architecture-proof-linear-range}. We now proceed recursively.
Fix
$r\in\left\{2,\ldots,K\right\}$
and assume that
\begin{align}
\max_{k\in\left[m\right]}
\left|
z_k^{(r-1)}
\left(
\b x
\right)
\right|
\le
A_{\mX}^{\,2^{-(r-1)}}.
\label{eq:finite-architecture-proof-induction-hypothesis}
\end{align}
For every
$j\in\left[m\right]$,
\begin{align}
\left|
s_j^{(r)}
\left(
\b x
\right)
\right|
&\stackrel{(a)}{=}
\left|
\sum_{k=1}^{m}
W_{jk}^{(r)}
z_k^{(r-1)}
\left(
\b x
\right)
\right|
\stackrel{(b)}{\le}
\sum_{k=1}^{m}
\left|
W_{jk}^{(r)}
\right|
\left|
z_k^{(r-1)}
\left(
\b x
\right)
\right|
\stackrel{(c)}{\le}
A_{\mX}^{\,2^{-(r-1)}}
\sum_{k=1}^{m}
\left|
W_{jk}^{(r)}
\right|
\notag\\
&\stackrel{(d)}{\le}
A_{\mX}^{\,2^{-(r-1)}}.
\label{eq:finite-architecture-proof-preactivation-range}
\end{align}
Here,
(a) applies
\eqref{eq:finite-lower-support-preactivation},
(b) applies the triangle inequality,
(c) invokes
\eqref{eq:finite-architecture-proof-induction-hypothesis},
and
(d) uses the row-sum constraint
\eqref{eq:finite-mixing-matrix-class}. Applying the profile bound gives
\begin{align}
\left|
z_j^{(r)}
\left(
\b x
\right)
\right|
&\stackrel{(a)}{=}
\left|
q_j^{(r)}
\left(
s_j^{(r)}
\left(
\b x
\right)
\right)
\right|
 \stackrel{(b)}{\le}
\left|
s_j^{(r)}
\left(
\b x
\right)
\right|^{1/2}
\stackrel{(c)}{\le}
A_{\mX}^{\,2^{-r}}.
\end{align}
Here,
(a) applies
\eqref{eq:finite-lower-support-hidden-layer},
(b) applies
\eqref{eq:finite-profile-proof-pointwise-bound},
and
(c) uses
\eqref{eq:finite-architecture-proof-preactivation-range}.
By induction,
\begin{align}
\max_{j\in\left[m\right]}
\left|
z_j^{(r)}
\left(
\b x
\right)
\right|
\le
A_{\mX}^{\,2^{-r}},
\qquad
r\in\left[K\right].
\label{eq:finite-architecture-proof-all-hidden-ranges}
\end{align}
Since
$A_{\mX}\ge1$,
\eqref{eq:finite-architecture-proof-preactivation-range}
also gives
\begin{align}
\left|
s_j^{(r)}
\left(
\b x
\right)
\right|
\le
A_{\mX},
\qquad
r\in\left\{2,\ldots,K\right\},
\quad
j\in\left[m\right].
\end{align}
Thus, every profile is evaluated within the interpolation interval used
in the finite architecture. Finally, using the readout constraint,
\begin{align}
\left|
a_{\boldsymbol\theta}
\left(
\b x
\right)
\right|
&\stackrel{(a)}{=}
\left|
\sum_{j=1}^{m}
\beta_j
z_j^{(K)}
\left(
\b x
\right)
\right|
 \stackrel{(b)}{\le}
\sum_{j=1}^{m}
\left|
\beta_j
\right|
\left|
z_j^{(K)}
\left(
\b x
\right)
\right|
 \stackrel{(c)}{\le}
A_{\mX}^{\,2^{-K}}
\sum_{j=1}^{m}
\left|
\beta_j
\right|
\stackrel{(d)}{\le}
A_{\mX}^{\,2^{-K}}
\stackrel{(e)}{=}
A_{\mX}^{\,2^{-(L-2)}}.
\end{align}
Here,
(a) applies
\eqref{eq:finite-lower-support-output},
(b) applies the triangle inequality,
(c) uses
\eqref{eq:finite-architecture-proof-all-hidden-ranges},
(d) uses
\eqref{eq:finite-readout-class},
and
(e) substitutes
$K=L-2$.
This proves
\eqref{eq:finite-architecture-sharp-range}
and
\eqref{eq:finite-architecture-sharp-uniform-range}. It remains to establish the range of the associated outer dictionary.
Fix
$f\in\mathcal D_L^{m,G}$.
By
\eqref{eq:finite-architecture-outer-dictionary},
there exists
$a\in\mathcal A_{L-1}^{m,G}$
such that
$f \in \mH_{k_a}, \qquad \left\| f \right\|_{\mH_{k_a}} \le 1$.
For every
$\b x\in\mX$,
the reproducing property gives
\begin{align}
\left|
f\left(
\b x
\right)
\right|
&\stackrel{(a)}{=}
\left|
\left\langle
f,
k_a\left(
\b x,\cdot
\right)
\right\rangle_{\mH_{k_a}}
\right|
\stackrel{(b)}{\le}
\left\|
f
\right\|_{\mH_{k_a}}
\left\|
k_a\left(
\b x,\cdot
\right)
\right\|_{\mH_{k_a}}
\stackrel{(c)}{=}
\left\|
f
\right\|_{\mH_{k_a}}
k_a\left(
\b x,\b x
\right)^{1/2}
\notag\\
&\stackrel{(d)}{\le}
\left|
a\left(
\b x
\right)
\right|^{1/2}
\stackrel{(e)}{\le}
A_{\mX}^{\,2^{-(L-1)}}.
\end{align}
Here,
(a) applies the reproducing property in
$\mH_{k_a}$,
(b) applies the Cauchy--Schwarz inequality,
(c) uses the standard RKHS section-norm identity,
(d) uses
$\left\|f\right\|_{\mH_{k_a}}\le1$
and
\begin{align}
k_a\left(
\b x,\b x
\right)
=
\kB\left(
a\left(
\b x
\right),
a\left(
\b x
\right)
\right)
=
\left|
a\left(
\b x
\right)
\right|,
\end{align}
and
(e) applies
\eqref{eq:finite-architecture-sharp-range}.
This proves
\eqref{eq:finite-architecture-sharp-outer-range}.

\paragraph{Proof of
\ref{lem:finite-architecture-parameter-count}.}
We first consider
$L=2$.
In this case,
\begin{align}
\mathcal A_1^{m,G}
=
\mU_1
=
\left\{
\b x
\mapsto
\bomega^\top\b x
:
\bomega\in\Omega
\right\}.
\end{align}
Every such function is specified by the
$d$
coordinates of
$\bomega$.
Therefore,
$P_{1,m,G} = d$,
which proves
\eqref{eq:finite-bkl-explicit-parameter-count-base}. Suppose now that
$L\ge3$,
and set
$K := L-2$.
We count the four parameter groups appearing in
\eqref{eq:finite-lower-support-parameter-tuple}. First, the architecture contains
$m$
first-layer directions
$\bomega_1,\ldots,\bomega_m\in\mathbb R^d$.
Hence, the directions contribute at most
\begin{align}
P_{\mathrm{dir}}
&\stackrel{(a)}{\le}
md.
\label{eq:finite-architecture-proof-direction-count}
\end{align}
Here,
(a) counts
$d$
ambient coordinates for each of the
$m$
directions.
The restriction
$\bomega_j\in\Omega$
does not increase this upper bound. Second, each of the
$K$
profile layers contains
$m$
finite Brownian profiles.
Every profile is specified by at most
$G+1$
nodal coefficients.
Therefore, the profile parameters contribute at most
\begin{align}
P_{\mathrm{prof}}
&\stackrel{(a)}{\le}
Km\left(
G+1
\right).
\label{eq:finite-architecture-proof-profile-count}
\end{align}
Here,
(a) multiplies the number
$Km$
of profiles by the upper bound
$G+1$
on the number of nodal coordinates per profile.
The anchor condition and the energy constraint restrict the admissible
parameter set but do not increase its ambient dimension. Third, mixing matrices occur at the levels
$r \in \left\{ 2,\ldots,K \right\}$.
When
$K=1$,
this index set is empty.
When
$K\ge2$,
it contains
$K-1$
indices.
Each matrix belongs to
$\mathbb R^{m\times m}$
and therefore contains at most
$m^2$
real coordinates.
Consequently,
\begin{align}
P_{\mathrm{mix}}
&\stackrel{(a)}{\le}
\left(
K-1
\right)
m^2.
\label{eq:finite-architecture-proof-mixing-count}
\end{align}
Here,
(a) multiplies the number of internal mixing levels by the number of
entries in each matrix.
The row-sum constraints restrict the admissible matrices but introduce
no additional parameters. Fourth, the final readout vector
$\boldsymbol\beta\in\mathbb R^m$
contributes at most
\begin{align}
P_{\mathrm{out}}
&\stackrel{(a)}{\le}
m
\label{eq:finite-architecture-proof-readout-count}
\end{align}
real parameters.
Here,
(a) counts the entries of the readout vector.
Again, the
$\ell^1$
constraint restricts the admissible set without increasing its ambient
dimension. Combining
\eqref{eq:finite-architecture-proof-direction-count},
\eqref{eq:finite-architecture-proof-profile-count},
\eqref{eq:finite-architecture-proof-mixing-count},
and
\eqref{eq:finite-architecture-proof-readout-count}
gives
\begin{align}
P_{L-1,m,G}
&\stackrel{(a)}{\le}
P_{\mathrm{dir}}
+
P_{\mathrm{prof}}
+
P_{\mathrm{mix}}
+
P_{\mathrm{out}}
\notag\\
&\stackrel{(b)}{\le}
md
+
Km\left(
G+1
\right)
+
\left(
K-1
\right)
m^2
+
m
\notag\\
&\stackrel{(c)}{=}
md
+
\left(
L-2
\right)
m\left(
G+1
\right)
+
\left(
L-3
\right)
m^2
+
m.
\end{align}
Here,
(a) separates the four parameter groups,
(b) substitutes their individual upper bounds,
and
(c) uses
$K=L-2$.
This proves
\eqref{eq:finite-bkl-explicit-parameter-count}
and completes the proof.
\end{proof}

\begin{lemma}\textnormal{(Recursive RKHS structure of the Variation BKL dictionary)}\ 
\label{lem:vbkl-union-rkhs}
Fix
$l\ge2$
and
$u\in\mU_{l-1}$.
Let
$k_u$
be the Brownian pullback kernel defined in
\eqref{eq:pullbackk},
and let
$\mH_{k_u}$
denote its RKHS\@.
Define
$
T_u
:
\mH_{\kB}
\longrightarrow
\mathbb R^{\mX}$, $\left(
T_ug
\right)
\left(
\b x
\right)
:=
g\left(
u\left(
\b x
\right)
\right),
\quad
\b x\in\mX,
$
and
$
\mathcal N_u
:=
\ker\left(
T_u
\right)
=
\left\{
g\in\mH_{\kB}
:
g\left(
u\left(
\b x
\right)
\right)
=
0
\text{ for every }
\b x\in\mX
\right\}.
$
Then
$\mathcal N_u$
is a closed linear subspace of
$\mH_{\kB}$.
Let
$\mathcal N_u^\perp$
denote its orthogonal complement in
$\mH_{\kB}$.
Then
\begin{align}
\mH_{k_u}
=
T_u\left(
\mH_{\kB}
\right)
=
T_u\left(
\mathcal N_u^\perp
\right).
\label{eq:vbkl-pullback-rkhs-range}
\end{align}
as sets of functions on
$\mX$. More precisely, for every
$f\in\mH_{k_u}$,
there exists a unique
$g_f\in\mathcal N_u^\perp$
such that
\begin{align}
T_ug_f
=
f.
\label{eq:vbkl-pullback-minimal-representative}
\end{align}
This representative satisfies
\begin{align}
\left\|
f
\right\|_{\mH_{k_u}}
&=
\left\|
g_f
\right\|_{\mH_{\kB}}
\notag\\
&=
\min
\left\{
\left\|
g
\right\|_{\mH_{\kB}}
:
g\in\mH_{\kB},
\;
T_ug=f
\right\}.
\label{eq:vbkl-pullback-quotient-norm}
\end{align}
Consequently,
\begin{align}
\left\{
T_ug
:
g\in\mH_{\kB},
\;
\left\|
g
\right\|_{\mH_{\kB}}
\le
1
\right\}
=
\left\{
f\in\mH_{k_u}
:
\left\|
f
\right\|_{\mH_{k_u}}
\le
1
\right\}.
\label{eq:vbkl-pullback-unit-ball}
\end{align}
Therefore, the depth-$l$ recursive dictionary satisfies
\begin{align}
\mU_l
=
\bigcup_{u\in\mU_{l-1}}
\left\{
f\in\mH_{k_u}
:
\left\|
f
\right\|_{\mH_{k_u}}
\le
1
\right\},
\label{eq:union-rkhs}
\end{align}
or equivalently,
\begin{align}
\mU_l
=
\left\{
\b x
\mapsto
g\left(
u\left(
\b x
\right)
\right)
:
u\in\mU_{l-1},
\;
g\in\mH_{\kB},
\;
\left\|
g
\right\|_{\mH_{\kB}}
\le
1
\right\}.
\label{eq:union-rkhs-pullback}
\end{align}
Let
$\mathcal V^{(l)}$
denote the depth-$l$ variation space obtained by replacing
$L$
with
$l$
in
\eqref{eq:variation-bkl-space}.
Then
\begin{align}
\mathcal V^{(l)}
=
\left\{
F\in L^2\left(
\nu
\right)
:
F
=
\int_{\mU_l}
a
\,\d\rho\left(
a
\right),
\;
\rho\in\mathcal M\left(
\mU_l
\right),
\;
\left\|
\rho
\right\|_{\mathrm{TV}}
<
\infty
\right\}.
\label{eq:variation-hull}
\end{align}
Thus,
$\mathcal V^{(l)}$
is the variation hull generated by the recursively constructed union of Brownian pullback RKHS unit balls.

\end{lemma}

\begin{proof}
Fix
$l\ge2$
and
$u\in\mU_{l-1}$,
and set
\begin{align}
Z_u
:=
u\left(
\mX
\right).
\end{align}
\begingroup
\emergencystretch=2em
Apply the kernel pullback and restriction result
(\cite{MohammadigohariDiFattaNicosiaPardalos2026BKL}, Lemma~C4(i)--(iii))
with
\par\endgroup
\begin{align}
Z_0
&=
\mathbb R,
&
Z
&=
Z_u,
&
k
&=
\kB.
\end{align}
It follows that, for every
$g\in\mH_{\kB}$,
\begin{align}
T_ug
&\in
\mH_{k_u},
&
\left\|
T_ug
\right\|_{\mH_{k_u}}
&\le
\left\|
g
\right\|_{\mH_{\kB}}.
\end{align}
Moreover, for every
$f\in\mH_{k_u}$,
there exists a minimum-norm representative
$g_f\in\mH_{\kB}$
such that
\begin{align}
T_ug_f
&=
f,
&
\left\|
g_f
\right\|_{\mH_{\kB}}
&=
\left\|
f
\right\|_{\mH_{k_u}}.
\end{align}
The operator
$T_u:\mH_{\kB}\rightarrow\mH_{k_u}$
is linear and continuous.
Therefore,
\begin{align}
\mathcal N_u
=
\ker\left(
T_u
\right)
\end{align}
is a closed linear subspace of
$\mH_{\kB}$.
Let
$P_u$
denote the orthogonal projection of
$\mH_{\kB}$
onto
$\mathcal N_u$.
Since
$P_ug_f\in\mathcal N_u$,
one has
$T_u\left(P_ug_f\right)=0$.
Therefore,
\begin{align}
T_u\left(
g_f-P_ug_f
\right)
&\stackrel{(a)}{=}
T_ug_f
-
T_u\left(
P_ug_f
\right)
\stackrel{(b)}{=}
f.
\end{align}
Here,
(a) applies the linearity of
$T_u$,
while
(b) uses
$T_ug_f=f$
and
$T_u\left(P_ug_f\right)=0$. Suppose that
$P_ug_f\neq0$.
The vectors
$g_f-P_ug_f$
and
$P_ug_f$
are orthogonal.
Thus,
\begin{align}
\left\|
g_f-P_ug_f
\right\|_{\mH_{\kB}}^2
&\stackrel{(a)}{=}
\left\|
g_f
\right\|_{\mH_{\kB}}^2
-
\left\|
P_ug_f
\right\|_{\mH_{\kB}}^2
\stackrel{(b)}{<}
\left\|
g_f
\right\|_{\mH_{\kB}}^2.
\end{align}
Here,
(a) applies the Pythagorean identity to the orthogonal decomposition
\begin{align}
g_f
=
\left(
g_f-P_ug_f
\right)
+
P_ug_f,
\end{align}
while
(b) uses
$P_ug_f\neq0$.
The preceding two displays produce another representative of
$f$
with strictly smaller
$\mH_{\kB}$-norm, contradicting the minimum-norm property of
$g_f$.
Hence,
\begin{align}
P_ug_f
&=
0,
&
g_f
&\in
\mathcal N_u^\perp.
\end{align}
We next prove uniqueness in
$\mathcal N_u^\perp$.
Let
$h_f\in\mathcal N_u^\perp$
satisfy
$T_uh_f=f$.
Hence,
\begin{align}
T_u\left(
g_f-h_f
\right)
&\stackrel{(a)}{=}
T_ug_f
-
T_uh_f
\stackrel{(b)}{=}
0,
\notag\\
g_f-h_f
&\stackrel{(c)}{\in}
\mathcal N_u
\cap
\mathcal N_u^\perp
\stackrel{(d)}{=}
\left\{
0
\right\}.
\end{align}
Here,
(a) applies the linearity of
$T_u$,
(b) uses
$T_ug_f=T_uh_f=f$,
(c) follows from the first line together with
$g_f,h_f\in\mathcal N_u^\perp$,
and
(d) uses
\begin{align}
\mathcal N_u
\cap
\mathcal N_u^\perp
=
\left\{
0
\right\}.
\end{align}
Thus,
$g_f=h_f$,
which proves
\eqref{eq:vbkl-pullback-minimal-representative}. The preceding construction also identifies the range of the pullback
operator.
Consequently,
\begin{align}
\mH_{k_u}
&\stackrel{(a)}{\subseteq}
T_u\left(
\mathcal N_u^\perp
\right)
\stackrel{(b)}{\subseteq}
T_u\left(
\mH_{\kB}
\right)
\stackrel{(c)}{\subseteq}
\mH_{k_u}.
\end{align}
Here,
(a) follows because every
$f\in\mH_{k_u}$
has a representative
$g_f\in\mathcal N_u^\perp$,
(b) uses
$\mathcal N_u^\perp\subseteq\mH_{\kB}$,
and
(c) applies the pullback contraction established above.
Therefore, all three sets are equal, which proves
\eqref{eq:vbkl-pullback-rkhs-range}. Let
$g\in\mH_{\kB}$
be any representative of
$f\in\mH_{k_u}$,
so that
$T_ug=f$.
Therefore,
\begin{align}
\left\|
f
\right\|_{\mH_{k_u}}
&\stackrel{(a)}{=}
\left\|
T_ug
\right\|_{\mH_{k_u}}
\stackrel{(b)}{\le}
\left\|
g
\right\|_{\mH_{\kB}}.
\end{align}
Here,
(a) uses
$T_ug=f$,
while
(b) applies the pullback contraction.
The minimum-norm representative
$g_f$
is admissible and attains equality in this lower bound.
Consequently,
\begin{align}
\left\|
f
\right\|_{\mH_{k_u}}
&\stackrel{(a)}{=}
\left\|
g_f
\right\|_{\mH_{\kB}}
\stackrel{(b)}{=}
\min
\left\{
\left\|
g
\right\|_{\mH_{\kB}}
:
g\in\mH_{\kB},
\;
T_ug=f
\right\}.
\end{align}
Here,
(a) applies the minimum-extension identity established above,
while
(b) combines the preceding lower bound with the admissibility of
$g_f$.
This proves
\eqref{eq:vbkl-pullback-quotient-norm}. We now prove the exact unit-ball identity.
Let
$g\in\mH_{\kB}$
satisfy
$\left\|g\right\|_{\mH_{\kB}}\le1$.
Hence,
\begin{align}
\left\|
T_ug
\right\|_{\mH_{k_u}}
&\stackrel{(a)}{\le}
\left\|
g
\right\|_{\mH_{\kB}}
\stackrel{(b)}{\le}
1.
\end{align}
Here,
(a) applies the pullback contraction,
while
(b) uses the assumed norm bound on
$g$.
Therefore,
\begin{align}
\left\{
T_ug
:
g\in\mH_{\kB},
\;
\left\|
g
\right\|_{\mH_{\kB}}
\le
1
\right\}
\subseteq
\left\{
f\in\mH_{k_u}
:
\left\|
f
\right\|_{\mH_{k_u}}
\le
1
\right\}.
\end{align}
Conversely, let
$f\in\mH_{k_u}$
satisfy
$\left\|f\right\|_{\mH_{k_u}}\le1$.
Its minimum-norm representative satisfies $T_ug_f
=
f$. Hence,
\begin{align}
\left\|
g_f
\right\|_{\mH_{\kB}}
&\stackrel{(a)}{=}
\left\|
f
\right\|_{\mH_{k_u}}
\stackrel{(b)}{\le}
1.
\end{align}
Here,
(a) applies the minimum-extension identity,
while
(b) uses the assumed norm bound on
$f$.
Therefore,
\begin{align}
\left\{
f\in\mH_{k_u}
:
\left\|
f
\right\|_{\mH_{k_u}}
\le
1
\right\}
\subseteq
\left\{
T_ug
:
g\in\mH_{\kB},
\;
\left\|
g
\right\|_{\mH_{\kB}}
\le
1
\right\}.
\end{align}
Combining the two inclusions proves
\eqref{eq:vbkl-pullback-unit-ball}. Finally, apply the unit-ball identity to the recursive definition of the
atomic dictionary.
Therefore,
\begin{align}
\mU_l
&\stackrel{(a)}{=}
\bigcup_{u\in\mU_{l-1}}
\left\{
T_ug
:
g\in\mH_{\kB},
\;
\left\|
g
\right\|_{\mH_{\kB}}
\le
1
\right\}
\notag\\
&\stackrel{(b)}{=}
\bigcup_{u\in\mU_{l-1}}
\left\{
f\in\mH_{k_u}
:
\left\|
f
\right\|_{\mH_{k_u}}
\le
1
\right\}.
\end{align}
Here,
(a) applies the compositional definition of
$\mU_l$,
while
(b) applies
\eqref{eq:vbkl-pullback-unit-ball}
for every fixed
$u\in\mU_{l-1}$.
This proves
\eqref{eq:union-rkhs}.
By the definition of
$T_u$,
\eqref{eq:union-rkhs-pullback}
is the equivalent compositional form of the same identity. By the definition of the depth-$l$ VBKL space,
\begin{align}
\mathcal V^{(l)}
=
\left\{
F\in L^2\left(
\nu
\right)
:
F
=
\int_{\mU_l}
a
\,\d\rho\left(
a
\right),
\;
\rho\in\mathcal M\left(
\mU_l
\right),
\;
\left\|
\rho
\right\|_{\mathrm{TV}}
<
\infty
\right\}.
\end{align}
This proves
\eqref{eq:variation-hull}
and completes the proof.
\end{proof}

\begin{lemma}\textnormal{(Exact reduction from finite variation hulls to symmetric dictionaries)}\ 
\label{lem:variation-reduction}
Let
$n\in\mathbb N$,
$n\ge1$,
and let
$\b x_1,\ldots,\b x_n\in\mX$
be fixed sample points.
Let
$\mathcal D$
be a nonempty class of real-valued functions on
$\mX$,
equipped with a
$\sigma$-algebra for which every evaluation map
$
f
\longmapsto
f\left(
\b x
\right)$, $\b x\in\mX,
$
is measurable.
Assume that
$\mathcal D$
is symmetric:
\begin{align}
\mathcal D
=
-\mathcal D
:=
\left\{
-f
:
f\in\mathcal D
\right\}.
\label{eq:variation-reduction-dictionary-symmetry}
\end{align}

For
$R>0$,
define the signed-measure variation hull generated by
$\mathcal D$
as
\begin{align}
\mathcal W_R\left(
\mathcal D
\right)
:=
\left\{
F_\mu
:
F_\mu\left(
\b x
\right)
=
\int_{\mathcal D}
f\left(
\b x
\right)
\,\d\mu\left(
f
\right),
\;
\mu\in\mathcal M\left(
\mathcal D
\right),
\;
\left\|
\mu
\right\|_{\mathrm{TV}}
\le
R
\right\}.
\label{eq:variation-reduction-general-variation-hull}
\end{align}
Assume that every such integral is well defined at
$\b x_1,\ldots,\b x_n$,
and define
\begin{align}
M_{\b x}\left(
\mathcal D
\right)
:=
\sup_{f\in\mathcal D}
\max_{i\in\left[
n
\right]}
\left|
f\left(
\b x_i
\right)
\right|
<
\infty.
\label{eq:variation-reduction-sample-envelope}
\end{align}
Then
\begin{align}
\widehat{\mathfrak R}_n
\left(
\mathcal W_R\left(
\mathcal D
\right)
\right)
&=
R
\widehat{\mathfrak R}_n
\left(
\mathcal D
\right)
\notag\\
&\le
R
M_{\b x}\left(
\mathcal D
\right).
\label{eq:variation-reduction-exact-identity}
\end{align}
In particular, fix
$L\ge2$,
$m\ge1$,
and
$G\ge2$.
The finite outer Brownian dictionary
$\mathcal D_L^{m,G}$
defined in
\eqref{eq:finite-architecture-outer-dictionary}
is symmetric.
Define
\begin{align}
B_{\b x}
:=
\sup_{a\in\mathcal A_{L-1}^{m,G}}
\max_{i\in\left[
n
\right]}
\left|
a\left(
\b x_i
\right)
\right|.
\label{eq:variation-reduction-finite-support-envelope}
\end{align}
Then
\begin{align}
M_{\b x}\left(
\mathcal D_L^{m,G}
\right)
\le
B_{\b x}^{1/2}.
\label{eq:variation-reduction-outer-envelope}
\end{align}
Consequently,
\begin{align}
\widehat{\mathfrak R}_n
\left(
\mathcal W_R^{L;m,G}
\right)
&=
R
\widehat{\mathfrak R}_n
\left(
\mathcal D_L^{m,G}
\right)
\notag\\
&\le
R
B_{\b x}^{1/2}.
\label{eq:variation-reduction-finite-architecture}
\end{align}

\end{lemma}

\begin{proof}
Fix sample points
$\b x_1,\ldots,\b x_n\in\mX$
and fix a realization
$\varepsilon_1,\ldots,\varepsilon_n$
of the Rademacher variables.
Define the sample-dependent linear functional
\begin{align}
\Phi_{\varepsilon}
\left(
f
\right)
:=
\frac{1}{n}
\sum_{i=1}^{n}
\varepsilon_i
f\left(
\b x_i
\right),
\qquad
f\in\mathcal D.
\label{eq:variation-reduction-linear-functional}
\end{align}
By
\eqref{eq:variation-reduction-sample-envelope},
for every
$f\in\mathcal D$,
\begin{align}
\left|
\Phi_{\varepsilon}
\left(
f
\right)
\right|
&\stackrel{(a)}{=}
\left|
\frac{1}{n}
\sum_{i=1}^{n}
\varepsilon_i
f\left(
\b x_i
\right)
\right|
 \stackrel{(b)}{\le}
\frac{1}{n}
\sum_{i=1}^{n}
\left|
\varepsilon_i
\right|
\left|
f\left(
\b x_i
\right)
\right|
 \stackrel{(c)}{=}
\frac{1}{n}
\sum_{i=1}^{n}
\left|
f\left(
\b x_i
\right)
\right|
 \stackrel{(d)}{\le}
M_{\b x}
\left(
\mathcal D
\right).
\label{eq:variation-reduction-functional-boundedness}
\end{align}
Here,
(a) applies the definition of
$\Phi_{\varepsilon}$,
(b) applies the triangle inequality,
(c) uses
$\left|\varepsilon_i\right|=1$
for every
$i\in\left[n\right]$,
and
(d) applies
\eqref{eq:variation-reduction-sample-envelope}.
Thus,
$\Phi_{\varepsilon}$
is bounded on
$\mathcal D$. We next use the symmetry of
$\mathcal D$.
For every
$f\in\mathcal D$,
\begin{align}
\left|
\Phi_{\varepsilon}
\left(
f
\right)
\right|
&\stackrel{(a)}{=}
\max
\left\{
\Phi_{\varepsilon}
\left(
f
\right),
-
\Phi_{\varepsilon}
\left(
f
\right)
\right\}
 \stackrel{(b)}{=}
\max
\left\{
\Phi_{\varepsilon}
\left(
f
\right),
\Phi_{\varepsilon}
\left(
-f
\right)
\right\}
 \stackrel{(c)}{\le}
\sup_{h\in\mathcal D}
\Phi_{\varepsilon}
\left(
h
\right).
\label{eq:variation-reduction-symmetry-pointwise}
\end{align}
Here,
(a) uses the elementary identity
$\left|s\right|=\max\left\{s,-s\right\}$,
(b) uses the linearity of
$\Phi_{\varepsilon}$,
and
(c) uses
$f\in\mathcal D$
together with
$-f\in\mathcal D$,
which follows from
\eqref{eq:variation-reduction-dictionary-symmetry}.
Taking the supremum over
$f\in\mathcal D$
in
\eqref{eq:variation-reduction-symmetry-pointwise}
gives
\begin{align}
\sup_{f\in\mathcal D}
\left|
\Phi_{\varepsilon}
\left(
f
\right)
\right|
\le
\sup_{f\in\mathcal D}
\Phi_{\varepsilon}
\left(
f
\right).
\label{eq:variation-reduction-symmetry-first-direction}
\end{align}
Conversely,
\begin{align}
\sup_{f\in\mathcal D}
\Phi_{\varepsilon}
\left(
f
\right)
&\stackrel{(a)}{\le}
\sup_{f\in\mathcal D}
\left|
\Phi_{\varepsilon}
\left(
f
\right)
\right|.
\label{eq:variation-reduction-symmetry-second-direction}
\end{align}
Here,
(a) uses
$\Phi_{\varepsilon}\left(f\right)
\le
\left|\Phi_{\varepsilon}\left(f\right)\right|$
for every
$f\in\mathcal D$.
Combining
\eqref{eq:variation-reduction-symmetry-first-direction}
and
\eqref{eq:variation-reduction-symmetry-second-direction}
yields
\begin{align}
\sup_{f\in\mathcal D}
\left|
\Phi_{\varepsilon}
\left(
f
\right)
\right|
=
\sup_{f\in\mathcal D}
\Phi_{\varepsilon}
\left(
f
\right).
\label{eq:variation-reduction-symmetry-identity}
\end{align}
We now establish the upper bound for the variation hull.
Using
\eqref{eq:variation-reduction-general-variation-hull},
we obtain
\begin{align}
&
\sup_{F\in\mathcal W_R\left(\mathcal D\right)}
\frac{1}{n}
\sum_{i=1}^{n}
\varepsilon_i
F\left(
\b x_i
\right)
 \stackrel{(a)}{=}
\sup_{
\substack{
\mu\in\mathcal M\left(\mathcal D\right)\\
\left\|\mu\right\|_{\mathrm{TV}}\le R
}
}
\frac{1}{n}
\sum_{i=1}^{n}
\varepsilon_i
\int_{\mathcal D}
f\left(
\b x_i
\right)
\,\d\mu\left(
f
\right)
\notag\\
&\stackrel{(b)}{=}
\sup_{
\substack{
\mu\in\mathcal M\left(\mathcal D\right)\\
\left\|\mu\right\|_{\mathrm{TV}}\le R
}
}
\int_{\mathcal D}
\left(
\frac{1}{n}
\sum_{i=1}^{n}
\varepsilon_i
f\left(
\b x_i
\right)
\right)
\,\d\mu\left(
f
\right)
\notag\\
&\stackrel{(c)}{=}
\sup_{
\substack{
\mu\in\mathcal M\left(\mathcal D\right)\\
\left\|\mu\right\|_{\mathrm{TV}}\le R
}
}
\int_{\mathcal D}
\Phi_{\varepsilon}
\left(
f
\right)
\,\d\mu\left(
f
\right)
\stackrel{(d)}{\le}
\sup_{
\substack{
\mu\in\mathcal M\left(\mathcal D\right)\\
\left\|\mu\right\|_{\mathrm{TV}}\le R
}
}
\int_{\mathcal D}
\left|
\Phi_{\varepsilon}
\left(
f
\right)
\right|
\,\d\left|\mu\right|\left(
f
\right)
\notag\\
&\stackrel{(e)}{\le}
\sup_{
\substack{
\mu\in\mathcal M\left(\mathcal D\right)\\
\left\|\mu\right\|_{\mathrm{TV}}\le R
}
}
\left[
\sup_{f\in\mathcal D}
\left|
\Phi_{\varepsilon}
\left(
f
\right)
\right|
\right]
\left|\mu\right|
\left(
\mathcal D
\right)
\stackrel{(f)}{\le}
R
\sup_{f\in\mathcal D}
\left|
\Phi_{\varepsilon}
\left(
f
\right)
\right|
\stackrel{(g)}{=}
R
\sup_{f\in\mathcal D}
\Phi_{\varepsilon}
\left(
f
\right).
\label{eq:variation-reduction-upper-bound}
\end{align}
Here,
(a) substitutes the measure representation from
\eqref{eq:variation-reduction-general-variation-hull},
(b) interchanges a finite sum and the integral,
(c) invokes
\eqref{eq:variation-reduction-linear-functional},
(d) applies the total-variation inequality for finite signed measures,
(e) bounds the integrand by its supremum over
$\mathcal D$,
(f) uses
\begin{align}
\left|\mu\right|
\left(
\mathcal D
\right)
=
\left\|
\mu
\right\|_{\mathrm{TV}}
\le
R,
\end{align}
and
(g) applies
\eqref{eq:variation-reduction-symmetry-identity}. We next prove the reverse inequality.
Set
$ S_{\varepsilon} : \allowbreak= \sup_{f\in\mathcal D} \Phi_{\varepsilon} \left( f \right). $
By
\eqref{eq:variation-reduction-functional-boundedness},
the quantity
$S_{\varepsilon}$
is finite.
Let
$\eta>0$.
By the defining property of the supremum, there exists
$f_{\eta}\in\mathcal D$
such that
\begin{align}
\Phi_{\varepsilon}
\left(
f_{\eta}
\right)
\ge
S_{\varepsilon}
-
\eta.
\label{eq:variation-reduction-approximate-maximizer}
\end{align}
Define the finite positive measure
$ \mu_{\eta} : \allowbreak= R \delta_{f_{\eta}}, $
where
$\delta_{f_{\eta}}$
denotes the Dirac measure at
$f_{\eta}$.
Its total variation satisfies
\begin{align}
\left\|
\mu_{\eta}
\right\|_{\mathrm{TV}}
&\stackrel{(a)}{=}
R
\left\|
\delta_{f_{\eta}}
\right\|_{\mathrm{TV}}
\stackrel{(b)}{=}
R.
\end{align}
Here,
(a) uses the positive homogeneity of the total-variation norm,
while
(b) uses the fact that a Dirac probability measure has total variation
equal to one.
Therefore, the function
\begin{align}
F_{\eta}
\left(
\b x
\right)
&:=
\int_{\mathcal D}
f\left(
\b x
\right)
\,\d\mu_{\eta}\left(
f
\right)
\stackrel{(a)}{=}
R
f_{\eta}
\left(
\b x
\right)
\label{eq:variation-reduction-dirac-function}
\end{align}
belongs to
$\mathcal W_R\left(\mathcal D\right)$.
Here,
(a) applies the defining property of the Dirac measure. Consequently,
\begin{align}
&
\sup_{F\in\mathcal W_R\left(\mathcal D\right)}
\frac{1}{n}
\sum_{i=1}^{n}
\varepsilon_i
F\left(
\b x_i
\right)
\stackrel{(a)}{\ge}
\frac{1}{n}
\sum_{i=1}^{n}
\varepsilon_i
F_{\eta}
\left(
\b x_i
\right)
\stackrel{(b)}{=}
R
\Phi_{\varepsilon}
\left(
f_{\eta}
\right)
\stackrel{(c)}{\ge}
R
\left(
S_{\varepsilon}
-
\eta
\right).
\label{eq:variation-reduction-lower-bound-eta}
\end{align}
Here,
(a) evaluates the supremum at the admissible function
$F_{\eta}$,
(b) applies
\eqref{eq:variation-reduction-dirac-function}
and
\eqref{eq:variation-reduction-linear-functional},
and
(c) invokes
\eqref{eq:variation-reduction-approximate-maximizer}.
Since
\eqref{eq:variation-reduction-lower-bound-eta}
holds for every
$\eta>0$,
letting
$\eta\downarrow0$
gives
\begin{align}
\sup_{F\in\mathcal W_R\left(\mathcal D\right)}
\frac{1}{n}
\sum_{i=1}^{n}
\varepsilon_i
F\left(
\b x_i
\right)
\ge
R
S_{\varepsilon}
=
R
\sup_{f\in\mathcal D}
\Phi_{\varepsilon}
\left(
f
\right).
\label{eq:variation-reduction-lower-bound}
\end{align}
Combining
\eqref{eq:variation-reduction-upper-bound}
and
\eqref{eq:variation-reduction-lower-bound},
we obtain the fixed-Rademacher identity
\begin{align}
&
\sup_{F\in\mathcal W_R\left(\mathcal D\right)}
\frac{1}{n}
\sum_{i=1}^{n}
\varepsilon_i
F\left(
\b x_i
\right)
=
R
\sup_{f\in\mathcal D}
\frac{1}{n}
\sum_{i=1}^{n}
\varepsilon_i
f\left(
\b x_i
\right).
\label{eq:variation-reduction-fixed-rademacher-identity}
\end{align}
Taking expectation with respect to
$\varepsilon_1,\ldots,\varepsilon_n$
in
\eqref{eq:variation-reduction-fixed-rademacher-identity}
yields
\begin{align}
\widehat{\mathfrak R}_n
\left(
\mathcal W_R
\left(
\mathcal D
\right)
\right)
&\stackrel{(a)}{=}
R
\widehat{\mathfrak R}_n
\left(
\mathcal D
\right).
\label{eq:variation-reduction-exact-equality-proof}
\end{align}
Here,
(a) applies the definition of empirical Rademacher complexity to both
classes. It remains to prove the envelope bound.
For every realization of the Rademacher variables,
\begin{align}
\sup_{f\in\mathcal D}
\Phi_{\varepsilon}
\left(
f
\right)
&\stackrel{(a)}{\le}
\sup_{f\in\mathcal D}
\left|
\Phi_{\varepsilon}
\left(
f
\right)
\right|
\stackrel{(b)}{\le}
M_{\b x}
\left(
\mathcal D
\right).
\label{eq:variation-reduction-envelope-proof}
\end{align}
Here,
(a) uses
$s\le\left|s\right|$,
and
(b) invokes
\eqref{eq:variation-reduction-functional-boundedness}.
Taking expectation in
\eqref{eq:variation-reduction-envelope-proof}
and combining the result with
\eqref{eq:variation-reduction-exact-equality-proof}
gives
\begin{align}
\widehat{\mathfrak R}_n
\left(
\mathcal W_R
\left(
\mathcal D
\right)
\right)
\le
R
M_{\b x}
\left(
\mathcal D
\right).
\end{align}
This proves
\eqref{eq:variation-reduction-exact-identity}. We finally specialize the result to the finite outer Brownian
dictionary.
Let
$f\in\mathcal D_L^{m,G}$.
By
\eqref{eq:finite-architecture-outer-dictionary},
there exists
$a\in\mathcal A_{L-1}^{m,G}$
such that $f
\in
\mH_{k_a}$,
$\left\|
f
\right\|_{\mH_{k_a}}
\le
1$.
Since an RKHS unit ball is symmetric,
\begin{align}
\left\|
-f
\right\|_{\mH_{k_a}}
&\stackrel{(a)}{=}
\left\|
f
\right\|_{\mH_{k_a}}
\stackrel{(b)}{\le}
1.
\end{align}
Here,
(a) uses the absolute homogeneity of the Hilbert-space norm,
and
(b) applies $\left\|
f
\right\|_{\mH_{k_a}}
\le
1$.
Thus,
$-f$
belongs to the same RKHS unit ball and therefore
$-f \in \mathcal D_L^{m,G}$.
Since
$f\in\mathcal D_L^{m,G}$
was arbitrary,
$ \mathcal D_L^{m,G} \allowbreak= -\mathcal D_L^{m,G}. $
For every
$i\in\left[n\right]$,
the reproducing property in the fixed RKHS
$\mH_{k_a}$
gives
\begin{align}
\left|
f\left(
\b x_i
\right)
\right|
&\stackrel{(a)}{=}
\left|
\left\langle
f,
k_a\left(
\b x_i,\cdot
\right)
\right\rangle_{\mH_{k_a}}
\right|
\stackrel{(b)}{\le}
\left\|
f
\right\|_{\mH_{k_a}}
\left\|
k_a\left(
\b x_i,\cdot
\right)
\right\|_{\mH_{k_a}}
\stackrel{(c)}{=}
\left\|
f
\right\|_{\mH_{k_a}}
k_a\left(
\b x_i,\b x_i
\right)^{1/2}
\notag\\
&\stackrel{(d)}{\le}
k_a\left(
\b x_i,\b x_i
\right)^{1/2}
\stackrel{(e)}{=}
\left|
a\left(
\b x_i
\right)
\right|^{1/2}
\stackrel{(f)}{\le}
B_{\b x}^{1/2}.
\label{eq:variation-reduction-finite-outer-pointwise}
\end{align}
Here,
(a) applies the reproducing property,
(b) applies the Cauchy--Schwarz inequality in
$\mH_{k_a}$,
(c) applies the RKHS section-norm identity,
(d) uses
$\left\|f\right\|_{\mH_{k_a}}\le1$,
(e) uses
\begin{align}
k_a\left(
\b x_i,\b x_i
\right)
&\stackrel{(a)}{=}
\kB\left(
a\left(
\b x_i
\right),
a\left(
\b x_i
\right)
\right)
\stackrel{(b)}{=}
\left|
a\left(
\b x_i
\right)
\right|,
\end{align}
where
(a) is the definition of the Brownian pullback kernel and
(b) uses
$\kB\left(s,s\right)=\left|s\right|$,
and
(f) invokes
\eqref{eq:variation-reduction-finite-support-envelope}. Taking the maximum over
$i\in\left[n\right]$
and then the supremum over
$f\in\mathcal D_L^{m,G}$
in
\eqref{eq:variation-reduction-finite-outer-pointwise}
gives
$M_{\b x} \left( \mathcal D_L^{m,G} \right) \le B_{\b x}^{1/2}$,
which proves
\eqref{eq:variation-reduction-outer-envelope}. Finally,
\eqref{eq:finite-architecture-variation-ball}
and
\eqref{eq:variation-reduction-general-variation-hull}
give
$\mathcal W_R^{L;m,G} = \mathcal W_R \left( \mathcal D_L^{m,G} \right)$.
Applying
\eqref{eq:variation-reduction-exact-identity}
with
$\mathcal D=\mathcal D_L^{m,G}$
and using
\eqref{eq:variation-reduction-outer-envelope}
proves
\eqref{eq:variation-reduction-finite-architecture}.
This completes the proof.

\end{proof}

\begin{lemma}[Rademacher bound for unions of Brownian pullback RKHS balls\kern-3.5pt]\leavevmode\linebreak
\label{lem:brownian-pullback-rademacher}
Let
$n\in\mathbb N$,
$n\ge1$,
and let
$\b x_1,\ldots,\b x_n\in\mX$
be fixed sample points.
Let
$\mathcal A$
be a nonempty class of real-valued support functions on
$\mX$.
For every
$a\in\mathcal A$,
define
\begin{align}
k_a\left(
\b x,
\b x'
\right)
:=
\kB\left(
a\left(
\b x
\right),
a\left(
\b x'
\right)
\right),
\qquad
\b x,\b x'\in\mX.
\label{eq:brownian-pullback-rademacher-kernel}
\end{align}
Assume that
\begin{align}
B_{\b x}\left(
\mathcal A
\right)
:=
\sup_{a\in\mathcal A}
\max_{i\in\left[
n
\right]}
k_a\left(
\b x_i,
\b x_i
\right)
<
\infty.
\label{eq:brownian-pullback-diagonal-envelope}
\end{align}
Define
\begin{align}
\widehat{\mathfrak C}_{n,\b x}^{(B)}
\left(
\mathcal A
\right)
:=
\mathbb E_{\varepsilon}
\left[
\sup_{a\in\mathcal A}
\left|
\frac{1}{n}
\sum_{1\le i<j\le n}
\varepsilon_i
\varepsilon_j
k_a\left(
\b x_i,
\b x_j
\right)
\right|
\right],
\label{eq:empirical-brownian-chaos-complexity}
\end{align}
where
$\varepsilon_1,\ldots,\varepsilon_n$
are independent Rademacher random variables. For
$r>0$,
let
\begin{align}
\mathcal D_r\left(
\mathcal A
\right)
:=
\bigcup_{a\in\mathcal A}
\left\{
f\in\mH_{k_a}
:
\left\|
f
\right\|_{\mH_{k_a}}
\le
r
\right\}.
\label{eq:radius-r-pullback-rkhs-union}
\end{align}
Then
$\mathcal D_r\left(
\mathcal A
\right)$
is symmetric and
\begin{align}
\widehat{\mathfrak R}_n
\left(
\mathcal D_r\left(
\mathcal A
\right)
\right)
&=
\frac{r}{n}
\mathbb E_{\varepsilon}
\left[
\sup_{a\in\mathcal A}
\left(
\sum_{i,j=1}^{n}
\varepsilon_i
\varepsilon_j
k_a\left(
\b x_i,
\b x_j
\right)
\right)^{1/2}
\right]
\label{eq:brownian-pullback-rademacher-exact}
\\
&\le
\frac{r}{\sqrt n}
\left(
2
\widehat{\mathfrak C}_{n,\b x}^{(B)}
\left(
\mathcal A
\right)
+
B_{\b x}\left(
\mathcal A
\right)
\right)^{1/2}.
\label{eq:brownian-pullback-rademacher-bound}
\end{align}

In particular, fix
$L\ge2$,
$m\ge1$,
and
$G\ge2$,
and set
\begin{align}
B_{\b x}
:=
\sup_{a\in\mathcal A_{L-1}^{m,G}}
\max_{i\in\left[
n
\right]}
\left|
a\left(
\b x_i
\right)
\right|.
\label{eq:finite-architecture-rkhs-union-support-envelope}
\end{align}
Then
\begin{align}
\mathcal D_1\left(
\mathcal A_{L-1}^{m,G}
\right)
=
\mathcal D_L^{m,G},
\label{eq:finite-architecture-dictionary-as-rkhs-union}
\end{align}
and
\begin{align}
B_{\b x}\left(
\mathcal A_{L-1}^{m,G}
\right)
=
B_{\b x}.
\label{eq:finite-architecture-diagonal-envelope-identity}
\end{align}
Consequently,
\begin{align}
\widehat{\mathfrak R}_n
\left(
\mathcal D_L^{m,G}
\right)
\le
\frac{1}{\sqrt n}
\left(
2
\widehat{\mathfrak C}_{n,\b x}^{(B)}
\left(
\mathcal A_{L-1}^{m,G}
\right)
+
B_{\b x}
\right)^{1/2}.
\label{eq:finite-architecture-dictionary-rademacher-bound}
\end{align}

\end{lemma}

\begin{proof}
We divide the proof into four steps.

\paragraph{Step 1: symmetry of the union class.}

Fix
$f\in\mathcal D_r\left(\mathcal A\right)$.
By
\eqref{eq:radius-r-pullback-rkhs-union},
there exists
$a\in\mathcal A$
such that
\begin{align}
f
&\in
\mH_{k_a},
&
\left\|
f
\right\|_{\mH_{k_a}}
&\le
r.
\label{eq:brownian-pullback-rademacher-fixed-ball-membership}
\end{align}
Since
$\mH_{k_a}$
is a vector space,
$-f \in \mH_{k_a}$.
Moreover,
\begin{align}
\left\|
-f
\right\|_{\mH_{k_a}}
&\stackrel{(a)}{=}
\left|
-1
\right|
\left\|
f
\right\|_{\mH_{k_a}}
\stackrel{(b)}{=}
\left\|
f
\right\|_{\mH_{k_a}}
\stackrel{(c)}{\le}
r.
\end{align}
Here,
(a) applies the absolute homogeneity of the Hilbert-space norm,
(b) uses
$\left|-1\right|=1$,
and
(c) invokes
\eqref{eq:brownian-pullback-rademacher-fixed-ball-membership}.
Therefore,
$-f \in \mathcal D_r \left( \mathcal A \right)$.
Since
$f\in\mathcal D_r\left(\mathcal A\right)$
was arbitrary,
\begin{align}
\mathcal D_r
\left(
\mathcal A
\right)
=
-
\mathcal D_r
\left(
\mathcal A
\right).
\label{eq:brownian-pullback-rademacher-union-symmetry}
\end{align}
Fix a realization
$\varepsilon_1,\ldots,\varepsilon_n$
of the Rademacher variables and define
\begin{align}
\Lambda_{\varepsilon}
\left(
f
\right)
:=
\frac{1}{n}
\sum_{i=1}^{n}
\varepsilon_i
f\left(
\b x_i
\right),
\qquad
f\in
\mathcal D_r
\left(
\mathcal A
\right).
\end{align}
By
\eqref{eq:brownian-pullback-rademacher-union-symmetry},
\begin{align}
\sup_{f\in\mathcal D_r\left(\mathcal A\right)}
\left|
\Lambda_{\varepsilon}
\left(
f
\right)
\right|
&\stackrel{(a)}{=}
\sup_{f\in\mathcal D_r\left(\mathcal A\right)}
\max
\left\{
\Lambda_{\varepsilon}
\left(
f
\right),
-
\Lambda_{\varepsilon}
\left(
f
\right)
\right\}
 \stackrel{(b)}{=}
\sup_{f\in\mathcal D_r\left(\mathcal A\right)}
\max
\left\{
\Lambda_{\varepsilon}
\left(
f
\right),
\Lambda_{\varepsilon}
\left(
-f
\right)
\right\}
\notag\\
&\stackrel{(c)}{=}
\sup_{f\in\mathcal D_r\left(\mathcal A\right)}
\Lambda_{\varepsilon}
\left(
f
\right).
\label{eq:brownian-pullback-rademacher-absolute-supremum}
\end{align}
Here,
(a) uses
$\left|s\right|=\max\left\{s,-s\right\}$,
(b) uses the linearity of
$\Lambda_{\varepsilon}$,
and
(c) uses the symmetry
\eqref{eq:brownian-pullback-rademacher-union-symmetry}.
Thus, the absolute-value form may be used without changing the
empirical Rademacher complexity.

\paragraph{Step 2: exact RKHS reduction for a fixed support.}

Fix
$a\in\mathcal A$
and define
\begin{align}
V_{a,\varepsilon}
:=
\sum_{i=1}^{n}
\varepsilon_i
k_a
\left(
\b x_i,
\cdot
\right)
\in
\mH_{k_a}.
\label{eq:brownian-pullback-rademacher-representer}
\end{align}
For every
$f\in\mH_{k_a}$,
the reproducing property gives
\begin{align}
\sum_{i=1}^{n}
\varepsilon_i
f\left(
\b x_i
\right)
&\stackrel{(a)}{=}
\sum_{i=1}^{n}
\varepsilon_i
\left\langle
f,
k_a
\left(
\b x_i,
\cdot
\right)
\right\rangle_{\mH_{k_a}}
 \stackrel{(b)}{=}
\left\langle
f,
\sum_{i=1}^{n}
\varepsilon_i
k_a
\left(
\b x_i,
\cdot
\right)
\right\rangle_{\mH_{k_a}}
 \stackrel{(c)}{=}
\left\langle
f,
V_{a,\varepsilon}
\right\rangle_{\mH_{k_a}}.
\label{eq:brownian-pullback-rademacher-reproducing-reduction}
\end{align}
Here,
(a) applies the reproducing property in the fixed RKHS
$\mH_{k_a}$,
(b) uses the linearity of the inner product in its second argument,
and
(c) invokes
\eqref{eq:brownian-pullback-rademacher-representer}. Consequently,
\begin{align}
&
\sup_{
\substack{
f\in\mH_{k_a}\\
\left\|f\right\|_{\mH_{k_a}}\le r
}
}
\left|
\frac{1}{n}
\sum_{i=1}^{n}
\varepsilon_i
f\left(
\b x_i
\right)
\right|
\stackrel{(a)}{=}
\frac{1}{n}
\sup_{
\substack{
f\in\mH_{k_a}\\
\left\|f\right\|_{\mH_{k_a}}\le r
}
}
\left|
\left\langle
f,
V_{a,\varepsilon}
\right\rangle_{\mH_{k_a}}
\right|
\notag\\
&\stackrel{(b)}{=}
\frac{r}{n}
\left\|
V_{a,\varepsilon}
\right\|_{\mH_{k_a}}.
\label{eq:brownian-pullback-rademacher-dual-norm}
\end{align}
Here,
(a) applies
\eqref{eq:brownian-pullback-rademacher-reproducing-reduction},
and
(b) applies the Hilbert-space dual-norm identity.
More explicitly, the Cauchy--Schwarz inequality gives the upper bound
\begin{align}
\left|
\left\langle
f,
V_{a,\varepsilon}
\right\rangle_{\mH_{k_a}}
\right|
\le
r
\left\|
V_{a,\varepsilon}
\right\|_{\mH_{k_a}},
\end{align}
and equality is attained, whenever
$V_{a,\varepsilon}\neq0$,
by choosing
$f = r \frac{ V_{a,\varepsilon} }{ \left\| V_{a,\varepsilon} \right\|_{\mH_{k_a}} }$.
When
$V_{a,\varepsilon}=0$,
both sides of
\eqref{eq:brownian-pullback-rademacher-dual-norm}
are zero. The squared norm in
\eqref{eq:brownian-pullback-rademacher-dual-norm}
satisfies
\begin{align}
\left\|
V_{a,\varepsilon}
\right\|_{\mH_{k_a}}^2
&\stackrel{(a)}{=}
\left\langle
\sum_{i=1}^{n}
\varepsilon_i
k_a
\left(
\b x_i,
\cdot
\right),
\sum_{j=1}^{n}
\varepsilon_j
k_a
\left(
\b x_j,
\cdot
\right)
\right\rangle_{\mH_{k_a}}
 \stackrel{(b)}{=}
\sum_{i,j=1}^{n}
\varepsilon_i
\varepsilon_j
\left\langle
k_a
\left(
\b x_i,
\cdot
\right),
k_a
\left(
\b x_j,
\cdot
\right)
\right\rangle_{\mH_{k_a}}
\notag\\
&\stackrel{(c)}{=}
\sum_{i,j=1}^{n}
\varepsilon_i
\varepsilon_j
k_a
\left(
\b x_i,
\b x_j
\right).
\label{eq:brownian-pullback-rademacher-quadratic-form}
\end{align}
Here,
(a) substitutes
\eqref{eq:brownian-pullback-rademacher-representer},
(b) expands the inner product by bilinearity,
and
(c) applies the RKHS kernel-section identity.
In particular,
\begin{align}
\sum_{i,j=1}^{n}
\varepsilon_i
\varepsilon_j
k_a
\left(
\b x_i,
\b x_j
\right)
\ge
0,
\label{eq:brownian-pullback-rademacher-quadratic-nonnegative}
\end{align}
because it is the squared Hilbert-space norm appearing on the left-hand
side of
\eqref{eq:brownian-pullback-rademacher-quadratic-form}. Combining
\eqref{eq:brownian-pullback-rademacher-dual-norm}
and
\eqref{eq:brownian-pullback-rademacher-quadratic-form}
gives
\begin{align}
&
\sup_{
\substack{
f\in\mH_{k_a}\\
\left\|f\right\|_{\mH_{k_a}}\le r
}
}
\left|
\frac{1}{n}
\sum_{i=1}^{n}
\varepsilon_i
f\left(
\b x_i
\right)
\right|
=
\frac{r}{n}
\left(
\sum_{i,j=1}^{n}
\varepsilon_i
\varepsilon_j
k_a
\left(
\b x_i,
\b x_j
\right)
\right)^{1/2}.
\label{eq:brownian-pullback-rademacher-fixed-support-exact}
\end{align}

\paragraph{Step 3: reduction of the union to Brownian quadratic chaos.}

Using
\eqref{eq:radius-r-pullback-rkhs-union},
\eqref{eq:brownian-pullback-rademacher-absolute-supremum},
and
\eqref{eq:brownian-pullback-rademacher-fixed-support-exact},
we obtain
\begin{align}
&
\sup_{f\in\mathcal D_r\left(\mathcal A\right)}
\frac{1}{n}
\sum_{i=1}^{n}
\varepsilon_i
f\left(
\b x_i
\right)
\stackrel{(a)}{=}
\sup_{f\in\mathcal D_r\left(\mathcal A\right)}
\left|
\frac{1}{n}
\sum_{i=1}^{n}
\varepsilon_i
f\left(
\b x_i
\right)
\right|
\notag\\
&\stackrel{(b)}{=}
\sup_{a\in\mathcal A}
\sup_{
\substack{
f\in\mH_{k_a}\\
\left\|f\right\|_{\mH_{k_a}}\le r
}
}
\left|
\frac{1}{n}
\sum_{i=1}^{n}
\varepsilon_i
f\left(
\b x_i
\right)
\right|
\stackrel{(c)}{=}
\frac{r}{n}
\sup_{a\in\mathcal A}
\left(
\sum_{i,j=1}^{n}
\varepsilon_i
\varepsilon_j
k_a
\left(
\b x_i,
\b x_j
\right)
\right)^{1/2}.
\label{eq:brownian-pullback-rademacher-fixed-epsilon-union}
\end{align}
Here,
(a) applies
\eqref{eq:brownian-pullback-rademacher-absolute-supremum},
(b) applies the union definition
\eqref{eq:radius-r-pullback-rkhs-union},
and
(c) applies
\eqref{eq:brownian-pullback-rademacher-fixed-support-exact}
for each fixed
$a\in\mathcal A$. Taking expectation with respect to the Rademacher variables in
\eqref{eq:brownian-pullback-rademacher-fixed-epsilon-union}
proves
\eqref{eq:brownian-pullback-rademacher-exact}. For notational convenience, define
\begin{align}
Q_a
\left(
\varepsilon
\right)
:=
\sum_{i,j=1}^{n}
\varepsilon_i
\varepsilon_j
k_a
\left(
\b x_i,
\b x_j
\right).
\end{align}
By
\eqref{eq:brownian-pullback-rademacher-quadratic-nonnegative},
one has
$Q_a \left( \varepsilon \right) \ge 0$, $a\in\mathcal A$.
Since the square-root function is increasing on
$\left[0,\infty\right)$,
\begin{align}
\sup_{a\in\mathcal A}
Q_a
\left(
\varepsilon
\right)^{1/2}
&\stackrel{(a)}{=}
\left(
\sup_{a\in\mathcal A}
Q_a
\left(
\varepsilon
\right)
\right)^{1/2}.
\label{eq:brownian-pullback-rademacher-supremum-square-root}
\end{align}
Here,
(a) uses the monotonicity and continuity of the square-root function. Applying Jensen's inequality to the concave function
$t\mapsto t^{1/2}$
gives
\begin{align}
&
\mathbb E_{\varepsilon}
\left[
\sup_{a\in\mathcal A}
Q_a
\left(
\varepsilon
\right)^{1/2}
\right]
\stackrel{(a)}{=}
\mathbb E_{\varepsilon}
\left[
\left(
\sup_{a\in\mathcal A}
Q_a
\left(
\varepsilon
\right)
\right)^{1/2}
\right]
\stackrel{(b)}{\le}
\left(
\mathbb E_{\varepsilon}
\left[
\sup_{a\in\mathcal A}
Q_a
\left(
\varepsilon
\right)
\right]
\right)^{1/2}.
\label{eq:brownian-pullback-rademacher-jensen}
\end{align}
Here,
(a) applies
\eqref{eq:brownian-pullback-rademacher-supremum-square-root},
and
(b) applies Jensen's inequality.

\paragraph{Step 4: diagonal and off-diagonal decomposition.}

For every
$a\in\mathcal A$,
\begin{align}
Q_a
\left(
\varepsilon
\right)
&\stackrel{(a)}{=}
\sum_{i=1}^{n}
\varepsilon_i^2
k_a
\left(
\b x_i,
\b x_i
\right)
+
2
\sum_{1\le i<j\le n}
\varepsilon_i
\varepsilon_j
k_a
\left(
\b x_i,
\b x_j
\right)
\notag\\
&\stackrel{(b)}{=}
\sum_{i=1}^{n}
k_a
\left(
\b x_i,
\b x_i
\right)
+
2
\sum_{1\le i<j\le n}
\varepsilon_i
\varepsilon_j
k_a
\left(
\b x_i,
\b x_j
\right).
\label{eq:brownian-pullback-rademacher-diagonal-offdiagonal}
\end{align}
Here,
(a) separates the diagonal and off-diagonal terms in the double sum,
and
(b) uses
$\varepsilon_i^2=1$
for every
$i\in\left[n\right]$. Consequently,
\begin{align}
\sup_{a\in\mathcal A}
Q_a
\left(
\varepsilon
\right)
&\stackrel{(a)}{\le}
2
\sup_{a\in\mathcal A}
\left|
\sum_{1\le i<j\le n}
\varepsilon_i
\varepsilon_j
k_a
\left(
\b x_i,
\b x_j
\right)
\right|
+
\sup_{a\in\mathcal A}
\sum_{i=1}^{n}
k_a
\left(
\b x_i,
\b x_i
\right)
\notag\\
&\stackrel{(b)}{\le}
2
\sup_{a\in\mathcal A}
\left|
\sum_{1\le i<j\le n}
\varepsilon_i
\varepsilon_j
k_a
\left(
\b x_i,
\b x_j
\right)
\right|
+
n
B_{\b x}
\left(
\mathcal A
\right).
\label{eq:brownian-pullback-rademacher-quadratic-upper-bound}
\end{align}
Here,
(a) applies
\eqref{eq:brownian-pullback-rademacher-diagonal-offdiagonal}
and bounds the off-diagonal term by its absolute value,
while
(b) applies
\eqref{eq:brownian-pullback-diagonal-envelope}
to every diagonal kernel value. Taking expectation in
\eqref{eq:brownian-pullback-rademacher-quadratic-upper-bound}
gives
\begin{align}
&
\mathbb E_{\varepsilon}
\left[
\sup_{a\in\mathcal A}
Q_a
\left(
\varepsilon
\right)
\right]
\stackrel{(a)}{\le}
2
\mathbb E_{\varepsilon}
\left[
\sup_{a\in\mathcal A}
\left|
\sum_{1\le i<j\le n}
\varepsilon_i
\varepsilon_j
k_a
\left(
\b x_i,
\b x_j
\right)
\right|
\right]
+
n
B_{\b x}
\left(
\mathcal A
\right)
\notag\\
&\stackrel{(b)}{=}
2n
\widehat{\mathfrak C}_{n,\b x}^{(B)}
\left(
\mathcal A
\right)
+
n
B_{\b x}
\left(
\mathcal A
\right)
\stackrel{(c)}{=}
n
\left(
2
\widehat{\mathfrak C}_{n,\b x}^{(B)}
\left(
\mathcal A
\right)
+
B_{\b x}
\left(
\mathcal A
\right)
\right).
\label{eq:brownian-pullback-rademacher-expected-quadratic-bound}
\end{align}
Here,
(a) takes expectations in
\eqref{eq:brownian-pullback-rademacher-quadratic-upper-bound},
(b) applies the normalization in
\eqref{eq:empirical-brownian-chaos-complexity},
and
(c) factors out
$n$. Combining
\eqref{eq:brownian-pullback-rademacher-exact},
\eqref{eq:brownian-pullback-rademacher-jensen},
and
\eqref{eq:brownian-pullback-rademacher-expected-quadratic-bound},
we obtain
\begin{align}
\widehat{\mathfrak R}_n
\left(
\mathcal D_r
\left(
\mathcal A
\right)
\right)
&\le
\frac{r}{n}
\left[
n
\left(
2
\widehat{\mathfrak C}_{n,\b x}^{(B)}
\left(
\mathcal A
\right)
+
B_{\b x}
\left(
\mathcal A
\right)
\right)
\right]^{1/2}
\notag\\
&=
\frac{r}{\sqrt n}
\left(
2
\widehat{\mathfrak C}_{n,\b x}^{(B)}
\left(
\mathcal A
\right)
+
B_{\b x}
\left(
\mathcal A
\right)
\right)^{1/2}.
\end{align}
Here, the inequality substitutes
\eqref{eq:brownian-pullback-rademacher-expected-quadratic-bound}
into the Jensen estimate, and the equality simplifies
$n^{1/2}/n=n^{-1/2}$.
This proves
\eqref{eq:brownian-pullback-rademacher-bound}. We finally specialize the result to the finite lower-support
architecture.
By
\eqref{eq:finite-architecture-outer-dictionary},
\begin{align}
\mathcal D_L^{m,G}
&\stackrel{(a)}{=}
\bigcup_{a\in\mathcal A_{L-1}^{m,G}}
\left\{
f\in\mH_{k_a}
:
\left\|
f
\right\|_{\mH_{k_a}}
\le
1
\right\}
 \stackrel{(b)}{=}
\mathcal D_1
\left(
\mathcal A_{L-1}^{m,G}
\right).
\end{align}
Here,
(a) applies the definition of
$\mathcal D_L^{m,G}$,
and
(b) applies
\eqref{eq:radius-r-pullback-rkhs-union}
with
$r=1$.
This proves
\eqref{eq:finite-architecture-dictionary-as-rkhs-union}. Moreover, for every
$a\in\mathcal A_{L-1}^{m,G}$
and
$i\in\left[n\right]$,
\begin{align}
k_a
\left(
\b x_i,
\b x_i
\right)
&\stackrel{(a)}{=}
\kB
\left(
a\left(
\b x_i
\right),
a\left(
\b x_i
\right)
\right)
\stackrel{(b)}{=}
\left|
a\left(
\b x_i
\right)
\right|.
\label{eq:brownian-pullback-rademacher-finite-diagonal}
\end{align}
Here,
(a) applies
\eqref{eq:brownian-pullback-rademacher-kernel},
and
(b) uses
$\kB \left( s,s \right) = \left| s \right|$, $s\in\mathbb R$.
Taking the maximum over
$i\in\left[n\right]$
and then the supremum over
$a\in\mathcal A_{L-1}^{m,G}$
in
\eqref{eq:brownian-pullback-rademacher-finite-diagonal}
proves
\eqref{eq:finite-architecture-diagonal-envelope-identity}. Finally, apply
\eqref{eq:brownian-pullback-rademacher-bound}
with
$\mathcal A = \mathcal A_{L-1}^{m,G}, r = 1$.
Using
\eqref{eq:finite-architecture-dictionary-as-rkhs-union}
and
\eqref{eq:finite-architecture-diagonal-envelope-identity}
gives
\eqref{eq:finite-architecture-dictionary-rademacher-bound}.
This completes the proof.

\end{proof}

\begin{lemma}\textnormal{(Finite threshold-chaos bound)}\ 
\label{lem:self-contained-threshold-chaos}
Let
$n\ge1$,
and let
$\mathcal T$
be a nonempty finite collection of subsets of
$\left[n\right]$.
Set
$ N : \allowbreak= \left| \mathcal T \right| \ge 1. $
For every
$T\in\mathcal T$,
define
\begin{align}
Q_T
\left(
\varepsilon
\right)
:=
\frac{1}{n}
\sum_{1\le i<j\le n}
\varepsilon_i
\varepsilon_j
\mathbf 1_{\left\{i\in T\right\}}
\mathbf 1_{\left\{j\in T\right\}},
\label{eq:finite-threshold-chaos-variable}
\end{align}
where
$\varepsilon_1,\ldots,\varepsilon_n$
are independent Rademacher random variables.
Then there exists a universal numerical constant
$C>0$
such that
\begin{align}
\mathbb E_{\varepsilon}
\left[
\max_{T\in\mathcal T}
\left|
Q_T
\left(
\varepsilon
\right)
\right|
\right]
\le
C
\ln
\left(
1+N
\right).
\label{eq:finite-threshold-chaos-bound}
\end{align}

\end{lemma}

\begin{proof}
For every
$T\in\mathcal T$,
define
\begin{align}
S_T
:=
\sum_{i\in T}
\varepsilon_i
=
\sum_{i=1}^{n}
\varepsilon_i
\mathbf 1_{\left\{i\in T\right\}}.
\label{eq:finite-threshold-chaos-linear-sum}
\end{align}
Expanding the square gives
\begin{align}
S_T^2
&\stackrel{(a)}{=}
\sum_{i=1}^{n}
\varepsilon_i^2
\mathbf 1_{\left\{i\in T\right\}}
+
2
\sum_{1\le i<j\le n}
\varepsilon_i
\varepsilon_j
\mathbf 1_{\left\{i\in T\right\}}
\mathbf 1_{\left\{j\in T\right\}}
 \stackrel{(b)}{=}
\left|
T
\right|
+
2n
Q_T
\left(
\varepsilon
\right).
\label{eq:finite-threshold-chaos-square-expansion}
\end{align}
Here,
(a) expands the square of
\eqref{eq:finite-threshold-chaos-linear-sum},
while
(b) uses
$\varepsilon_i^2 = 1, \qquad i\in\left[n\right]$,
and applies the definition
\eqref{eq:finite-threshold-chaos-variable}.
Rearranging
\eqref{eq:finite-threshold-chaos-square-expansion}
yields
\begin{align}
Q_T
\left(
\varepsilon
\right)
=
\frac{
S_T^2-\left|T\right|
}{
2n
}.
\label{eq:finite-threshold-chaos-square-identity}
\end{align}
Consequently,
\begin{align}
\left|
Q_T
\left(
\varepsilon
\right)
\right|
&\stackrel{(a)}{=}
\frac{1}{2n}
\left|
S_T^2-\left|T\right|
\right|
 \stackrel{(b)}{\le}
\frac{
S_T^2
}{
2n
}
+
\frac{
\left|T\right|
}{
2n
}
\stackrel{(c)}{\le}
\frac{
S_T^2
}{
2n
}
+
\frac{1}{2}.
\label{eq:finite-threshold-chaos-pointwise-reduction}
\end{align}
Here,
(a) applies
\eqref{eq:finite-threshold-chaos-square-identity},
(b) applies the triangle inequality, and
(c) uses
$\left| T \right| \le n$.
Taking the maximum over
$T\in\mathcal T$
in
\eqref{eq:finite-threshold-chaos-pointwise-reduction}
gives
\begin{align}
\max_{T\in\mathcal T}
\left|
Q_T
\left(
\varepsilon
\right)
\right|
\le
\frac{1}{2n}
\max_{T\in\mathcal T}
S_T^2
+
\frac{1}{2}.
\label{eq:finite-threshold-chaos-maximum-reduction}
\end{align}
We next derive a uniform tail bound for the random variables
$S_T$.
Fix
$T\in\mathcal T$
and
$\lambda\in\mathbb R$.
Independence of the Rademacher variables gives
\begin{align}
\mathbb E_{\varepsilon}
\left[
\exp
\left(
\lambda S_T
\right)
\right]
&\stackrel{(a)}{=}
\prod_{i\in T}
\mathbb E_{\varepsilon_i}
\left[
\exp
\left(
\lambda\varepsilon_i
\right)
\right]
 \stackrel{(b)}{=}
\left(
\cosh
\left(
\lambda
\right)
\right)^{\left|T\right|}
\stackrel{(c)}{\le}
\exp
\left(
\frac{
\lambda^2\left|T\right|
}{
2
}
\right)\\&
 \stackrel{(d)}{\le}
\exp
\left(
\frac{
\lambda^2n
}{
2
}
\right).
\label{eq:finite-threshold-chaos-mgf-bound}
\end{align}
Here,
(a) uses independence,
(b) uses $\mathbb E_{\varepsilon_i}
\left[
\exp
\left(
\lambda\varepsilon_i
\right)
\right]
=
\frac{
e^\lambda+e^{-\lambda}
}{
2
}
=
\cosh
\left(
\lambda
\right)$, (c) applies the elementary inequality $\cosh
\left(
\lambda
\right)
\le
\exp
\left(
\frac{\lambda^2}{2}
\right)$, $\lambda\in\mathbb R$, and
(d) uses
$\left|T\right|\le n$. Let
$t>0$.
For every
$\lambda>0$,
Markov's inequality and
\eqref{eq:finite-threshold-chaos-mgf-bound}
give
\begin{align}
\mathbb P_{\varepsilon}
\left(
S_T
\ge
t
\right)
&\stackrel{(a)}{=}
\mathbb P_{\varepsilon}
\left(
\exp
\left(
\lambda S_T
\right)
\ge
\exp
\left(
\lambda t
\right)
\right)
 \stackrel{(b)}{\le}
\exp
\left(
-\lambda t
\right)
\mathbb E_{\varepsilon}
\left[
\exp
\left(
\lambda S_T
\right)
\right]\\&
 \stackrel{(c)}{\le}
\exp
\left(
-\lambda t
+
\frac{
\lambda^2n
}{
2
}
\right).
\label{eq:finite-threshold-chaos-positive-tail-lambda}
\end{align}
Here,
(a) uses the strict monotonicity of the exponential function,
(b) applies Markov's inequality, and
(c) applies
\eqref{eq:finite-threshold-chaos-mgf-bound}.
Choosing
$ \lambda : \allowbreak= \frac{t}{n} $
in
\eqref{eq:finite-threshold-chaos-positive-tail-lambda}
yields
\begin{align}
\mathbb P_{\varepsilon}
\left(
S_T
\ge
t
\right)
\le
\exp
\left(
-\frac{
t^2
}{
2n
}
\right).
\label{eq:finite-threshold-chaos-positive-tail}
\end{align}
Applying the same argument to
$-S_T$
gives
\begin{align}
\mathbb P_{\varepsilon}
\left(
S_T
\le
-t
\right)
\le
\exp
\left(
-\frac{
t^2
}{
2n
}
\right).
\label{eq:finite-threshold-chaos-negative-tail}
\end{align}
Therefore,
\begin{align}
\mathbb P_{\varepsilon}
\left(
\left|
S_T
\right|
\ge
t
\right)
&\stackrel{(a)}{\le}
\mathbb P_{\varepsilon}
\left(
S_T
\ge
t
\right)
+
\mathbb P_{\varepsilon}
\left(
S_T
\le
-t
\right)
 \stackrel{(b)}{\le}
2
\exp
\left(
-\frac{
t^2
}{
2n
}
\right).
\label{eq:finite-threshold-chaos-two-sided-tail}
\end{align}
Here,
(a) applies the union bound, and
(b) combines
\eqref{eq:finite-threshold-chaos-positive-tail}
and
\eqref{eq:finite-threshold-chaos-negative-tail}. Applying the union bound over the
$N$
sets in
$\mathcal T$
gives
\begin{align}
\mathbb P_{\varepsilon}
\left(
\max_{T\in\mathcal T}
\left|
S_T
\right|
\ge
t
\right)
&\stackrel{(a)}{\le}
\sum_{T\in\mathcal T}
\mathbb P_{\varepsilon}
\left(
\left|
S_T
\right|
\ge
t
\right)
 \stackrel{(b)}{\le}
2N
\exp
\left(
-\frac{
t^2
}{
2n
}
\right).
\label{eq:finite-threshold-chaos-union-tail}
\end{align}
Here,
(a) applies the union bound, while
(b) applies
\eqref{eq:finite-threshold-chaos-two-sided-tail}
to every
$T\in\mathcal T$. Define
\begin{align}
Z
:=
\max_{T\in\mathcal T}
S_T^2.
\label{eq:finite-threshold-chaos-max-square}
\end{align}
Since
$Z\ge0$,
the tail-integral representation gives
\begin{align}
\mathbb E_{\varepsilon}
\left[
Z
\right]
=
\int_0^\infty
\mathbb P_{\varepsilon}
\left(
Z
\ge
s
\right)
\,\d s.
\label{eq:finite-threshold-chaos-tail-integral}
\end{align}
Moreover,
\begin{align}
\left\{
Z
\ge
s
\right\}
&\stackrel{(a)}{=}
\left\{
\max_{T\in\mathcal T}
\left|
S_T
\right|
\ge
s^{1/2}
\right\}.
\end{align}
Here,
(a) follows from the definition
\eqref{eq:finite-threshold-chaos-max-square}.
Consequently,
\eqref{eq:finite-threshold-chaos-union-tail}
gives
\begin{align}
\mathbb P_{\varepsilon}
\left(
Z
\ge
s
\right)
\le
2N
\exp
\left(
-\frac{
s
}{
2n
}
\right).
\label{eq:finite-threshold-chaos-square-tail}
\end{align}
Set
$ s_0 : \allowbreak= 2n \ln \left( 2N \right). $
Since probabilities are bounded above by one,
\begin{align}
\mathbb E_{\varepsilon}
\left[
Z
\right]
&\stackrel{(a)}{=}
\int_0^{s_0}
\mathbb P_{\varepsilon}
\left(
Z
\ge
s
\right)
\,\d s
+
\int_{s_0}^{\infty}
\mathbb P_{\varepsilon}
\left(
Z
\ge
s
\right)
\,\d s
\notag\\
&\stackrel{(b)}{\le}
s_0
+
\int_{s_0}^{\infty}
2N
\exp
\left(
-\frac{
s
}{
2n
}
\right)
\,\d s
\stackrel{(c)}{=}
2n
\ln
\left(
2N
\right)
+
4nN
\exp
\left(
-\frac{
s_0
}{
2n
}
\right)
\notag\\
&\stackrel{(d)}{=}
2n
\ln
\left(
2N
\right)
+
2n.
\label{eq:finite-threshold-chaos-second-moment-max}
\end{align}
Here,
(a) splits the integral in
\eqref{eq:finite-threshold-chaos-tail-integral}
at
$s_0$,
(b) bounds the first probability by one and applies
\eqref{eq:finite-threshold-chaos-square-tail}
to the second integral,
(c) evaluates the exponential integral, and
(d) uses $\exp
\left(
-\frac{
s_0
}{
2n
}
\right)
=
\exp
\left(
-\ln
\left(
2N
\right)
\right)
=
\frac{1}{2N}$. Taking expectations in
\eqref{eq:finite-threshold-chaos-maximum-reduction}
and applying
\eqref{eq:finite-threshold-chaos-second-moment-max}
gives
\begin{align}
\mathbb E_{\varepsilon}
\left[
\max_{T\in\mathcal T}
\left|
Q_T
\left(
\varepsilon
\right)
\right|
\right]
&\stackrel{(a)}{\le}
\frac{1}{2n}
\mathbb E_{\varepsilon}
\left[
Z
\right]
+
\frac{1}{2}
 \stackrel{(b)}{\le}
\ln
\left(
2N
\right)
+
\frac{3}{2}
 \stackrel{(c)}{\le}
C
\ln
\left(
1+N
\right).
\end{align}
Here,
(a) applies
\eqref{eq:finite-threshold-chaos-maximum-reduction},
(b) substitutes
\eqref{eq:finite-threshold-chaos-second-moment-max},
and
(c) uses
$N\ge1$
together with $2N
\le
\left(
1+N
\right)^2$, $\frac{3}{2}
\le
\frac{3}{2\ln 2}
\ln
\left(
1+N
\right)$, and enlarges the universal numerical constant.
This proves
\eqref{eq:finite-threshold-chaos-bound}
and completes the proof.
\end{proof}

\begin{lemma}[Brownian chaos controlled by signed threshold traces]
\label{lem:threshold-traces-control-chaos}

\begingroup
\emergencystretch=1em
Let
$n\ge1$,
let
$\b x_1,\ldots,\b x_n\in\mX$
be fixed sample points,
and let
$\mathcal A$
be a nonempty class of real-valued support functions on
$\mX$.
Assume that
$B_{\b x}\left(\mathcal A\right)<\infty$,
where
$k_a$
and
$B_{\b x}\left(\mathcal A\right)$
are defined in
\eqref{eq:brownian-pullback-rademacher-kernel}
and
\eqref{eq:brownian-pullback-diagonal-envelope},
respectively. Define the signed threshold trace family by
\par\endgroup
\begin{align}
\mathcal T_{n,\b x}^{\pm}
\left(
\mathcal A
\right)
:=
\left\{
\left\{
i\in\left[n\right]
:
\sigma
a\left(
\b x_i
\right)
\ge
\tau
\right\}
:
a\in\mathcal A,
\;
\sigma\in
\left\{
-1,1
\right\},
\;
\tau\in\mathbb R
\right\}.
\label{eq:signed-threshold-trace-family}
\end{align}
Then
$\mathcal T_{n,\b x}^{\pm}\left(\mathcal A\right)$
is a nonempty finite collection of subsets of
$\left[n\right]$,
and there exists a universal numerical constant
$C>0$
such that
\begin{align}
\widehat{\mathfrak C}_{n,\b x}^{(B)}
\left(
\mathcal A
\right)
\le
C
B_{\b x}
\left(
\mathcal A
\right)
\ln
\left(
1+
\left|
\mathcal T_{n,\b x}^{\pm}
\left(
\mathcal A
\right)
\right|
\right).
\label{eq:threshold-traces-brownian-chaos-bound}
\end{align}
Here,
$\widehat{\mathfrak C}_{n,\b x}^{(B)}$
is the Brownian quadratic-chaos complexity defined in
\eqref{eq:empirical-brownian-chaos-complexity}. Consequently, if
$E\ge0$
satisfies
\begin{align}
\ln
\left|
\mathcal T_{n,\b x}^{\pm}
\left(
\mathcal A
\right)
\right|
\le
E,
\label{eq:threshold-traces-entropy-assumption}
\end{align}
then
\begin{align}
\widehat{\mathfrak C}_{n,\b x}^{(B)}
\left(
\mathcal A
\right)
\le
C
B_{\b x}
\left(
\mathcal A
\right)
\left(
1+E
\right).
\label{eq:threshold-traces-brownian-chaos-entropy-bound}
\end{align}

\end{lemma}

\begin{proof}
Set
\begin{align}
B
:=
B_{\b x}
\left(
\mathcal A
\right).
\label{eq:threshold-traces-proof-envelope-abbreviation}
\end{align}
We first consider the case
$B=0$.
By
\eqref{eq:brownian-pullback-diagonal-envelope}
and the diagonal identity $k_a
\left(
\b x_i,
\b x_i
\right)
=
\left|
a\left(
\b x_i
\right)
\right|$, one has
\begin{align}
a\left(
\b x_i
\right)
=
0,
\qquad
a\in\mathcal A,
\quad
i\in\left[n\right].
\label{eq:threshold-traces-proof-zero-support-values}
\end{align}
Consequently, for every
$a\in\mathcal A$
and
$i,j\in\left[n\right]$,
\begin{align}
k_a
\left(
\b x_i,
\b x_j
\right)
&\stackrel{(a)}{=}
\kB
\left(
0,0
\right)
\stackrel{(b)}{=}
0.
\end{align}
Here,
(a) uses
\eqref{eq:threshold-traces-proof-zero-support-values},
and
(b) follows from the definition of the Brownian kernel.
Therefore,
$\widehat{\mathfrak C}_{n,\b x}^{(B)}
\left(
\mathcal A
\right)
=
0$,
and
\eqref{eq:threshold-traces-brownian-chaos-bound}
holds trivially.
We may therefore assume throughout the remainder of the proof that
$B>0$. We first establish the threshold representation of the full Brownian kernel on
$\left[-B,B\right]$.
Fix
$s,t\in\left[-B,B\right]$.
The definition of the Brownian kernel gives
\begin{align}
\kB
\left(
s,t
\right)
&\stackrel{(a)}{=}
\mathbf 1_{\left\{s\ge0\right\}}
\mathbf 1_{\left\{t\ge0\right\}}
\min
\left\{
s,t
\right\}
+
\mathbf 1_{\left\{s\le0\right\}}
\mathbf 1_{\left\{t\le0\right\}}
\min
\left\{
-s,-t
\right\}.
\label{eq:threshold-traces-brownian-sign-decomposition}
\end{align}
Here,
(a) follows by considering the three possible sign configurations.
If
$s,t\ge0$,
then
\begin{align}
\kB
\left(
s,t
\right)
=
\frac{
s+t-\left|s-t\right|
}{
2
}
=
\min
\left\{
s,t
\right\}.
\end{align}
If
$s,t\le0$,
then
$\kB
\left(
s,t
\right)
=
\min
\left\{
-s,-t
\right\}$.
If
$s$
and
$t$
have opposite signs, then
$\left|s-t\right|
=
\left|s\right|
+
\left|t\right|$,
and hence
$\kB
\left(
s,t
\right)
=
0$.
For the nonnegative branch,
\begin{align}
&
\int_0^B
\mathbf 1_{\left\{s\ge r\right\}}
\mathbf 1_{\left\{t\ge r\right\}}
\,\d r
\stackrel{(a)}{=}
\mathbf 1_{\left\{s\ge0\right\}}
\mathbf 1_{\left\{t\ge0\right\}}
\int_0^{\min\left\{s,t\right\}}
1
\,\d r
\stackrel{(b)}{=}
\mathbf 1_{\left\{s\ge0\right\}}
\mathbf 1_{\left\{t\ge0\right\}}
\min
\left\{
s,t
\right\}.
\label{eq:threshold-traces-positive-branch}
\end{align}
Here,
(a) uses the fact that both threshold inequalities hold exactly for
$r\in\left[0,\min\left\{s,t\right\}\right]$
when
$s,t\ge0$,
and for no positive-measure set of
$r$
otherwise.
Equality
(b)
evaluates the integral. Similarly,
\begin{align}
\int_0^B
\mathbf 1_{\left\{s\le-r\right\}}
\mathbf 1_{\left\{t\le-r\right\}}
\,\d r
&\stackrel{(a)}{=}
\mathbf 1_{\left\{s\le0\right\}}
\mathbf 1_{\left\{t\le0\right\}}
\int_0^{\min\left\{-s,-t\right\}}
1
\,\d r
\notag\\
&\stackrel{(b)}{=}
\mathbf 1_{\left\{s\le0\right\}}
\mathbf 1_{\left\{t\le0\right\}}
\min
\left\{
-s,-t
\right\}.
\label{eq:threshold-traces-negative-branch}
\end{align}
Here,
(a) rewrites
$s\le-r$
and
$t\le-r$
as
$r\le-s$
and
$r\le-t$,
while
(b)
evaluates the resulting integral. Combining
\eqref{eq:threshold-traces-brownian-sign-decomposition},
\eqref{eq:threshold-traces-positive-branch},
and
\eqref{eq:threshold-traces-negative-branch}
gives
\begin{align}
\kB
\left(
s,t
\right)
&=
\int_0^B
\mathbf 1_{\left\{s\ge r\right\}}
\mathbf 1_{\left\{t\ge r\right\}}
\,\d r
\qquad
+
\int_0^B
\mathbf 1_{\left\{s\le-r\right\}}
\mathbf 1_{\left\{t\le-r\right\}}
\,\d r.
\label{eq:threshold-traces-full-brownian-representation}
\end{align}
Fix
$a\in\mathcal A$
and
$r\in\left[0,B\right]$.
Define
\begin{align}
T_{a,r}^{+}
&:=
\left\{
i\in\left[n\right]
:
a\left(
\b x_i
\right)
\ge
r
\right\},
\label{eq:threshold-traces-positive-trace}
\\
T_{a,r}^{-}
&:=
\left\{
i\in\left[n\right]
:
-a\left(
\b x_i
\right)
\ge
r
\right\}
=
\left\{
i\in\left[n\right]
:
a\left(
\b x_i
\right)
\le
-r
\right\}.
\label{eq:threshold-traces-negative-trace}
\end{align}
By
\eqref{eq:signed-threshold-trace-family},
\begin{align}
T_{a,r}^{+}
&\in
\mathcal T_{n,\b x}^{\pm}
\left(
\mathcal A
\right),
&
T_{a,r}^{-}
&\in
\mathcal T_{n,\b x}^{\pm}
\left(
\mathcal A
\right).
\label{eq:threshold-traces-membership}
\end{align}
By
\eqref{eq:brownian-pullback-diagonal-envelope}
and the diagonal identity $k_a
\left(
\b x,
\b x
\right)
=
\left|
a\left(
\b x
\right)
\right|$, the values
$a\left(\b x_i\right)$
and
$a\left(\b x_j\right)$
belong to
$\left[-B,B\right]$.
Applying
\eqref{eq:threshold-traces-full-brownian-representation}
with
$s=a\left(\b x_i\right)$
and
$t=a\left(\b x_j\right)$
gives
\begin{align}
k_a
\left(
\b x_i,
\b x_j
\right)
&\stackrel{(a)}{=}
\int_0^B
\mathbf 1_{\left\{i\in T_{a,r}^{+}\right\}}
\mathbf 1_{\left\{j\in T_{a,r}^{+}\right\}}
\,\d r
\qquad
+
\int_0^B
\mathbf 1_{\left\{i\in T_{a,r}^{-}\right\}}
\mathbf 1_{\left\{j\in T_{a,r}^{-}\right\}}
\,\d r.
\label{eq:threshold-traces-pullback-representation}
\end{align}
Here,
(a) applies
\eqref{eq:threshold-traces-full-brownian-representation}
and then uses
\eqref{eq:threshold-traces-positive-trace}
and
\eqref{eq:threshold-traces-negative-trace}. For every
$T\subseteq\left[n\right]$,
define
\begin{align}
Q_T
\left(
\varepsilon
\right)
:=
\frac{1}{n}
\sum_{1\le i<j\le n}
\varepsilon_i
\varepsilon_j
\mathbf 1_{\left\{i\in T\right\}}
\mathbf 1_{\left\{j\in T\right\}}.
\label{eq:threshold-traces-indicator-chaos}
\end{align}
Substituting
\eqref{eq:threshold-traces-pullback-representation}
into the Brownian chaos sum gives
\begin{align}
&
\frac{1}{n}
\sum_{1\le i<j\le n}
\varepsilon_i
\varepsilon_j
k_a
\left(
\b x_i,
\b x_j
\right)
\stackrel{(a)}{=}
\int_0^B
Q_{T_{a,r}^{+}}
\left(
\varepsilon
\right)
\,\d r
+
\int_0^B
Q_{T_{a,r}^{-}}
\left(
\varepsilon
\right)
\,\d r.
\label{eq:threshold-traces-chaos-integral-decomposition}
\end{align}
Here,
(a) interchanges the finite sum over
$i<j$
with the two integrals and applies
\eqref{eq:threshold-traces-indicator-chaos}. Taking absolute values in
\eqref{eq:threshold-traces-chaos-integral-decomposition}
and applying the triangle inequality gives
\begin{align}
\left|
\frac{1}{n}
\sum_{1\le i<j\le n}
\varepsilon_i
\varepsilon_j
k_a
\left(
\b x_i,
\b x_j
\right)
\right|
&\stackrel{(a)}{\le}
\int_0^B
\left|
Q_{T_{a,r}^{+}}
\left(
\varepsilon
\right)
\right|
\,\d r
+
\int_0^B
\left|
Q_{T_{a,r}^{-}}
\left(
\varepsilon
\right)
\right|
\,\d r
\notag\\
&\stackrel{(b)}{\le}
2B
\max_{
T\in
\mathcal T_{n,\b x}^{\pm}
\left(
\mathcal A
\right)
}
\left|
Q_T
\left(
\varepsilon
\right)
\right|.
\label{eq:threshold-traces-chaos-pointwise-bound}
\end{align}
Here,
(a) applies the triangle inequality for integrals,
while
(b) uses
\eqref{eq:threshold-traces-membership}
and the fact that both integration intervals have length
$B$. The right-hand side of
\eqref{eq:threshold-traces-chaos-pointwise-bound}
does not depend on
$a$.
Therefore,
\begin{align}
&
\sup_{a\in\mathcal A}
\left|
\frac{1}{n}
\sum_{1\le i<j\le n}
\varepsilon_i
\varepsilon_j
k_a
\left(
\b x_i,
\b x_j
\right)
\right|
\le
2B
\max_{
T\in
\mathcal T_{n,\b x}^{\pm}
\left(
\mathcal A
\right)
}
\left|
Q_T
\left(
\varepsilon
\right)
\right|.
\label{eq:threshold-traces-chaos-supremum-bound}
\end{align}
Since every member of
$\mathcal T_{n,\b x}^{\pm}\left(\mathcal A\right)$
is a subset of
$\left[n\right]$, $1
\le
\left|
\mathcal T_{n,\b x}^{\pm}
\left(
\mathcal A
\right)
\right|
\le
2^n$. Thus, the trace family is finite and nonempty.
Taking expectations in
\eqref{eq:threshold-traces-chaos-supremum-bound}
and applying
\Cref{lem:self-contained-threshold-chaos}
gives
\begin{align}
\widehat{\mathfrak C}_{n,\b x}^{(B)}
\left(
\mathcal A
\right)
&\stackrel{(a)}{\le}
2B
\mathbb E_{\varepsilon}
\left[
\max_{
T\in
\mathcal T_{n,\b x}^{\pm}
\left(
\mathcal A
\right)
}
\left|
Q_T
\left(
\varepsilon
\right)
\right|
\right]
\stackrel{(b)}{\le}
2CB
\ln
\left(
1+
\left|
\mathcal T_{n,\b x}^{\pm}
\left(
\mathcal A
\right)
\right|
\right)
\stackrel{(c)}{\le}
CB
\ln
\left(
1+
\left|
\mathcal T_{n,\b x}^{\pm}
\left(
\mathcal A
\right)
\right|
\right).
\end{align}
Here,
(a) applies
\eqref{eq:threshold-traces-chaos-supremum-bound}
and the definition
\eqref{eq:empirical-brownian-chaos-complexity},
(b) applies
\Cref{lem:self-contained-threshold-chaos},
and
(c) enlarges the universal numerical constant to absorb the factor
$2$.
Substituting
\eqref{eq:threshold-traces-proof-envelope-abbreviation}
proves
\eqref{eq:threshold-traces-brownian-chaos-bound}. Finally, suppose that
\eqref{eq:threshold-traces-entropy-assumption}
holds.
Set $N_{\mathcal T}
:=
\left|
\mathcal T_{n,\b x}^{\pm}
\left(
\mathcal A
\right)
\right|$. Since
$N_{\mathcal T}\ge1$,
\begin{align}
\ln
\left(
1+N_{\mathcal T}
\right)
&\stackrel{(a)}{\le}
\ln
\left(
2N_{\mathcal T}
\right)
\stackrel{(b)}{=}
\ln 2
+
\ln
N_{\mathcal T}
\stackrel{(c)}{\le}
1+E.
\label{eq:threshold-traces-log-cardinality-consequence}
\end{align}
Here,
(a) uses
$1+N_{\mathcal T}\le2N_{\mathcal T}$,
(b) uses the logarithm of a product,
and
(c) uses
$\ln2\le1$
together with
\eqref{eq:threshold-traces-entropy-assumption}.
Combining
\eqref{eq:threshold-traces-brownian-chaos-bound}
and
\eqref{eq:threshold-traces-log-cardinality-consequence}
proves
\eqref{eq:threshold-traces-brownian-chaos-entropy-bound}.
This completes the proof.

\end{proof}

\begin{lemma}\textnormal{(VC dimension of signed finite-architecture threshold classes)}\ 
\label{lem:finite-bkl-vc-dimension}
Let
$L,m,G\in\mathbb N$
satisfy
$L\ge2$,
$m\ge1$,
and
$G\ge2$.
Let
$P_{L-1,m,G}$
be the parameter-count bound specified in
\Cref{subsec:finite-parametric-bkl}
and established in
\Cref{lem:finite-bkl-parameter-count},
part~\ref{lem:finite-architecture-parameter-count},
and set
$
P:=P_{L-1,m,G}.
$ Let
$\Theta_{L-1,m,G}\subseteq\mathbb R^P$
denote the admissible parameter set of
$\mathcal A_{L-1}^{m,G}$.
If a particular representation uses fewer than
$P$
coordinates, identify it with an element of
$\mathbb R^P$
by padding the remaining coordinates with zeros.
For
$\boldsymbol\theta\in\Theta_{L-1,m,G}$,
let
$a_{\boldsymbol\theta}$
denote the associated lower-support function. Define
\begin{align}
\mathcal H_{L-1,m,G}^{\pm}
:=
\left\{
h_{\boldsymbol\theta,\tau,\sigma}
:
\boldsymbol\theta\in\Theta_{L-1,m,G},
\;
\tau\in\mathbb R,
\;
\sigma\in\left\{
-1,1
\right\}
\right\},
\label{eq:finite-architecture-signed-threshold-class}
\end{align}
where
\begin{align}
h_{\boldsymbol\theta,\tau,\sigma}
\left(
\b x
\right)
:=
\mathbf 1_{\left\{
\sigma
a_{\boldsymbol\theta}\left(
\b x
\right)
\ge
\tau
\right\}},
\qquad
\b x\in\mX.
\label{eq:finite-architecture-signed-threshold-function}
\end{align}
Then there exists
$C_L>0$,
depending only on
$L$,
such that
\begin{align}
\operatorname{VC}
\left(
\mathcal H_{L-1,m,G}^{\pm}
\right)
\le
C_L
\left(
P+1
\right)
\ln\left(
e\left(
P+1
\right)
\right).
\label{eq:finite-architecture-signed-threshold-vc-bound}
\end{align}
Equivalently,
\begin{align}
\operatorname{VC}
\left(
\mathcal H_{L-1,m,G}^{\pm}
\right)
\le
C_L
\left(
P_{L-1,m,G}+1
\right)
\ln\left(
e\left(
P_{L-1,m,G}+1
\right)
\right).
\label{eq:finite-architecture-signed-threshold-vc-bound-expanded}
\end{align}

\end{lemma}

\begin{proof}
Set
$W:=P+1$.
The additional coordinate in
$W$
is the variable threshold
$\tau$.
For each fixed
$\sigma\in\left\{-1,1\right\}$,
the corresponding classifiers are therefore parameterized by
$\boldsymbol\eta:=\left(\boldsymbol\theta,\tau\right)\in\mathbb R^W$. We divide the proof into four steps.

\par\noindent\textbf{Step 1: enlargement of the admissible parameter set.} The admissible architecture parameters satisfy the direction,
profile-energy, mixing, and readout constraints introduced in
\Cref{subsec:finite-parametric-bkl}.
For an upper bound on the number of possible labelings, we may remove
all of these constraints. More precisely, let
$\widetilde{\mathcal A}_{L-1}^{m,G}$
denote the class obtained by allowing every coordinate of
$\boldsymbol\theta$
to vary freely in
$\mathbb R^P$,
while retaining the same computational graph, the same interpolation
knots, and the same constant profile extensions.
Then
$\mathcal A_{L-1}^{m,G}\subseteq\widetilde{\mathcal A}_{L-1}^{m,G}$.
Consequently,
$\mathcal H_{L-1,m,G}^{\pm}\subseteq\widetilde{\mathcal H}_{L-1,m,G}^{\pm}$,
where
$\widetilde{\mathcal H}_{L-1,m,G}^{\pm}$
is defined by
\eqref{eq:finite-architecture-signed-threshold-class}
with
$\boldsymbol\theta\in\mathbb R^P$.
By the monotonicity of the VC dimension under class inclusion,
\begin{align}
\operatorname{VC}
\left(
\mathcal H_{L-1,m,G}^{\pm}
\right)
&\le
\operatorname{VC}
\left(
\widetilde{\mathcal H}_{L-1,m,G}^{\pm}
\right).
\label{eq:finite-threshold-vc-enlargement-monotonicity}
\end{align}
We will therefore count the labelings generated by the enlarged class.

\par\noindent\textbf{Step 2: polynomial sign-pattern estimate.} We use the following standard sign-pattern estimate.
Let
$p_1,\ldots,p_M$
be real polynomials in
$W$
real variables, each having degree at most
$D$.
Then the number of realizable sign vectors
\begin{align}
\left(
\operatorname{sgn}
\left(
p_1\left(
\boldsymbol\eta
\right)
\right),
\ldots,
\operatorname{sgn}
\left(
p_M\left(
\boldsymbol\eta
\right)
\right)
\right)
\in
\left\{
-1,0,1
\right\}^M,
\qquad
\boldsymbol\eta\in\mathbb R^W,
\end{align}
is bounded by
\begin{align}
\mathfrak S
\left(
M,D,W
\right)
\le
\left(
C_0
D
\left(
M+W
\right)
\right)^W,
\label{eq:finite-threshold-vc-polynomial-sign-bound}
\end{align}
where
$C_0>0$
is a universal numerical constant.
This is the standard polynomial sign-pattern estimate underlying
semialgebraic VC-dimension bounds; see, for example,
\citet{goldberg1995complexity}. For a binary class
$\mathcal H$,
define its growth function by
\begin{align}
\Pi_{\mathcal H}
\left(
N
\right)
:=
\sup_{
\left(
\b x_1,\ldots,\b x_N
\right)
\in
\mX^N
}
\left|
\left\{
\left(
h\left(
\b x_1
\right),
\ldots,
h\left(
\b x_N
\right)
\right)
:
h\in\mathcal H
\right\}
\right|.
\end{align}
We will prove that there exists a constant
$C_{L,1}>0$,
depending only on
$L$,
such that
\begin{align}
\Pi_{
\widetilde{\mathcal H}_{L-1,m,G}^{\pm}
}
\left(
N
\right)
\le
2
\left(
C_{L,1}
N
P
\right)^{
\left(
L-1
\right)
W
}.
\label{eq:finite-threshold-vc-growth-bound}
\end{align}

\par\noindent\textbf{Step 3: layerwise semialgebraic decomposition.}\\
Fix
$N\ge1$
and fixed points
$\b x_1,\ldots,\b x_N\in\mX$.
Set
$K:=L-2$.
Thus,
$K=0$
when
$L=2$,
whereas
$K\ge1$
when
$L\ge3$. We first identify the polynomial degree generated by the recursive
architecture.
When
$L=2$,
one has
\begin{align}
a_{\boldsymbol\theta}
\left(
\b x_i
\right)
=
\boldsymbol\omega^\top\b x_i,
\label{eq:finite-threshold-vc-depth-two-output}
\end{align}
which is a polynomial of degree one in
$\boldsymbol\theta$. Suppose now that
$L\ge3$.
For each profile
$q_j^{(r)}$,
write its nodal parameters as
$c_{j,0}^{(r)}, \ldots, c_{j,G}^{(r)}$.
Using the constant extensions from
\eqref{eq:finite-profile-left-extension}
and
\eqref{eq:finite-profile-right-extension},
the profile has the piecewise representation
\begin{align}
q_j^{(r)}
\left(
t
\right)
&=
\begin{cases}
c_{j,0}^{(r)},
&
t\le t_0,
\\[2mm]
\displaystyle
\frac{
t_{\ell+1}-t
}{
t_{\ell+1}-t_\ell
}
c_{j,\ell}^{(r)}
+
\frac{
t-t_\ell
}{
t_{\ell+1}-t_\ell
}
c_{j,\ell+1}^{(r)},
&
t\in
\left[
t_\ell,t_{\ell+1}
\right],
\quad
\ell\in
\left\{
0,\ldots,G-1
\right\},
\\[3mm]
c_{j,G}^{(r)},
&
t\ge t_G.
\end{cases}
\label{eq:finite-threshold-vc-piecewise-profile}
\end{align}
At a grid knot, the two adjacent affine formulas have the same value.
Therefore, any fixed convention for assigning a knot to one of its
adjacent intervals yields the same profile value. For
$i\in\left[N\right]$,
$j\in\left[m\right]$,
and
$r\in\left[K\right]$,
define the input to the
$r$-th profile by
\begin{align}
v_{i,j}^{(1)}
&:=
z_j^{(0)}
\left(
\b x_i
\right),
\label{eq:finite-threshold-vc-first-profile-input}
\\
v_{i,j}^{(r)}
&:=
s_j^{(r)}
\left(
\b x_i
\right),
\qquad
r\in
\left\{
2,\ldots,K
\right\}.
\label{eq:finite-threshold-vc-later-profile-input}
\end{align}
The active region of
$q_j^{(r)}$
is determined by the signs of the
$G+1$
quantities
\begin{align}
p_{i,j,\ell}^{(r)}
\left(
\boldsymbol\eta
\right)
:=
v_{i,j}^{(r)}
-
t_\ell,
\qquad
\ell\in
\left\{
0,\ldots,G
\right\}.
\label{eq:finite-threshold-vc-region-polynomials}
\end{align} 
We now prove by induction that, after fixing all profile regions up to
level
$r$,
every value
$z_j^{(r)}\left(\b x_i\right)$
is a polynomial in
$\boldsymbol\eta$
of degree at most
$2r$. At the first profile level,
\begin{align}
v_{i,j}^{(1)}
&\stackrel{(a)}{=}
z_j^{(0)}
\left(
\b x_i
\right)
\stackrel{(b)}{=}
\boldsymbol\omega_j^\top\b x_i.
\label{eq:finite-threshold-vc-first-input-polynomial}
\end{align}
Here,
(a) applies
\eqref{eq:finite-threshold-vc-first-profile-input},
while
(b) applies
\eqref{eq:finite-lower-support-linear-layer}.
Since the sample point
$\b x_i$
is fixed,
\eqref{eq:finite-threshold-vc-first-input-polynomial}
is a polynomial of degree one in the architecture parameters. Once the active region is fixed,
\eqref{eq:finite-threshold-vc-piecewise-profile}
shows that
\begin{align}
z_j^{(1)}
\left(
\b x_i
\right)
&=
q_j^{(1)}
\left(
v_{i,j}^{(1)}
\right)
\end{align}
is a polynomial of degree at most two.
Indeed, on an interior interval,
\eqref{eq:finite-threshold-vc-piecewise-profile}
can be rewritten as
\begin{align}
q_j^{(1)}
\left(
v_{i,j}^{(1)}
\right)
&=
\frac{
t_{\ell+1}
c_{j,\ell}^{(1)}
-
t_\ell
c_{j,\ell+1}^{(1)}
}{
t_{\ell+1}-t_\ell
}
+
\frac{
c_{j,\ell+1}^{(1)}
-
c_{j,\ell}^{(1)}
}{
t_{\ell+1}-t_\ell
}
v_{i,j}^{(1)}.
\label{eq:finite-threshold-vc-first-profile-degree}
\end{align}
The first term in
\eqref{eq:finite-threshold-vc-first-profile-degree}
has degree one in the nodal parameters.
The second term is the product of a degree-one nodal expression and
the degree-one polynomial
$v_{i,j}^{(1)}$,
and therefore has degree at most two.
On either exterior region, the profile value is a single nodal
parameter and has degree one.
This proves the degree bound at
$r=1$. Fix
$r\in\left\{2,\ldots,K\right\}$
and assume that, after fixing the profile regions up to level
$r-1$,
every
$z_k^{(r-1)}\left(\b x_i\right)$,
with
$i\in\left[N\right]$
and
$k\in\left[m\right]$,
is a polynomial in
$\boldsymbol\eta$
of degree at most
$2\left(r-1\right)$.
Then
\begin{align}
v_{i,j}^{(r)}
&\stackrel{(a)}{=}
s_j^{(r)}
\left(
\b x_i
\right)
 \stackrel{(b)}{=}
\sum_{k=1}^{m}
W_{jk}^{(r)}
z_k^{(r-1)}
\left(
\b x_i
\right).
\label{eq:finite-threshold-vc-inductive-preactivation}
\end{align}
Here,
(a) applies
\eqref{eq:finite-threshold-vc-later-profile-input},
and
(b) applies
\eqref{eq:finite-lower-support-preactivation}.
Each mixing coefficient
$W_{jk}^{(r)}$
has polynomial degree one in
$\boldsymbol\eta$.
The induction hypothesis therefore implies that every summand in
\eqref{eq:finite-threshold-vc-inductive-preactivation}
has degree at most
$1+2\left(r-1\right)=2r-1$.
Consequently,
\begin{align}
\deg
\left(
v_{i,j}^{(r)}
\right)
\le
2r-1.
\label{eq:finite-threshold-vc-preactivation-degree}
\end{align}
Once the active region at level
$r$
is fixed,
the same calculation as in
\eqref{eq:finite-threshold-vc-first-profile-degree}
gives
\begin{align}
\deg
\left(
z_j^{(r)}
\left(
\b x_i
\right)
\right)
&\stackrel{(a)}{\le}
1
+
\deg
\left(
v_{i,j}^{(r)}
\right)
\stackrel{(b)}{\le}
2r.
\end{align}
Here,
(a) follows because an interior profile branch multiplies one nodal
coefficient by the profile input, while
(b) applies
\eqref{eq:finite-threshold-vc-preactivation-degree}.
This completes the induction. We next count the possible profile-region states.
Let
\begin{align}
\mathcal P_0
:=
\left\{
\mathbb R^W
\right\}.
\label{eq:finite-threshold-vc-initial-partition}
\end{align}
For
$r\in\left[K\right]$,
let
$\mathcal P_r$
denote the refinement of
$\mathcal P_{r-1}$
obtained by fixing the signs of all polynomials
\eqref{eq:finite-threshold-vc-region-polynomials}
at the
$r$-th profile level.

For every fixed member of
$\mathcal P_{r-1}$,
there is one region polynomial for every sample point, every unit, and
every grid knot. Hence, the number of region polynomials at level
$r$
is
\begin{align}
M_r
&=
N
m
\left(
G+1
\right).
\label{eq:finite-threshold-vc-number-region-polynomials}
\end{align}
Their degrees are bounded by
$2r-1\le2K+1$.
Define
$D_L:=2K+1=2L-3$. When
$L\ge3$,
the parameter count
\eqref{eq:finite-bkl-explicit-parameter-count}
contains the profile contribution
$Km\left(G+1\right)$.
Since
$K\ge1$,
one obtains
\begin{align}
m
\left(
G+1
\right)
\le
P.
\label{eq:finite-threshold-vc-profile-count-bounded-by-p}
\end{align}
Moreover, since
$P\ge1$,
\begin{align}
W
=
P+1
\le
2P.
\label{eq:finite-threshold-vc-w-bounded-by-p}
\end{align}
Therefore,
\begin{align}
M_r+W
&\stackrel{(a)}{\le}
NP
+
2P
 \stackrel{(b)}{\le}
3NP.
\label{eq:finite-threshold-vc-region-polynomial-total}
\end{align}
Here,
(a) uses
\eqref{eq:finite-threshold-vc-number-region-polynomials},
\eqref{eq:finite-threshold-vc-profile-count-bounded-by-p},
and
\eqref{eq:finite-threshold-vc-w-bounded-by-p},
while
(b) uses
$N\ge1$. Applying
\eqref{eq:finite-threshold-vc-polynomial-sign-bound}
inside each member of
$\mathcal P_{r-1}$
gives
\begin{align}
\left|
\mathcal P_r
\right|
&\stackrel{(a)}{\le}
\left|
\mathcal P_{r-1}
\right|
\left(
C_0
D_L
\left(
M_r+W
\right)
\right)^W
 \stackrel{(b)}{\le}
\left|
\mathcal P_{r-1}
\right|
\left(
3
C_0
D_L
NP
\right)^W.
\label{eq:finite-threshold-vc-partition-recursion}
\end{align}
Here,
(a) applies the polynomial sign-pattern estimate to the region
polynomials on each preceding state, and
(b) applies
\eqref{eq:finite-threshold-vc-region-polynomial-total}. Set
\begin{align}
C_{L,0}
:=
3
C_0
D_L.
\label{eq:finite-threshold-vc-depth-constant}
\end{align}
Iterating
\eqref{eq:finite-threshold-vc-partition-recursion}
from
$r=1$
to
$r=K$
and using
\eqref{eq:finite-threshold-vc-initial-partition}
gives
\begin{align}
\left|
\mathcal P_K
\right|
&\le
\left(
C_{L,0}
NP
\right)^{KW}.
\label{eq:finite-threshold-vc-final-state-count}
\end{align}
When
$K=0$,
the right-hand side equals one and
\eqref{eq:finite-threshold-vc-final-state-count}
reduces to
$\left|\mathcal P_0\right|=1$. On every member of
$\mathcal P_K$,
the architecture output is a polynomial in
$\boldsymbol\eta$.
For
$L\ge3$,
\eqref{eq:finite-lower-support-output}
gives
\begin{align}
a_{\boldsymbol\theta}
\left(
\b x_i
\right)
&=
\sum_{j=1}^{m}
\beta_j
z_j^{(K)}
\left(
\b x_i
\right).
\label{eq:finite-threshold-vc-output-polynomial}
\end{align}
Since the readout coefficients have degree one and the final hidden
coordinates have degree at most
$2K$,
\begin{align}
\deg
\left(
a_{\boldsymbol\theta}
\left(
\b x_i
\right)
\right)
&\le
1+2K
=
D_L.
\label{eq:finite-threshold-vc-output-degree}
\end{align}
For
$L=2$,
the same estimate follows from
\eqref{eq:finite-threshold-vc-depth-two-output}
and
$D_L=1$. Fix
$\sigma\in\left\{-1,1\right\}$.
On every member of
$\mathcal P_K$,
the labels of the
$N$
sample points are determined by the signs of the
$N$
polynomials
\begin{align}
r_{i,\sigma}
\left(
\boldsymbol\eta
\right)
:=
\sigma
a_{\boldsymbol\theta}
\left(
\b x_i
\right)
-
\tau,
\qquad
i\in\left[N\right].
\end{align}
By
\eqref{eq:finite-threshold-vc-output-degree},
these polynomials have degree at most
$D_L$.
Furthermore,
\begin{align}
N+W
&\stackrel{(a)}{\le}
N+2P
\stackrel{(b)}{\le}
3NP.
\label{eq:finite-threshold-vc-final-polynomial-total}
\end{align}
Here,
(a) applies
\eqref{eq:finite-threshold-vc-w-bounded-by-p},
and
(b) uses
$N\ge1$
and
$P\ge1$.
Thus,
\eqref{eq:finite-threshold-vc-polynomial-sign-bound}
implies that the number of label vectors generated on each member of
$\mathcal P_K$
for one fixed sign
$\sigma$
is at most
\begin{align}
\left(
C_0
D_L
\left(
N+W
\right)
\right)^W
&\le
\left(
C_{L,0}
NP
\right)^W,
\label{eq:finite-threshold-vc-labelings-per-state}
\end{align}
where the inequality follows from
\eqref{eq:finite-threshold-vc-final-polynomial-total}
and
\eqref{eq:finite-threshold-vc-depth-constant}. There are two possible choices of
$\sigma$.
Combining
\eqref{eq:finite-threshold-vc-final-state-count}
and
\eqref{eq:finite-threshold-vc-labelings-per-state}
therefore yields
\begin{align}
\Pi_{
\widetilde{\mathcal H}_{L-1,m,G}^{\pm}
}
\left(
N
\right)
&\stackrel{(a)}{\le}
2
\left|
\mathcal P_K
\right|
\left(
C_{L,0}
NP
\right)^W
 \stackrel{(b)}{\le}
2
\left(
C_{L,0}
NP
\right)^{
\left(
K+1
\right)
W
}
 \stackrel{(c)}{=}
2
\left(
C_{L,0}
NP
\right)^{
\left(
L-1
\right)
W
}.
\label{eq:finite-threshold-vc-growth-bound-proof}
\end{align}
Here,
(a) sums the bounds corresponding to the two signs,
(b) substitutes
\eqref{eq:finite-threshold-vc-final-state-count},
and
(c) uses
$K=L-2$.
This proves
\eqref{eq:finite-threshold-vc-growth-bound}
after relabeling the depth-dependent constant.

\par\noindent\textbf{Step 4: inversion of the growth bound.} Suppose that
$\widetilde{\mathcal H}_{L-1,m,G}^{\pm}$
shatters
$N$
points.
Then
\begin{align}
2^N
&\stackrel{(a)}{\le}
\Pi_{
\widetilde{\mathcal H}_{L-1,m,G}^{\pm}
}
\left(
N
\right)
\stackrel{(b)}{\le}
2
\left(
C_{L,0}
NP
\right)^{
\left(
L-1
\right)
W
}.
\end{align}
Here,
(a) follows from the definition of shattering, and
(b) applies
\eqref{eq:finite-threshold-vc-growth-bound-proof}.
Taking natural logarithms gives
\begin{align}
N
\ln 2
\le
\ln 2
+
\left(
L-1
\right)
W
\ln
\left(
C_{L,0}
NP
\right).
\label{eq:finite-threshold-vc-logarithmic-inequality}
\end{align}
Set
$c_L:=L-1$ and $r:=N/W$.
Since
$P\le W$
and
$N=rW$,
\begin{align}
\ln
\left(
C_{L,0}
NP
\right)
&\stackrel{(a)}{=}
\ln
\left(
C_{L,0}
rWP
\right)
 \stackrel{(b)}{\le}
\ln
C_{L,0}
+
\ln r
+
2
\ln W.
\label{eq:finite-threshold-vc-log-product-bound}
\end{align}
Here,
(a) substitutes
$N=rW$,
and
(b) uses
$P\le W$.
Dividing
\eqref{eq:finite-threshold-vc-logarithmic-inequality}
by
$W$
and applying
\eqref{eq:finite-threshold-vc-log-product-bound}
gives
\begin{align}
r
\ln 2
&\stackrel{(a)}{\le}
\frac{
\ln 2
}{
W
}
+
c_L
\left(
\ln C_{L,0}
+
\ln r
+
2
\ln W
\right)
 \stackrel{(b)}{\le}
\ln 2
+
c_L
\left(
\ln C_{L,0}
+
\ln r
+
2
\ln W
\right).
\label{eq:finite-threshold-vc-ratio-inequality}
\end{align}
Here,
(a) performs the division and substitution, while
(b) uses
$W\ge1$.
Define
\begin{align}
\eta_L
:=
\frac{
\ln 2
}{
2c_L
}.
\label{eq:finite-threshold-vc-eta}
\end{align}
The elementary inequality
$\ln y\le y-1$ for $y>0$,
applied with
$y=\eta_Lr$,
gives
\begin{align}
\ln r
&\stackrel{(a)}{=}
\ln
\left(
\eta_Lr
\right)
-
\ln
\eta_L
 \stackrel{(b)}{\le}
\eta_Lr
-
1
-
\ln
\eta_L.
\label{eq:finite-threshold-vc-log-r-bound}
\end{align}
Here,
(a) uses the logarithm of a product, while
(b) applies
$\ln y\le y-1$.
Multiplying
\eqref{eq:finite-threshold-vc-log-r-bound}
by
$c_L$
and using
\eqref{eq:finite-threshold-vc-eta}
gives
\begin{align}
c_L
\ln r
&\le
\frac{
r\ln 2
}{
2
}
+
c_L
\left(
-1
-
\ln
\eta_L
\right).
\label{eq:finite-threshold-vc-absorbed-log-r}
\end{align}
Substituting
\eqref{eq:finite-threshold-vc-absorbed-log-r}
into
\eqref{eq:finite-threshold-vc-ratio-inequality}
and moving
$r\ln 2/2$
to the left-hand side yields
\begin{align}
\frac{
r\ln 2
}{
2
}
&\le
\ln 2
+
c_L
\left(
\ln C_{L,0}
-
1
-
\ln\eta_L
+
2\ln W
\right).
\end{align}
Since
$c_L$,
$C_{L,0}$,
and
$\eta_L$
depend only on
$L$,
there exists a constant
$C_{L,2}>0$,
depending only on
$L$,
such that
\begin{align}
r
\le
C_{L,2}
\ln
\left(
eW
\right).
\label{eq:finite-threshold-vc-ratio-bound}
\end{align}
Multiplying
\eqref{eq:finite-threshold-vc-ratio-bound}
by
$W$
and using
$N=rW$
gives
\begin{align}
N
\le
C_{L,2}
W
\ln
\left(
eW
\right).
\label{eq:finite-threshold-vc-shattered-size}
\end{align}
Since every shattered set satisfies
\eqref{eq:finite-threshold-vc-shattered-size},
\begin{align}
\operatorname{VC}
\left(
\widetilde{\mathcal H}_{L-1,m,G}^{\pm}
\right)
&\stackrel{(a)}{\le}
C_{L,2}
W
\ln
\left(
eW
\right)
 \stackrel{(b)}{=}
C_{L,2}
\left(
P+1
\right)
\ln
\left(
e
\left(
P+1
\right)
\right).
\label{eq:finite-threshold-vc-enlarged-class-bound}
\end{align}
Here,
(a) applies the definition of the VC dimension, while
(b) substitutes
$W=P+1$.
Finally,
\eqref{eq:finite-threshold-vc-enlargement-monotonicity}
and
\eqref{eq:finite-threshold-vc-enlarged-class-bound}
give
\begin{align}
\operatorname{VC}
\left(
\mathcal H_{L-1,m,G}^{\pm}
\right)
\le
C_L
\left(
P+1
\right)
\ln
\left(
e
\left(
P+1
\right)
\right),
\end{align}
after relabeling the depth-dependent constant.
This proves
\eqref{eq:finite-architecture-signed-threshold-vc-bound}.
Substituting
$P=P_{L-1,m,G}$
proves
\eqref{eq:finite-architecture-signed-threshold-vc-bound-expanded}
and completes the proof.
\end{proof}

\begin{lemma}\textnormal{(Signed threshold entropy of finite lower-support architectures)}\ 
\label{lem:finite-bkl-threshold-entropy}
Fix
$L\ge2$,
$m\ge1$,
$G\ge2$,
$n\ge1$,
and sample points
$\b x_1,\ldots,\b x_n\in\mX$.
Let
\begin{align}
v_{L-1,m,G}
:=
\operatorname{VC}
\left(
\mathcal H_{L-1,m,G}^{\pm}
\right),
\label{eq:finite-threshold-entropy-vc-abbreviation}
\end{align}
where
$\mathcal H_{L-1,m,G}^{\pm}$
is defined in
\eqref{eq:finite-architecture-signed-threshold-class}. The signed trace family from
\eqref{eq:signed-threshold-trace-family}
satisfies
\begin{align}
\mathcal T_{n,\b x}^{\pm}
\left(
\mathcal A_{L-1}^{m,G}
\right)
=
\left\{
\left\{
i\in\left[n\right]
:
h\left(
\b x_i
\right)
=
1
\right\}
:
h\in
\mathcal H_{L-1,m,G}^{\pm}
\right\}.
\label{eq:finite-threshold-entropy-exact-trace-identity}
\end{align}
If
$v_{L-1,m,G}=0$,
then
\begin{align}
\left|
\mathcal T_{n,\b x}^{\pm}
\left(
\mathcal A_{L-1}^{m,G}
\right)
\right|
=
1.
\label{eq:finite-threshold-entropy-zero-vc}
\end{align}
If
$v_{L-1,m,G}\ge1$,
then
\begin{align}
\ln
\left|
\mathcal T_{n,\b x}^{\pm}
\left(
\mathcal A_{L-1}^{m,G}
\right)
\right|
&\le
\left(
v_{L-1,m,G}
\wedge
n
\right)
\ln
\left(
\frac{
en
}{
v_{L-1,m,G}
\wedge
n
}
\right).
\label{eq:finite-threshold-entropy-refined-bound}
\end{align}
Consequently, there exists a constant
$C_L>0$,
depending only on
$L$,
such that
\begin{align}
\ln
\left|
\mathcal T_{n,\b x}^{\pm}
\left(
\mathcal A_{L-1}^{m,G}
\right)
\right|
&\le
C_L
\left(
P_{L-1,m,G}+1
\right)
\ln
\left(
e
\left(
P_{L-1,m,G}+1
\right)
\right)
\ln
\left(
en
\right).
\label{eq:finite-threshold-entropy-final-bound}
\end{align}
\end{lemma}

\begin{proof}
For notational convenience, set
\begin{align}
\mathcal A
&:=
\mathcal A_{L-1}^{m,G},
&
\mathcal H^\pm
&:=
\mathcal H_{L-1,m,G}^{\pm},
&
v
&:=
v_{L-1,m,G}.
\end{align}
We first prove the exact trace identity
\eqref{eq:finite-threshold-entropy-exact-trace-identity}. Let
$T \in \mathcal T_{n,\b x}^{\pm} \left( \mathcal A \right)$.
By
\eqref{eq:signed-threshold-trace-family},
there exist
$a_{\boldsymbol\theta}\in\mathcal A$,
$\sigma\in\left\{-1,1\right\}$,
and
$\tau\in\mathbb R$
such that
\begin{align}
T
=
\left\{
i\in\left[n\right]
:
\sigma
a_{\boldsymbol\theta}
\left(
\b x_i
\right)
\ge
\tau
\right\}.
\label{eq:finite-threshold-entropy-arbitrary-trace}
\end{align}
Define
$h:=h_{\boldsymbol\theta,\tau,\sigma}$.
By
\eqref{eq:finite-architecture-signed-threshold-class},
one has
$h\in\mathcal H^\pm$.
Moreover,
\begin{align}
\left\{
i\in\left[n\right]
:
h\left(
\b x_i
\right)
=
1
\right\}
&\stackrel{(a)}{=}
\left\{
i\in\left[n\right]
:
\mathbf 1_{\left\{
\sigma
a_{\boldsymbol\theta}
\left(
\b x_i
\right)
\ge
\tau
\right\}}
=
1
\right\}
 \stackrel{(b)}{=}
\left\{
i\in\left[n\right]
:
\sigma
a_{\boldsymbol\theta}
\left(
\b x_i
\right)
\ge
\tau
\right\}
 \stackrel{(c)}{=}
T.
\end{align}
Here,
(a) applies
\eqref{eq:finite-architecture-signed-threshold-function},
(b) applies the defining property of an indicator function, and
(c) invokes
\eqref{eq:finite-threshold-entropy-arbitrary-trace}.
Therefore,
\begin{align}
\mathcal T_{n,\b x}^{\pm}
\left(
\mathcal A
\right)
\subseteq
\left\{
\left\{
i\in\left[n\right]
:
h\left(
\b x_i
\right)
=
1
\right\}
:
h\in\mathcal H^\pm
\right\}.
\label{eq:finite-threshold-entropy-first-set-inclusion}
\end{align}
Conversely, let
$h\in\mathcal H^\pm$.
By
\eqref{eq:finite-architecture-signed-threshold-class},
there exist
$\boldsymbol\theta\in\Theta_{L-1,m,G}$,
$\tau\in\mathbb R$,
and
$\sigma\in\left\{-1,1\right\}$
such that
$h=h_{\boldsymbol\theta,\tau,\sigma}$.
Therefore,
\begin{align}
\left\{
i\in\left[n\right]
:
h\left(
\b x_i
\right)
=
1
\right\}
&\stackrel{(a)}{=}
\left\{
i\in\left[n\right]
:
\sigma
a_{\boldsymbol\theta}
\left(
\b x_i
\right)
\ge
\tau
\right\}
 \stackrel{(b)}{\in}
\mathcal T_{n,\b x}^{\pm}
\left(
\mathcal A
\right).
\end{align}
Here,
(a) applies
\eqref{eq:finite-architecture-signed-threshold-function},
and
(b) applies
\eqref{eq:signed-threshold-trace-family}.
Hence,
\begin{align}
\left\{
\left\{
i\in\left[n\right]
:
h\left(
\b x_i
\right)
=
1
\right\}
:
h\in\mathcal H^\pm
\right\}
\subseteq
\mathcal T_{n,\b x}^{\pm}
\left(
\mathcal A
\right).
\label{eq:finite-threshold-entropy-second-set-inclusion}
\end{align}
Combining
\eqref{eq:finite-threshold-entropy-first-set-inclusion}
and
\eqref{eq:finite-threshold-entropy-second-set-inclusion}
proves
\eqref{eq:finite-threshold-entropy-exact-trace-identity}. We now count the traces.
By
\eqref{eq:finite-threshold-entropy-exact-trace-identity},
identifying each subset of
$\left[n\right]$
with its binary indicator vector shows that the cardinality of the
signed trace family is exactly the number of binary label vectors
induced by
$\mathcal H^\pm$
on the fixed sample:
\begin{align}
\left|
\mathcal T_{n,\b x}^{\pm}
\left(
\mathcal A
\right)
\right|
&=
\left|
\left\{
\left(
h\left(
\b x_1
\right),
\ldots,
h\left(
\b x_n
\right)
\right)
:
h\in\mathcal H^\pm
\right\}
\right|.
\label{eq:finite-threshold-entropy-label-vector-count}
\end{align}
Suppose first that
$v=0$.
If two members of
$\mathcal H^\pm$
gave different labels at some
$\b x\in\mX$,
then the singleton
$\left\{\b x\right\}$
would be shattered.
This would imply
$v\ge1$,
contradicting
$v=0$.
Thus, every member of
$\mathcal H^\pm$
has the same value at every point of
$\mX$.
Consequently, only one label vector occurs on the sample, and
$\left|\mathcal T_{n,\b x}^{\pm}\left(\mathcal A\right)\right|=1$.
This proves
\eqref{eq:finite-threshold-entropy-zero-vc}. Assume now that
$v\ge1$.
We distinguish two cases. If
$v<n$,
the Sauer--Shelah lemma and
\eqref{eq:finite-threshold-entropy-label-vector-count}
give
\begin{align}
\left|
\mathcal T_{n,\b x}^{\pm}
\left(
\mathcal A
\right)
\right|
&\stackrel{(a)}{\le}
\sum_{j=0}^{v}
\binom{n}{j}
 \stackrel{(b)}{\le}
\left(
\frac{
en
}{
v
}
\right)^v.
\label{eq:finite-threshold-entropy-sauer-small-vc}
\end{align}
Here,
(a) applies the Sauer--Shelah lemma to a class of VC dimension
$v$,
while
(b) applies the standard binomial-sum estimate for
$1\le v<n$. If
$v\ge n$,
every trace is a subset of
$\left[n\right]$,
and therefore
\begin{align}
\left|
\mathcal T_{n,\b x}^{\pm}
\left(
\mathcal A
\right)
\right|
&\stackrel{(a)}{\le}
2^n
\stackrel{(b)}{\le}
e^n
 \stackrel{(c)}{=}
\left(
\frac{
en
}{
n
}
\right)^n.
\label{eq:finite-threshold-entropy-large-vc}
\end{align}
Here,
(a) counts all subsets of
$\left[n\right]$,
(b) uses
$2\le e$,
and
(c) simplifies the fraction. Define
\begin{align}
q
:=
v\wedge n.
\label{eq:finite-threshold-entropy-q}
\end{align}
Both
\eqref{eq:finite-threshold-entropy-sauer-small-vc}
and
\eqref{eq:finite-threshold-entropy-large-vc}
can then be written as
\begin{align}
\left|
\mathcal T_{n,\b x}^{\pm}
\left(
\mathcal A
\right)
\right|
\le
\left(
\frac{
en
}{
q
}
\right)^q.
\label{eq:finite-threshold-entropy-unified-cardinality}
\end{align}
Taking natural logarithms in
\eqref{eq:finite-threshold-entropy-unified-cardinality}
gives
\begin{align}
\ln
\left|
\mathcal T_{n,\b x}^{\pm}
\left(
\mathcal A
\right)
\right|
&\stackrel{(a)}{\le}
q
\ln
\left(
\frac{
en
}{
q
}
\right)
 \stackrel{(b)}{=}
\left(
v\wedge n
\right)
\ln
\left(
\frac{
en
}{
v\wedge n
}
\right).
\label{eq:finite-threshold-entropy-refined-proof}
\end{align}
Here,
(a) takes logarithms, and
(b) substitutes
\eqref{eq:finite-threshold-entropy-q}.
This proves
\eqref{eq:finite-threshold-entropy-refined-bound}. It remains to derive the simpler parameter-dependent estimate.
Since
$q\ge1$,
the monotonicity of the logarithm gives
\begin{align}
\ln
\left(
\frac{
en
}{
q
}
\right)
&\le
\ln
\left(
en
\right).
\label{eq:finite-threshold-entropy-log-simplification}
\end{align}
Moreover,
\begin{align}
q
=
v\wedge n
\le
v.
\label{eq:finite-threshold-entropy-q-bounded-by-v}
\end{align}
Combining
\eqref{eq:finite-threshold-entropy-refined-proof},
\eqref{eq:finite-threshold-entropy-log-simplification},
and
\eqref{eq:finite-threshold-entropy-q-bounded-by-v}
gives
\begin{align}
\ln
\left|
\mathcal T_{n,\b x}^{\pm}
\left(
\mathcal A
\right)
\right|
\le
v
\ln
\left(
en
\right).
\label{eq:finite-threshold-entropy-vc-times-log-n}
\end{align}
Finally,
\Cref{lem:finite-bkl-vc-dimension}
gives
\begin{align}
v
&\stackrel{(a)}{=}
\operatorname{VC}
\left(
\mathcal H_{L-1,m,G}^{\pm}
\right)
 \stackrel{(b)}{\le}
C_L
\left(
P_{L-1,m,G}+1
\right)
\ln
\left(
e
\left(
P_{L-1,m,G}+1
\right)
\right).
\label{eq:finite-threshold-entropy-vc-substitution}
\end{align}
Here,
(a) applies
\eqref{eq:finite-threshold-entropy-vc-abbreviation},
and
(b) applies
\eqref{eq:finite-architecture-signed-threshold-vc-bound-expanded}.
Substituting
\eqref{eq:finite-threshold-entropy-vc-substitution}
into
\eqref{eq:finite-threshold-entropy-vc-times-log-n}
gives
\begin{align}
\ln
\left|
\mathcal T_{n,\b x}^{\pm}
\left(
\mathcal A_{L-1}^{m,G}
\right)
\right|
&\le
C_L
\left(
P_{L-1,m,G}+1
\right)
\ln
\left(
e
\left(
P_{L-1,m,G}+1
\right)
\right)
\ln
\left(
en
\right),
\end{align}
which proves
\eqref{eq:finite-threshold-entropy-final-bound}
and completes the proof.
\end{proof}

\begin{lemma}\textnormal{(Brownian chaos bound from finite-architecture threshold entropy)}\leavevmode\\
\label{lem:brownian-chaos-from-entropy}
Let
$L,m,G,n\in\mathbb N$
satisfy
$L\ge2$,
$m\ge1$,
$G\ge2$,
and
$n\ge1$,
and let
$\b x_1,\ldots,\b x_n\in\mX$
be fixed sample points.
Let
$P_{L-1,m,G}$
be the parameter-count bound specified in
\Cref{subsec:finite-parametric-bkl}
and established in
\Cref{lem:finite-bkl-parameter-count},
part~\ref{lem:finite-architecture-parameter-count}.
Set
\begin{align}
P
&:=
P_{L-1,m,G},
\label{eq:brownian-chaos-parameter-abbreviation}
\\
\Gamma_{L-1,m,G,n}
&:=
1
+
\left(
P+1
\right)
\ln\left(
e\left(
P+1
\right)
\right)
\ln\left(
en
\right),
\label{eq:brownian-chaos-complexity-factor}
\\
B_{\b x}
&:=
\sup_{a\in\mathcal A_{L-1}^{m,G}}
\max_{i\in\left[
n
\right]}
\left|
a\left(
\b x_i
\right)
\right|.
\label{eq:brownian-chaos-support-envelope}
\end{align}
Then
$
B_{\b x}<\infty.
$
Moreover, there exists
$C_L>0$,
depending only on
$L$,
such that
\begin{align}
\widehat{\mathfrak C}_{n,\b x}^{(B)}
\left(
\mathcal A_{L-1}^{m,G}
\right)
\le
C_L
B_{\b x}
\Gamma_{L-1,m,G,n}.
\label{eq:brownian-chaos-from-entropy-bound}
\end{align}

\end{lemma}

\begin{proof}
For notational convenience, set
\begin{align}
\mathcal A
&:=
\mathcal A_{L-1}^{m,G},
\label{eq:brownian-chaos-proof-architecture-abbreviation}
\\
\mathcal T
&:=
\mathcal T_{n,\b x}^{\pm}
\left(
\mathcal A
\right),
\label{eq:brownian-chaos-proof-trace-abbreviation}
\\
\Lambda_{L-1,m,G,n}
&:=
\left(
P+1
\right)
\ln
\left(
e
\left(
P+1
\right)
\right)
\ln
\left(
en
\right).
\label{eq:brownian-chaos-proof-lambda}
\end{align}
By
\eqref{eq:brownian-chaos-complexity-factor},
one has
\begin{align}
\Gamma_{L-1,m,G,n}
=
1
+
\Lambda_{L-1,m,G,n}.
\label{eq:brownian-chaos-proof-gamma-lambda}
\end{align}
We first verify that the support envelope is finite.
By
\Cref{lem:finite-bkl-parameter-count}\ref{lem:finite-architecture-range-control},
every
$a\in\mathcal A$
satisfies
\begin{align}
\left|
a\left(
\b x
\right)
\right|
\le
A_{\mX}^{\,2^{-(L-2)}},
\qquad
\b x\in\mX.
\label{eq:brownian-chaos-proof-global-range}
\end{align}
Consequently,
\begin{align}
B_{\b x}
&\stackrel{(a)}{=}
\sup_{a\in\mathcal A}
\max_{i\in\left[n\right]}
\left|
a\left(
\b x_i
\right)
\right|
 \stackrel{(b)}{\le}
A_{\mX}^{\,2^{-(L-2)}}
\stackrel{(c)}{<}
\infty.
\end{align}
Here,
(a) applies
\eqref{eq:brownian-chaos-support-envelope},
(b) applies
\eqref{eq:brownian-chaos-proof-global-range}
at every sample point, and
(c) uses the compactness of
$\mX$
and the definition of
$A_{\mX}$
in
\eqref{eq:finite-architecture-input-radius}. We next control the signed threshold entropy.
By
\Cref{lem:finite-bkl-threshold-entropy},
there exists a constant
$C_{\mathrm{ent},L}>0$,
depending only on
$L$,
such that
\begin{align}
\ln
\left|
\mathcal T
\right|
&\stackrel{(a)}{\le}
C_{\mathrm{ent},L}
\left(
P+1
\right)
\ln
\left(
e
\left(
P+1
\right)
\right)
\ln
\left(
en
\right)
 \stackrel{(b)}{=}
C_{\mathrm{ent},L}
\Lambda_{L-1,m,G,n}.
\label{eq:brownian-chaos-proof-threshold-entropy}
\end{align}
Here,
(a) applies
\eqref{eq:finite-threshold-entropy-final-bound},
while
(b) applies
\eqref{eq:brownian-chaos-proof-lambda}. Define
\begin{align}
E_{L-1,m,G,n}
:=
C_{\mathrm{ent},L}
\Lambda_{L-1,m,G,n}.
\label{eq:brownian-chaos-proof-entropy-quantity}
\end{align}
Then
\eqref{eq:brownian-chaos-proof-threshold-entropy}
gives
\begin{align}
\ln
\left|
\mathcal T_{n,\b x}^{\pm}
\left(
\mathcal A
\right)
\right|
\le
E_{L-1,m,G,n}.
\label{eq:brownian-chaos-proof-entropy-assumption}
\end{align}
Applying
\Cref{lem:threshold-traces-control-chaos}
with the support class
$\mathcal A$
and the envelope
$B_{\b x}$
gives
\begin{align}
\widehat{\mathfrak C}_{n,\b x}^{(B)}
\left(
\mathcal A
\right)
&\stackrel{(a)}{\le}
C_{\mathrm{tr}}
B_{\b x}
\left(
1
+
E_{L-1,m,G,n}
\right)
 \stackrel{(b)}{=}
C_{\mathrm{tr}}
B_{\b x}
\left(
1
+
C_{\mathrm{ent},L}
\Lambda_{L-1,m,G,n}
\right),
\label{eq:brownian-chaos-proof-threshold-application}
\end{align}
where
$C_{\mathrm{tr}}>0$
is a universal numerical constant.
Here,
(a) applies
\eqref{eq:threshold-traces-brownian-chaos-entropy-bound}
using
\eqref{eq:brownian-chaos-proof-entropy-assumption},
and
(b) substitutes
\eqref{eq:brownian-chaos-proof-entropy-quantity}. Set
$ C_{\mathrm{aux},L} : \allowbreak= \max \left\{ 1, C_{\mathrm{ent},L} \right\}. $
Since
$\Lambda_{L-1,m,G,n}\ge0$,
we obtain
\begin{align}
1
+
C_{\mathrm{ent},L}
\Lambda_{L-1,m,G,n}
&\stackrel{(a)}{\le}
C_{\mathrm{aux},L}
+
C_{\mathrm{aux},L}
\Lambda_{L-1,m,G,n}
\notag\\
&\stackrel{(b)}{=}
C_{\mathrm{aux},L}
\left(
1
+
\Lambda_{L-1,m,G,n}
\right)
\notag\\
&\stackrel{(c)}{=}
C_{\mathrm{aux},L}
\Gamma_{L-1,m,G,n}.
\label{eq:brownian-chaos-proof-complexity-simplification}
\end{align}
Here,
(a) uses
$1 \le C_{\mathrm{aux},L}, C_{\mathrm{ent},L} \le C_{\mathrm{aux},L}$,
(b) factors out
$C_{\mathrm{aux},L}$,
and
(c) applies
\eqref{eq:brownian-chaos-proof-gamma-lambda}. Substituting
\eqref{eq:brownian-chaos-proof-complexity-simplification}
into
\eqref{eq:brownian-chaos-proof-threshold-application}
gives
\begin{align}
\widehat{\mathfrak C}_{n,\b x}^{(B)}
\left(
\mathcal A
\right)
&\stackrel{(a)}{\le}
C_{\mathrm{tr}}
C_{\mathrm{aux},L}
B_{\b x}
\Gamma_{L-1,m,G,n}
 \stackrel{(b)}{=}
C_L
B_{\b x}
\Gamma_{L-1,m,G,n},
\end{align}
where
$C_L := C_{\mathrm{tr}} C_{\mathrm{aux},L}$.
Here,
(a) combines
\eqref{eq:brownian-chaos-proof-threshold-application}
and
\eqref{eq:brownian-chaos-proof-complexity-simplification},
while
(b) defines the depth-dependent constant
$C_L$.
Substituting
$\mathcal A=\mathcal A_{L-1}^{m,G}$
proves
\eqref{eq:brownian-chaos-from-entropy-bound}
and completes the proof.
\end{proof}

\begin{lemma}\textnormal{(Uniform H\"older, pointwise, and $L^2(\nu)$ bounds for atomic VBKL generators)}\ ~
\label{lem:path-atom-bound}
Let $L\ge2$ and assume that $\|\bomega\|_2\le1$ for every
$\bomega\in\Omega$. Define
\begin{align}
R_L
:=
\left(
\int_{\mX}
\|\b x\|_2^{\,2^{-(L-2)}}
\,\d\nu(\b x)
\right)^{1/2}.
\label{eq:atomic-generator-l2-radius}
\end{align}
Then every $f\in\mU_L$ satisfies
\begin{align}
|f(\b x)-f(\b x')|
&\le
\|\b x-\b x'\|_2^{\,2^{-(L-1)}},
\qquad
\b x,\b x'\in\mX,
\notag\\
|f(\b x)|
&\le
\|\b x\|_2^{\,2^{-(L-1)}},
\qquad
\b x\in\mX,
\label{eq:path-atom-pointwise-bound}
\end{align}
and consequently
\begin{align}
\|f\|_{L^2(\nu)}
\le
R_L.
\label{eq:path-atom-l2-bound}
\end{align}
\end{lemma}

\begin{proof}
For $g\in\mH_{\kB}$ with $\|g\|_{\mH_{\kB}}\le1$, the reproducing property
and the Brownian kernel metric give
\begin{align}
|g(s)-g(t)|
&\le
\|\kB(s,\cdot)-\kB(t,\cdot)\|_{\mH_{\kB}}
=
|s-t|^{1/2},
\notag\\
|g(s)|
&\le
\|\kB(s,\cdot)\|_{\mH_{\kB}}
=
|s|^{1/2}.
\end{align}
Consequently, for every support $a:\mX\to\mathbb R$,
\begin{align}
|g(a(\b x))-g(a(\b x'))|
&\le
|a(\b x)-a(\b x')|^{1/2},
\notag\\
|g(a(\b x))|
&\le
|a(\b x)|^{1/2}.
\label{eq:recursive-atomic-pointwise-step}
\end{align}
At the first level,
$|\bomega^\top(\b x-\b x')|\le\|\b x-\b x'\|_2$ and
$|\bomega^\top\b x|\le\|\b x\|_2$. Applying
\eqref{eq:recursive-atomic-pointwise-step} at each of the $L-1$ Brownian
compositions proves the two atomic bounds. Squaring the pointwise estimate
and integrating yields
\begin{align}
\|f\|_{L^2(\nu)}^2
\le
\int_{\mX}\|\b x\|_2^{\,2^{-(L-2)}}\,\d\nu(\b x)
=
R_L^2,
\end{align}
which proves the claim.
\end{proof}

\begin{lemma}\textnormal{(Attainment of the variation complexity)}\ 
\label{prop:variation-gauge-attainment}
Assume that
$\mU_L$
is compact in
$L^2\left(\nu\right)$.
Then, for every
$F\in\VL$,
there exists a finite signed Radon measure
$\mu_F$
on
$\mU_L$
such that
\begin{align}
F
=
\int_{\mU_L}
u
\,\d\mu_F\left(
u
\right)
\qquad
\text{in }
L^2\left(
\nu
\right),
\label{eq:variation-gauge-attainment-representation}
\end{align}
and
\begin{align}
\left\|
\mu_F
\right\|_{\mathrm{TV}}
=
\CLv\left(
F
\right).
\label{eq:variation-gauge-attainment-norm}
\end{align}

\end{lemma}

\begin{proof}
Fix
$F\in\VL$
and set
$ c_F : \allowbreak= \CLv\left( F \right). $
Since
$F\in\VL$,
one has
$c_F<\infty$.
By the definition of
$c_F$,
for every
$n\ge1$
there exists a finite signed measure
$\mu_n$
on
$\mU_L$
such that
\begin{align}
F
&=
\int_{\mU_L}
u
\,\d\mu_n\left(
u
\right)
\qquad
\text{in }
L^2\left(
\nu
\right),
\label{eq:variation-gauge-attainment-minimizing-representation}
\\
\left\|
\mu_n
\right\|_{\mathrm{TV}}
&\le
c_F
+
\frac{1}{n}.
\label{eq:variation-gauge-attainment-minimizing-norm}
\end{align}
Because
$\mU_L$
is a compact metric space, every finite signed Borel measure on
$\mU_L$
is a finite signed Radon measure.
Moreover,
\eqref{eq:variation-gauge-attainment-minimizing-norm}
implies
$ \sup_{n\ge1} \left\| \mu_n \right\|_{\mathrm{TV}} \le c_F+1 < \infty. $
Since
$\mU_L$
is compact and metrizable,
$C\left(\mU_L\right)$
is separable.
By the Riesz--Markov representation theorem, the space of finite signed Radon measures on
$\mU_L$
is isometrically identified with
$C\left(\mU_L\right)^*$,
where the dual norm is the total-variation norm.
The sequence
$\left(\mu_n\right)_{n\ge1}$
therefore lies in the closed dual ball of radius
$c_F+1$.
By the Banach--Alaoglu theorem, this ball is weak-* compact.
Since
$C\left(\mU_L\right)$
is separable, the weak-* topology on this ball is metrizable, and the ball is consequently weak-* sequentially compact.
Hence, after passing to a subsequence, not relabelled, there exists a finite signed Radon measure
$\mu_F$
on
$\mU_L$
such that
\begin{align}
\mu_n
\stackrel{*}{\rightharpoonup}
\mu_F
\qquad
\text{in }
C\left(
\mU_L
\right)^*.
\label{eq:variation-gauge-attainment-weak-star-convergence}
\end{align}
We next prove that
$\mu_F$
represents
$F$.
Since
$\mU_L$
is compact in
$L^2\left(\nu\right)$,
the canonical map
\begin{align}
\iota
:
\mU_L
\longrightarrow
L^2\left(
\nu
\right),
\qquad
\iota\left(
u
\right)
:=
u,
\end{align}
is continuous and bounded.
It is therefore strongly measurable and Bochner integrable with respect to the finite signed measure
$\mu_F$.
Consequently,
$ \int_{\mU_L} u \,\d\mu_F\left( u \right) \in L^2\left( \nu \right) $
is well-defined.
Fix
$\varphi\in L^2\left(\nu\right)$.
The scalar-valued map
\begin{align}
u
\longmapsto
\left\langle
u,
\varphi
\right\rangle_{L^2\left(\nu\right)}
\label{eq:variation-gauge-attainment-test-function}
\end{align}
is continuous on
$\mU_L$.
Therefore,
\begin{align}
\left\langle
F,
\varphi
\right\rangle_{L^2\left(\nu\right)}
&\stackrel{(a)}{=}
\left\langle
\int_{\mU_L}
u
\,\d\mu_n\left(
u
\right),
\varphi
\right\rangle_{L^2\left(\nu\right)}
\stackrel{(b)}{=}
\int_{\mU_L}
\left\langle
u,
\varphi
\right\rangle_{L^2\left(\nu\right)}
\,\d\mu_n\left(
u
\right)
\notag\\
&\stackrel{(c)}{\longrightarrow}
\int_{\mU_L}
\left\langle
u,
\varphi
\right\rangle_{L^2\left(\nu\right)}
\,\d\mu_F\left(
u
\right)
\stackrel{(d)}{=}
\left\langle
\int_{\mU_L}
u
\,\d\mu_F\left(
u
\right),
\varphi
\right\rangle_{L^2\left(\nu\right)}.
\label{eq:variation-gauge-attainment-duality-limit}
\end{align}
Here,
(a) applies
\eqref{eq:variation-gauge-attainment-minimizing-representation},
(b) applies the duality identity for the Bochner integral,
(c) follows from
\eqref{eq:variation-gauge-attainment-weak-star-convergence}
and the continuity of
\eqref{eq:variation-gauge-attainment-test-function},
and
(d) applies the Bochner-integral duality identity once more.
Since the left-hand side of
\eqref{eq:variation-gauge-attainment-duality-limit}
does not depend on
$n$,
the limit identity gives
\begin{align}
\left\langle
F,
\varphi
\right\rangle_{L^2\left(\nu\right)}
=
\left\langle
\int_{\mU_L}
u
\,\d\mu_F\left(
u
\right),
\varphi
\right\rangle_{L^2\left(\nu\right)}
\end{align}
for every
$\varphi\in L^2\left(\nu\right)$.
The non-degeneracy of the
$L^2\left(\nu\right)$
inner product therefore yields
\begin{align}
F
=
\int_{\mU_L}
u
\,\d\mu_F\left(
u
\right)
\qquad
\text{in }
L^2\left(
\nu
\right),
\end{align}
which proves
\eqref{eq:variation-gauge-attainment-representation}.
It remains to prove attainment of the variation complexity.
The total-variation norm is weak-* lower semicontinuous, and hence
\begin{align}
\left\|
\mu_F
\right\|_{\mathrm{TV}}
&\stackrel{(a)}{\le}
\liminf_{n\rightarrow\infty}
\left\|
\mu_n
\right\|_{\mathrm{TV}}
\stackrel{(b)}{\le}
\lim_{n\rightarrow\infty}
\left(
c_F
+
\frac{1}{n}
\right)
\stackrel{(c)}{=}
c_F
\stackrel{(d)}{\le}
\left\|
\mu_F
\right\|_{\mathrm{TV}}.
\label{eq:variation-gauge-attainment-final-chain}
\end{align}
Here,
(a) applies weak-* lower semicontinuity of the dual norm,
(b) applies
\eqref{eq:variation-gauge-attainment-minimizing-norm},
(c) evaluates the limit, and
(d) follows from the definition of
$c_F=\CLv\left(F\right)$
because
$\mu_F$
represents
$F$.
Thus, every inequality in
\eqref{eq:variation-gauge-attainment-final-chain}
is an equality, and therefore
$\left\| \mu_F \right\|_{\mathrm{TV}} = c_F = \CLv\left( F \right)$.
This proves
\eqref{eq:variation-gauge-attainment-norm}
and completes the proof.
\end{proof}

\begin{lemma}\textnormal{(Compactness of restricted Brownian profiles)}\ 
\label{lem:brownian-profile-compactness}
Let $A>0$, and define
\begin{align}
\mathcal G_A
:=
\left\{
 g|_{\left[-A,A\right]}
 :
 g\in\mathcal H_{\kB},
 \ \left\|g\right\|_{\mathcal H_{\kB}}\le1
\right\}.
\end{align}
Then $\mathcal G_A$ is compact in
$C\left(\left[-A,A\right]\right)$ equipped with the supremum norm.
\end{lemma}

\begin{proof}
Since
$C\left(\left[-A,A\right]\right)$,
equipped with the supremum norm, is a metric space, it is sufficient to
prove sequential compactness.
Let
$\left(g_n\right)_{n\ge1}$
be an arbitrary sequence in
$\mathcal G_A$.
By the definition of
$\mathcal G_A$,
for every
$n\ge1$
there exists
$\widetilde g_n\in\mathcal H_{\kB}$
such that
\begin{align}
g_n
&=
\widetilde g_n|_{\left[-A,A\right]},
\label{eq:restricted-profile-extension-representation}
\\
\left\|
\widetilde g_n
\right\|_{\mathcal H_{\kB}}
&\le
1.
\label{eq:restricted-profile-extension-norm}
\end{align}
By the characterization of the Brownian RKHS,
$\widetilde g_n\left(0\right)=0$,
the function
$\widetilde g_n$
is absolutely continuous, and
$\widetilde g_n'\in L^2\left(\mathbb R\right)$.
Moreover,
\begin{align}
\int_{\mathbb R}
\left|
\widetilde g_n'\left(
t
\right)
\right|^2
\,\d t
&\stackrel{(a)}{=}
\left\|
\widetilde g_n
\right\|_{\mathcal H_{\kB}}^2
\stackrel{(b)}{\le}
1.
\label{eq:restricted-profile-global-energy}
\end{align}
Here,
(a) applies the norm characterization of
$\mathcal H_{\kB}$,
while
(b) applies
\eqref{eq:restricted-profile-extension-norm}.
We first establish equicontinuity.
Fix
$s,t\in\left[-A,A\right]$
with
$s<t$.
Since
$\widetilde g_n$
is absolutely continuous,
\begin{align}
g_n\left(
t
\right)
-
g_n\left(
s
\right)
&\stackrel{(a)}{=}
\widetilde g_n\left(
t
\right)
-
\widetilde g_n\left(
s
\right)
\stackrel{(b)}{=}
\int_s^t
\widetilde g_n'\left(
r
\right)
\,\d r.
\end{align}
Here,
(a) applies
\eqref{eq:restricted-profile-extension-representation},
and
(b) applies the fundamental theorem of calculus for absolutely
continuous functions.
Consequently,
\begin{align}
\left|
g_n\left(
t
\right)
-
g_n\left(
s
\right)
\right|
&\stackrel{(a)}{\le}
\left(
\int_s^t
\left|
\widetilde g_n'\left(
r
\right)
\right|^2
\,\d r
\right)^{1/2}
\left|
t-s
\right|^{1/2}
\notag\\
&\stackrel{(b)}{\le}
\left(
\int_{\mathbb R}
\left|
\widetilde g_n'\left(
r
\right)
\right|^2
\,\d r
\right)^{1/2}
\left|
t-s
\right|^{1/2}
\stackrel{(c)}{\le}
\left|
t-s
\right|^{1/2}.
\label{eq:restricted-profile-holder}
\end{align}
Here,
(a) applies the Cauchy--Schwarz inequality on
$\left[s,t\right]$,
(b) enlarges the domain of integration, and
(c) applies
\eqref{eq:restricted-profile-global-energy}.
Thus,
$\left(g_n\right)_{n\ge1}$
is uniformly
$1/2$-Hölder continuous, and hence equicontinuous, on
$\left[-A,A\right]$.
We next establish uniform boundedness.
Since
\begin{align}
g_n\left(
0
\right)
&\stackrel{(a)}{=}
\widetilde g_n\left(
0
\right)
\stackrel{(b)}{=}
0,
\label{eq:restricted-profile-anchor}
\end{align}
where
(a) applies
\eqref{eq:restricted-profile-extension-representation}
and
(b) uses the Brownian RKHS anchor condition, the estimate
\eqref{eq:restricted-profile-holder}
gives, for every
$t\in\left[-A,A\right]$,
\begin{align}
\left|
g_n\left(
t
\right)
\right|
&\stackrel{(a)}{=}
\left|
g_n\left(
t
\right)
-
g_n\left(
0
\right)
\right|
\stackrel{(b)}{\le}
\left|
t
\right|^{1/2}
\stackrel{(c)}{\le}
A^{1/2}.
\end{align}
Here,
(a) applies
\eqref{eq:restricted-profile-anchor},
(b) applies
\eqref{eq:restricted-profile-holder}
with one endpoint equal to zero, and
(c) uses
$\left|t\right|\le A$.
The sequence
$\left(g_n\right)_{n\ge1}$
is therefore equicontinuous and uniformly bounded in
$C\left(\left[-A,A\right]\right)$.
By the Arzelà--Ascoli theorem, there exist a subsequence, not
relabelled, and a function
$g \in C\left( \left[-A,A\right] \right)$
such that
\begin{align}
\left\|
g_n-g
\right\|_{L^\infty\left(\left[-A,A\right]\right)}
\longrightarrow
0.
\label{eq:restricted-profile-uniform-convergence}
\end{align}
Since
$g_n\left(0\right)=0$
for every
$n\ge1$,
the uniform convergence gives
\begin{align}
\left|
g\left(
0
\right)
\right|
&\stackrel{(a)}{=}
\left|
g\left(
0
\right)
-
g_n\left(
0
\right)
\right|
\stackrel{(b)}{\le}
\left\|
g-g_n
\right\|_{L^\infty\left(\left[-A,A\right]\right)}
\stackrel{(c)}{\longrightarrow}
0.
\end{align}
Here,
(a) uses
$g_n\left(0\right)=0$,
(b) bounds pointwise evaluation by the supremum norm, and
(c) applies
\eqref{eq:restricted-profile-uniform-convergence}.
Hence,
\begin{align}
g\left(
0
\right)
=
0.
\label{eq:restricted-profile-limit-anchor-zero}
\end{align}
It remains to prove that
$g\in\mathcal G_A$.
For every
$n\ge1$,
define
\begin{align}
h_n
:=
\widetilde g_n'|_{\left[-A,A\right]}.
\label{eq:restricted-profile-derivative-restriction}
\end{align}
Then
\begin{align}
\left\|
h_n
\right\|_{L^2\left(\left[-A,A\right]\right)}^2
&\stackrel{(a)}{=}
\int_{-A}^{A}
\left|
\widetilde g_n'\left(
t
\right)
\right|^2
\,\d t
\stackrel{(b)}{\le}
\int_{\mathbb R}
\left|
\widetilde g_n'\left(
t
\right)
\right|^2
\,\d t
\stackrel{(c)}{\le}
1.
\label{eq:restricted-profile-local-energy}
\end{align}
Here,
(a) applies
\eqref{eq:restricted-profile-derivative-restriction},
(b) enlarges the domain of integration, and
(c) applies
\eqref{eq:restricted-profile-global-energy}.
Thus,
$\left(h_n\right)_{n\ge1}$
is a bounded sequence in the Hilbert space
$L^2\left(\left[-A,A\right]\right)$.
Since Hilbert spaces are reflexive, bounded sequences are weakly
sequentially relatively compact.
Therefore, after passing to a further subsequence, not relabelled,
there exists
$h \in L^2\left( \left[-A,A\right] \right)$
such that
\begin{align}
h_n
\rightharpoonup
h
\qquad
\text{weakly in }
L^2\left(
\left[-A,A\right]
\right).
\label{eq:restricted-profile-weak-derivatives}
\end{align}
The further subsequence continues to satisfy
\eqref{eq:restricted-profile-uniform-convergence}.
Fix
$s,t\in\left[-A,A\right]$
with
$s<t$.
By absolute continuity and
\eqref{eq:restricted-profile-derivative-restriction},
\begin{align}
g_n\left(
t
\right)
-
g_n\left(
s
\right)
=
\int_s^t
h_n\left(
r
\right)
\,\d r.
\label{eq:restricted-profile-derivative-identity}
\end{align}
The uniform convergence gives
\begin{align}
g_n\left(
t
\right)
-
g_n\left(
s
\right)
\longrightarrow
g\left(
t
\right)
-
g\left(
s
\right).
\label{eq:restricted-profile-increment-limit}
\end{align}
On the other hand,
\begin{align}
\int_s^t
h_n\left(
r
\right)
\,\d r
&\stackrel{(a)}{=}
\left\langle
h_n,
\mathbf 1_{\left[s,t\right]}
\right\rangle_{L^2\left(\left[-A,A\right]\right)}
\stackrel{(b)}{\longrightarrow}
\left\langle
h,
\mathbf 1_{\left[s,t\right]}
\right\rangle_{L^2\left(\left[-A,A\right]\right)}
\stackrel{(c)}{=}
\int_s^t
h\left(
r
\right)
\,\d r.
\label{eq:restricted-profile-derivative-weak-limit}
\end{align}
Here,
(a) rewrites the integral as an
$L^2\left(\left[-A,A\right]\right)$
inner product,
(b) applies
\eqref{eq:restricted-profile-weak-derivatives}
with the fixed test function
$\mathbf 1_{\left[s,t\right]}$, and
(c) rewrites the limiting inner product as an integral.
Comparing
\eqref{eq:restricted-profile-increment-limit}
and
\eqref{eq:restricted-profile-derivative-weak-limit}
in
\eqref{eq:restricted-profile-derivative-identity}
gives
\begin{align}
g\left(
t
\right)
-
g\left(
s
\right)
=
\int_s^t
h\left(
r
\right)
\,\d r,
\qquad
-A
\le
s
<
t
\le
A.
\label{eq:restricted-profile-limit-ac}
\end{align}
Since the interval
$\left[-A,A\right]$
has finite measure and
$h\in L^2\left(\left[-A,A\right]\right)$,
the Cauchy--Schwarz inequality gives
\begin{align}
\int_{-A}^{A}
\left|
h\left(
r
\right)
\right|
\,\d r
&\stackrel{(a)}{\le}
\left(
2A
\right)^{1/2}
\left\|
h
\right\|_{L^2\left(\left[-A,A\right]\right)}
\stackrel{(b)}{<}
\infty.
\end{align}
Here,
(a) applies the Cauchy--Schwarz inequality, and
(b) uses
$h\in L^2\left(\left[-A,A\right]\right)$.
Thus,
$h\in L^1\left(\left[-A,A\right]\right)$.
Equation
\eqref{eq:restricted-profile-limit-ac}
therefore shows that
$g$
is absolutely continuous on
$\left[-A,A\right]$
and that
$ g' \allowbreak= h \qquad \text{almost everywhere on } \left[-A,A\right]. $
By weak lower semicontinuity of the Hilbert-space norm,
\begin{align}
\int_{-A}^{A}
\left|
h\left(
r
\right)
\right|^2
\,\d r
&\stackrel{(a)}{=}
\left\|
h
\right\|_{L^2\left(\left[-A,A\right]\right)}^2
\stackrel{(b)}{\le}
\liminf_{n\rightarrow\infty}
\left\|
h_n
\right\|_{L^2\left(\left[-A,A\right]\right)}^2
\stackrel{(c)}{\le}
1.
\label{eq:restricted-profile-limit-energy}
\end{align}
Here,
(a) applies the definition of the
$L^2$
norm,
(b) applies weak lower semicontinuity under
\eqref{eq:restricted-profile-weak-derivatives},
and
(c) applies
\eqref{eq:restricted-profile-local-energy}.
Define the zero extension of
$h$
by
\begin{align}
\widehat h\left(
t
\right)
:=
h\left(
t
\right)
\mathbf 1_{\left[-A,A\right]}
\left(
t
\right),
\qquad
t\in\mathbb R,
\label{eq:restricted-profile-limit-zero-derivative}
\end{align}
and define
\begin{align}
\widehat g\left(
t
\right)
:=
\int_0^t
\widehat h\left(
r
\right)
\,\d r,
\qquad
t\in\mathbb R.
\label{eq:restricted-profile-limit-zero}
\end{align}
Because
$\widehat h$
has bounded support and belongs to
$L^2\left(\mathbb R\right)$,
one also has
$\widehat h\in L^1\left(\mathbb R\right)$.
Therefore,
$\widehat g$
is absolutely continuous on
$\mathbb R$,
satisfies
$ \widehat g\left( 0 \right) \allowbreak= 0, $
and has weak derivative
\begin{align}
\widehat g'
=
\widehat h
\qquad
\text{almost everywhere on }
\mathbb R.
\label{eq:restricted-profile-global-derivative}
\end{align}
We now verify that
$\widehat g$
extends
$g$.
If
$t\in\left[0,A\right]$,
then
\begin{align}
\widehat g\left(
t
\right)
&\stackrel{(a)}{=}
\int_0^t
h\left(
r
\right)
\,\d r
\stackrel{(b)}{=}
g\left(
t
\right)
-
g\left(
0
\right)
\stackrel{(c)}{=}
g\left(
t
\right).
\end{align}
Here,
(a) applies
\eqref{eq:restricted-profile-limit-zero-derivative},
(b) applies
\eqref{eq:restricted-profile-limit-ac}
with
$s=0$,
and
(c) applies
\eqref{eq:restricted-profile-limit-anchor-zero}.
If
$t\in\left[-A,0\right]$,
then
\begin{align}
\widehat g\left(
t
\right)
&\stackrel{(a)}{=}
\int_0^t
h\left(
r
\right)
\,\d r
\stackrel{(b)}{=}
-
\int_t^0
h\left(
r
\right)
\,\d r
 \stackrel{(c)}{=}
-
\left(
g\left(
0
\right)
-
g\left(
t
\right)
\right)
\stackrel{(d)}{=}
g\left(
t
\right).
\end{align}
Here,
(a) applies
\eqref{eq:restricted-profile-limit-zero},
(b) reverses the orientation of the integral,
(c) applies
\eqref{eq:restricted-profile-limit-ac},
and
(d) applies
\eqref{eq:restricted-profile-limit-anchor-zero}.
Therefore,
\begin{align}
\widehat g|_{\left[-A,A\right]}
=
g.
\label{eq:restricted-profile-extension-identity}
\end{align}
Finally,
\begin{align}
\left\|
\widehat g
\right\|_{\mathcal H_{\kB}}^2
&\stackrel{(a)}{=}
\int_{\mathbb R}
\left|
\widehat g'\left(
t
\right)
\right|^2
\,\d t
\stackrel{(b)}{=}
\int_{\mathbb R}
\left|
\widehat h\left(
t
\right)
\right|^2
\,\d t
 \stackrel{(c)}{=}
\int_{-A}^{A}
\left|
h\left(
t
\right)
\right|^2
\,\d t
\stackrel{(d)}{\le}
1.
\end{align}
Here,
(a) applies the Brownian RKHS norm characterization,
(b) applies
\eqref{eq:restricted-profile-global-derivative},
(c) applies the definition of the zero extension
\eqref{eq:restricted-profile-limit-zero-derivative},
and
(d) applies
\eqref{eq:restricted-profile-limit-energy}.
Hence,
\begin{align}
\widehat g
\in
\mathcal H_{\kB},
\qquad
\left\|
\widehat g
\right\|_{\mathcal H_{\kB}}
\le
1.
\end{align}
Together with
\eqref{eq:restricted-profile-extension-identity},
the definition of
$\mathcal G_A$
therefore gives
$g \in \mathcal G_A$.
We have shown that every sequence in
$\mathcal G_A$
admits a subsequence converging uniformly on
$\left[-A,A\right]$
to an element of
$\mathcal G_A$.
Thus,
$\mathcal G_A$
is sequentially compact in
$C\left(\left[-A,A\right]\right)$.
Since this is a metric space, sequential compactness is equivalent to
compactness.
Therefore,
$\mathcal G_A$
is compact in
$C\left(\left[-A,A\right]\right)$
equipped with the supremum norm.
\end{proof}

\begin{lemma}\textnormal{(Compactness of the atomic VBKL dictionary)}\ 
\label{lem:atomic-bkl-dictionary-compact}
Let
$\mX\subseteq\mathbb R^d$
be compact, let
$\nu$
be a Borel probability measure on
$\mX$,
and let
$\Omega\subseteq\mathbb S^{d-1}$
be compact.
Let
$\left(
\mU_l
\right)_{l\ge1}$
be the recursive atomic dictionaries introduced in
\Cref{subsec:atomic-bkl}. Then, for every
$L\ge1$,
the class
$\mU_L$
is compact in
$C\left(\mX\right)$
equipped with the supremum norm.
Consequently, under the canonical continuous embedding
$
C\left(
\mX
\right)
\longrightarrow
L^2\left(
\nu
\right),
$
the class
$\mU_L$
is also compact in
$L^2\left(\nu\right)$.
\end{lemma}

\begin{proof}
We prove the first assertion by induction on the recursion level.
For the base case, define
\begin{align}
T
&:
\Omega
\longrightarrow
C\left(
\mX
\right),
&
\left(
T\bomega
\right)
\left(
\b x
\right)
&:=
\bomega^\top\b x.
\end{align}
Set
$M_{\mX} := \sup_{\b x\in\mX} \left\| \b x \right\|_2$.
Since
$\mX$
is compact and
$\b x\mapsto\left\|\b x\right\|_2$
is continuous,
one has
$M_{\mX} < \infty$.
For every
$\bomega,\bomega'\in\Omega$,
\begin{align}
\left\|
T\bomega
-
T\bomega'
\right\|_\infty
&\stackrel{(a)}{=}
\sup_{\b x\in\mX}
\left|
\left(
\bomega-\bomega'
\right)^\top
\b x
\right|
 \stackrel{(b)}{\le}
\sup_{\b x\in\mX}
\left\|
\bomega-\bomega'
\right\|_2
\left\|
\b x
\right\|_2
\stackrel{(c)}{=}
M_{\mX}
\left\|
\bomega-\bomega'
\right\|_2.
\end{align}
Here,
(a) applies the definition of
$T$
and the supremum norm,
(b) applies the Euclidean Cauchy--Schwarz inequality, and
(c) factors out the quantity
$\left\|\bomega-\bomega'\right\|_2$
and applies the definition of
$M_{\mX}$.
Thus,
$T$
is Lipschitz continuous.
Since
$\Omega$
is compact and
$\mU_1 = T\left( \Omega \right)$,
the continuity of
$T$
implies that
$\mU_1$
is compact in
$C\left(\mX\right)$.

Assume now that
$\mU_l$
is compact in
$C\left(\mX\right)$
for some
$l\ge1$.
We prove that
$\mU_{l+1}$
is compact in
$C\left(\mX\right)$.

Let
$\left(f_n\right)_{n\ge1}$
be an arbitrary sequence in
$\mU_{l+1}$.
By the recursive definition of the atomic dictionary, for every
$n\ge1$
there exist
$a_n \in \mU_l, g_n \in \mH_{\kB}$
such that
\begin{align}
\left\|
g_n
\right\|_{\mH_{\kB}}
&\le
1,
\\
f_n\left(
\b x
\right)
&=
g_n\left(
a_n\left(
\b x
\right)
\right),
\qquad
\b x\in\mX.
\end{align}
By the induction hypothesis,
$\mU_l$
is compact in
$C\left(\mX\right)$.
Hence, after passing to a subsequence, not relabelled, there exists
$a\in\mU_l$
such that
\begin{align}
\left\|
a_n-a
\right\|_\infty
\longrightarrow
0.
\label{eq:compact-dict-an-uniform}
\end{align}
Since
$\mU_l$
is compact in
$C\left(\mX\right)$
and the supremum norm is continuous,
the class
$\mU_l$
is bounded.
Define
$A := \max \left\{ 1, \sup_{b\in\mU_l} \left\| b \right\|_\infty \right\}$.
Then
$A<\infty$,
and
\begin{align}
a\left(
\mX
\right)
&\subseteq
\left[
-A,A
\right],
&
a_n\left(
\mX
\right)
&\subseteq
\left[
-A,A
\right],
\qquad
n\ge1.
\end{align}
For every
$n\ge1$,
the restriction
$g_n|_{\left[-A,A\right]}$
belongs to the compact profile class
$\mathcal G_A$
from
\Cref{lem:brownian-profile-compactness}.
Therefore, after passing to a further subsequence, not relabelled,
there exists
$g_\ast\in\mathcal G_A$
such that
\begin{align}
\sup_{t\in\left[-A,A\right]}
\left|
g_n\left(
t
\right)
-
g_\ast\left(
t
\right)
\right|
\longrightarrow
0.
\label{eq:compact-dict-profile-uniform}
\end{align}
By the definition of
$\mathcal G_A$,
there exists
$\widetilde g\in\mH_{\kB}$
such that
\begin{align}
\widetilde g|_{\left[-A,A\right]}
&=
g_\ast,
&
\left\|
\widetilde g
\right\|_{\mH_{\kB}}
&\le
1.
\end{align}
We now prove uniform convergence of the composed functions.
Fix
$\b x\in\mX$.
Since
$a_n\left(\b x\right)$
and
$a\left(\b x\right)$
belong to
$\left[-A,A\right]$,
\begin{align}
&
\left|
g_n\left(
a_n\left(
\b x
\right)
\right)
-
\widetilde g\left(
a\left(
\b x
\right)
\right)
\right|
\stackrel{(a)}{\le}
\left|
g_n\left(
a_n\left(
\b x
\right)
\right)
-
\widetilde g\left(
a_n\left(
\b x
\right)
\right)
\right|
+
\left|
\widetilde g\left(
a_n\left(
\b x
\right)
\right)
-
\widetilde g\left(
a\left(
\b x
\right)
\right)
\right|
\notag\\
&\stackrel{(b)}{=}
\left|
g_n\left(
a_n\left(
\b x
\right)
\right)
-
g_\ast\left(
a_n\left(
\b x
\right)
\right)
\right|
+
\left|
\widetilde g\left(
a_n\left(
\b x
\right)
\right)
-
\widetilde g\left(
a\left(
\b x
\right)
\right)
\right|
\notag\\
&\stackrel{(c)}{\le}
\sup_{t\in\left[-A,A\right]}
\left|
g_n\left(
t
\right)
-
g_\ast\left(
t
\right)
\right|
+
\left|
\widetilde g\left(
a_n\left(
\b x
\right)
\right)
-
\widetilde g\left(
a\left(
\b x
\right)
\right)
\right|.
\label{eq:compact-dict-composition-pointwise}
\end{align}
Here,
(a) applies the triangle inequality,
(b) uses
$\widetilde g=g_\ast$
on
$\left[-A,A\right]$,
and
(c) bounds the first term by the uniform profile distance.

Define the modulus of continuity of
$\widetilde g$
on
$\left[-A,A\right]$
by
$ \omega_{\widetilde g} \left( \delta \right) := \sup $\allowbreak
$ \left\{ \left| \widetilde g\left( s \right) - \widetilde g\left( t \right) \right| : \right. $\allowbreak
$ \left. s,t\in\left[-A,A\right], \; \left| s-t \right| \le \delta \right\}. $
Since
$\widetilde g$
is continuous on the compact interval
$\left[-A,A\right]$,
it is uniformly continuous, and therefore
$ \omega_{\widetilde g} \left( \delta \right) \longrightarrow 0 $\allowbreak
$ \text{as } \delta\downarrow0. $
Taking the supremum over
$\b x\in\mX$
in
\eqref{eq:compact-dict-composition-pointwise}
gives
\begin{align}
\left\|
f_n
-
\widetilde g\circ a
\right\|_\infty
&\stackrel{(a)}{\le}
\sup_{t\in\left[-A,A\right]}
\left|
g_n\left(
t
\right)
-
g_\ast\left(
t
\right)
\right|
+
\omega_{\widetilde g}
\left(
\left\|
a_n-a
\right\|_\infty
\right).
\end{align}
Here,
(a) also uses
$f_n=g_n\circ a_n$
and the definition of the modulus of continuity.
The first term on the right-hand side converges to zero by
\eqref{eq:compact-dict-profile-uniform}.
The second converges to zero by
\eqref{eq:compact-dict-an-uniform}
and the continuity property of
$\omega_{\widetilde g}$.
Consequently,
$ \left\| f_n - \widetilde g\circ a \right\|_\infty \longrightarrow 0. $
Since
$a\in\mU_l$,
$\widetilde g\in\mH_{\kB}$,
and
$\left\|\widetilde g\right\|_{\mH_{\kB}}\le1$,
the recursive definition gives
$\widetilde g\circ a \in \mU_{l+1}$.
Thus, every sequence in
$\mU_{l+1}$
admits a subsequence converging in
$C\left(\mX\right)$
to an element of
$\mU_{l+1}$.
Therefore,
$\mU_{l+1}$
is sequentially compact.
Since
$C\left(\mX\right)$
is a metric space, sequential compactness is equivalent to
compactness.
The induction is complete, and
$\mU_L$
is compact in
$C\left(\mX\right)$
for every
$L\ge1$.

Finally, consider the canonical map
$J : C\left( \mX \right) \longrightarrow L^2\left( \nu \right), \qquad Jf := f$.
For every
$f\in C\left(\mX\right)$,
\begin{align}
\left\|
Jf
\right\|_{L^2\left(\nu\right)}^2
&\stackrel{(a)}{=}
\int_{\mX}
\left|
f\left(
\b x
\right)
\right|^2
\,\d\nu\left(
\b x
\right)
\stackrel{(b)}{\le}
\left\|
f
\right\|_\infty^2
\nu\left(
\mX
\right)
\stackrel{(c)}{=}
\left\|
f
\right\|_\infty^2.
\end{align}
Here,
(a) applies the definition of the
$L^2\left(\nu\right)$
norm,
(b) uses
$\left|f\left(\b x\right)\right|\le\left\|f\right\|_\infty$,
and
(c) uses
$\nu\left(\mX\right)=1$.
Hence,
$J$
is continuous.
Since
$\mU_L$
is compact in
$C\left(\mX\right)$,
its image under
$J$
is compact in
$L^2\left(\nu\right)$.
Thus,
$\mU_L$
is compact when regarded as a subset of
$L^2\left(\nu\right)$.
\end{proof}

\begin{lemma}\textnormal{(Uniform Brownian profile interpolation estimate)}\ 
\label{thm:profile-bkl-approximation-symmetric}
Let $A>0$ and $m\ge1$, let
$-A=t_0<t_1<\cdots<t_m=A$ be the uniform grid on
$\left[-A,A\right]$, and let $\Pi_mg$ denote the continuous piecewise-linear
interpolant of $g$ on this grid. Then, for every $g\in\mH_{\kB}$,
\begin{align}
\sup_{s\in\left[-A,A\right]}
\left|g\left(s\right)-\left(\Pi_mg\right)\left(s\right)\right|
&\le
\left(\frac{A}{2m}\right)^{1/2}
\left\|g\right\|_{\mH_{\kB}}.
\label{eq:symmetric-profile-interpolation-bound}
\end{align}
\end{lemma}

\begin{proof}
Set $h:=2A/m$. Fix $s\in\left[t_i,t_{i+1}\right]$, and define
\begin{align}
\phi_s\left(r\right)
:=
\frac{t_{i+1}-s}{h}
\mathbf 1_{\left[t_i,s\right]}\left(r\right)
-
\frac{s-t_i}{h}
\mathbf 1_{\left[s,t_{i+1}\right]}\left(r\right).
\end{align}
The interpolation formula and absolute continuity give
\begin{align}
g\left(s\right)-\left(\Pi_mg\right)\left(s\right)
&=
\int_{t_i}^{t_{i+1}}
 g'\left(r\right)\phi_s\left(r\right)
\,\d r.
\end{align}
Moreover,
\begin{align}
\left\|\phi_s\right\|_{L^2\left(\left[t_i,t_{i+1}\right]\right)}^2
&\stackrel{(a)}{=}
\frac{\left(s-t_i\right)\left(t_{i+1}-s\right)}{h}
\stackrel{(b)}{\le}
\frac{h}{4}.
\end{align}
Here, (a) is direct integration and (b) uses
$\left(s-t_i\right)\left(t_{i+1}-s\right)\le h^2/4$. Therefore,
\begin{align}
\left|g\left(s\right)-\left(\Pi_mg\right)\left(s\right)\right|
&\stackrel{(a)}{\le}
\left\|g'\right\|_{L^2\left(\left[t_i,t_{i+1}\right]\right)}
\left\|\phi_s\right\|_{L^2\left(\left[t_i,t_{i+1}\right]\right)}
 \stackrel{(b)}{\le}
\frac{\sqrt h}{2}
\left\|g\right\|_{\mH_{\kB}}
\stackrel{(c)}{=}
\left(\frac{A}{2m}\right)^{1/2}
\left\|g\right\|_{\mH_{\kB}}.
\end{align}
Here, (a) is Cauchy--Schwarz, (b) uses the preceding norm bound and the
Brownian RKHS norm characterization, and (c) substitutes $h=2A/m$. Taking the
supremum proves \eqref{eq:symmetric-profile-interpolation-bound}.
\end{proof}

\begin{lemma}\textnormal{(Uniform range bound for lower-level atomic supports)}\ 
\label{lem:lower-level-atomic-range-bound}
Assume that $\mX\subseteq\mathbb R^d$ is compact and that
$\left\|\bomega\right\|_2\le1$ for every $\bomega\in\Omega$. Let $L\ge2$, and
set
$R_{\mX} := \sup_{\b x\in\mX} \left\|\b x\right\|_2$.
Then every $a\in\mU_{L-1}$ satisfies
$\left|a\left(\b x\right)\right| \le R_{\mX}^{\,2^{-(L-2)}}, \qquad \b x\in\mX$.
Consequently,
$a\left(\mX\right)\subseteq\left[-A_0,A_0\right]$, where
$A_0:=R_{\mX}^{\,2^{-(L-2)}}$.
\end{lemma}

\begin{proof}
Compactness of $\mX$ gives $R_{\mX}<\infty$. Fix $a\in\mU_{L-1}$.
If $L=2$, then $a\left(\b x\right)=\bomega^\top\b x$ for some
$\bomega\in\Omega$, and
\begin{align}
\left|a\left(\b x\right)\right|
&\stackrel{(a)}{\le}
\left\|\bomega\right\|_2
\left\|\b x\right\|_2
\stackrel{(b)}{\le}
R_{\mX}.
\end{align}
Here, (a) is Cauchy--Schwarz and (b) uses
$\left\|\bomega\right\|_2\le1$. If $L\ge3$,
\Cref{lem:path-atom-bound} applied at depth $L-1$ gives
\begin{align}
\left|a\left(\b x\right)\right|
&\stackrel{(a)}{\le}
\left\|\b x\right\|_2^{\,2^{-(L-2)}}
\stackrel{(b)}{\le}
R_{\mX}^{\,2^{-(L-2)}}.
\end{align}
Here, (a) is the atomic pointwise estimate and (b) is the definition of
$R_{\mX}$. Since $2^{-(L-2)}=1$ when $L=2$, both cases prove the stated
bound, and the range inclusion follows immediately.
\end{proof}

\begin{lemma}\textnormal{(Stability of Brownian profile interpolation under composition)}\ 
\label{lem:profile-interpolation-composition}
Let
$A>0$
and
$m\in\mathbb N$,
$m\ge1$.
Let
$
-A=t_0<t_1<\cdots<t_m=A
$
be the uniform grid on
$\left[-A,A\right]$,
and let
$\Pi_m$
denote the associated continuous piecewise-linear interpolation operator. Let
$
a
:
\mX
\longrightarrow
\left[
-A,A
\right]
$
be measurable, and let
$g\in\mH_{\kB}$.
Define
$
u\left(
\b x
\right)
:=
g\left(
a\left(
\b x
\right)
\right)$, $u_m\left(
\b x
\right)
:=
\left(
\Pi_mg
\right)
\left(
a\left(
\b x
\right)
\right)$, $\b x\in\mX.
$
Then
\begin{align}
\left\|
u-u_m
\right\|_{L^2\left(\nu\right)}
\le
\left(
\frac{
A
}{
2
}
\right)^{1/2}
m^{-1/2}
\left\|
g
\right\|_{\mH_{\kB}}.
\label{eq:profile-interpolation-composition-bound}
\end{align}

\end{lemma}

\begin{proof}
For every $\b x\in\mX$, the range assumption on $a$ and
\Cref{thm:profile-bkl-approximation-symmetric} give
\begin{align}
\left|u\left(\b x\right)-u_m\left(\b x\right)\right|
&\stackrel{(a)}{\le}
\left(\frac{A}{2m}\right)^{1/2}
\left\|g\right\|_{\mH_{\kB}}.
\end{align}
Here, (a) applies \eqref{eq:symmetric-profile-interpolation-bound} at
$s=a\left(\b x\right)$. Squaring, integrating, and using
$\nu\left(\mX\right)=1$ yields
\begin{align}
\left\|u-u_m\right\|_{L^2\left(\nu\right)}
&\le
\left(\frac{A}{2m}\right)^{1/2}
\left\|g\right\|_{\mH_{\kB}}
=
\left(\frac{A}{2}\right)^{1/2}
m^{-1/2}
\left\|g\right\|_{\mH_{\kB}},
\end{align}
which proves \eqref{eq:profile-interpolation-composition-bound}.
\end{proof}

\vskip 0.2in
\bibliography{BIB/bkl_references}

@misc{MohammadigohariDiFattaNicosiaPardalos2026BKL,
  author       = {Mohammadigohari, Mahdi and Di Fatta, Giuseppe and Nicosia, Giuseppe and Pardalos, Panos M.},
  title        = {Brownian Kernel Ladders},
  year         = {2026},
  eprint       = {2606.15812},
  archivePrefix = {arXiv},
  primaryClass = {cs.LG},
  url          = {https://arxiv.org/abs/2606.15812},
  note         = {Revised preprint}
}

@article{goldberg1995complexity,
  author    = {Paul W. Goldberg and Mark Jerrum},
  title     = {Bounding the Vapnik--Chervonenkis Dimension of Concept Classes Parameterized by Real Numbers},
  journal   = {Machine Learning},
  volume    = {18},
  number    = {2--3},
  pages     = {131--148},
  year      = {1995},
  doi       = {10.1007/BF00993408}
}

@article{Aronszajn1950,
  author  = {Nachman Aronszajn},
  title   = {Theory of Reproducing Kernels},
  journal = {Transactions of the American Mathematical Society},
  volume  = {68},
  number  = {3},
  pages   = {337--404},
  year    = {1950}
}

@article{Barron1993,
  author  = {Andrew R. Barron},
  title   = {Universal Approximation Bounds for Superpositions of a Sigmoidal Function},
  journal = {IEEE Transactions on Information Theory},
  volume  = {39},
  number  = {3},
  pages   = {930--945},
  year    = {1993}
}

@article{Bach2017,
  author  = {Francis Bach},
  title   = {Breaking the Curse of Dimensionality with Convex Neural Networks},
  journal = {Journal of Machine Learning Research},
  volume  = {18},
  number  = {19},
  pages   = {1--53},
  year    = {2017}
}

@article{follain2025,
  author  = {Bertille Follain and Francis Bach},
  title   = {Enhanced Feature Learning via Regularisation: Integrating Neural Networks and Kernel Methods},
  journal = {Journal of Machine Learning Research},
  volume  = {26},
  number  = {172},
  pages   = {1--56},
  year    = {2025}
}

@article{DeVore1998,
  author  = {Ronald A. DeVore},
  title   = {Nonlinear Approximation},
  journal = {Acta Numerica},
  volume  = {7},
  pages   = {51--150},
  year    = {1998}
}

@book{Temlyakov2011,
  author    = {Vladimir N. Temlyakov},
  title     = {Greedy Approximation},
  publisher = {Cambridge University Press},
  year      = {2011}
}

@article{Chen2024,
  author  = {Zhengdao Chen},
  title   = {Neural Hilbert Ladders: Multi-Layer Neural Networks in Function Space},
  journal = {Journal of Machine Learning Research},
  volume  = {25},
  number  = {109},
  pages   = {1--65},
  year    = {2024}
}

@article{barron2019,
  author  = {Andrew R. Barron and Jason M. Klusowski},
  title   = {Complexity, Statistical Risk, and Metric Entropy of Deep Nets Using Total Path Variation},
  journal = {arXiv preprint arXiv:1902.00800},
  year    = {2019}
}

@inproceedings{bartlett1996,
  author    = {Peter L. Bartlett},
  title     = {For Valid Generalization the Size of the Weights Is More Important than the Size of the Network},
  booktitle = {Advances in Neural Information Processing Systems},
  volume    = {9},
  year      = {1996}
}

@article{bartolucci2023,
  author  = {Francesco Bartolucci and Ernesto De Vito and Lorenzo Rosasco and Stefano Vigogna},
  title   = {Understanding Neural Networks with Reproducing Kernel Banach Spaces},
  journal = {Applied and Computational Harmonic Analysis},
  volume  = {62},
  pages   = {194--236},
  year    = {2023}
}

@article{bartolucci2024,
  author  = {Francesco Bartolucci and Ernesto De Vito and Lorenzo Rosasco and Stefano Vigogna},
  title   = {Neural Reproducing Kernel Banach Spaces and Representer Theorems for Deep Networks},
  journal = {arXiv preprint arXiv:2403.08750},
  year    = {2024}
}

@inproceedings{eldan2016,
  author    = {Ronen Eldan and Ohad Shamir},
  title     = {The Power of Depth for Feedforward Neural Networks},
  booktitle = {Conference on Learning Theory},
  pages     = {907--940},
  publisher = {PMLR},
  year      = {2016}
}

@inproceedings{golowich2018,
  author    = {Noah Golowich and Alexander Rakhlin and Ohad Shamir},
  title     = {Size-Independent Sample Complexity of Neural Networks},
  booktitle = {Conference on Learning Theory},
  pages     = {297--299},
  publisher = {PMLR},
  year      = {2018}
}

@article{heeringa2025,
  author  = {T. J. Heeringa and L. Spek and C. Brune},
  title   = {Deep Networks Are Reproducing Kernel Chains},
  journal = {arXiv preprint arXiv:2501.03697},
  year    = {2025}
}

@article{kurkova2002,
  author  = {V\v{e}ra Kurkov{\'a} and Marcello Sanguineti},
  title   = {Bounds on Rates of Variable-Basis and Neural-Network Approximation},
  journal = {IEEE Transactions on Information Theory},
  volume  = {47},
  number  = {6},
  pages   = {2659--2666},
  year    = {2002}
}

@article{ma2022,
  author  = {Chao Ma and Lei Wu and others},
  title   = {The Barron Space and the Flow-Induced Function Spaces for Neural Network Models},
  journal = {Constructive Approximation},
  volume  = {55},
  number  = {1},
  pages   = {369--406},
  year    = {2022}
}

@inproceedings{neyshabur2015,
  author    = {Behnam Neyshabur and Ryota Tomioka and Nathan Srebro},
  title     = {Norm-Based Capacity Control in Neural Networks},
  booktitle = {Conference on Learning Theory},
  pages     = {1376--1401},
  publisher = {PMLR},
  year      = {2015}
}

@article{ongie2026,
  author  = {Greg Ongie and Rahul Parhi},
  title   = {Representation Costs in Data Science: Foundations and the Quasi-Banach Spaces of Deep Neural Networks},
  journal = {arXiv preprint arXiv:2606.14954},
  year    = {2026}
}

@inproceedings{ongie2019,
  author    = {Greg Ongie and Rebecca Willett and Daniel Soudry and Nathan Srebro},
  title     = {A Function Space View of Bounded Norm Infinite Width ReLU Nets: The Multivariate Case},
  booktitle = {International Conference on Learning Representations},
  year      = {2020},
  note      = {Originally available as arXiv:1907.05681 (2019)}
}

@article{parhi2021,
  author  = {Rahul Parhi and Robert D. Nowak},
  title   = {Banach Space Representer Theorems for Neural Networks and Ridge Splines},
  journal = {Journal of Machine Learning Research},
  volume  = {22},
  number  = {43},
  pages   = {1--40},
  year    = {2021}
}

@article{parhi2022a,
  author  = {Rahul Parhi and Robert D. Nowak},
  title   = {What Kinds of Functions Do Deep Neural Networks Learn? Insights from Variational Spline Theory},
  journal = {SIAM Journal on Mathematics of Data Science},
  volume  = {4},
  number  = {2},
  pages   = {464--489},
  year    = {2022}
}

@article{parhi2022b,
  author  = {Rahul Parhi and Robert D. Nowak},
  title   = {Near-Minimax Optimal Estimation with Shallow ReLU Neural Networks},
  journal = {IEEE Transactions on Information Theory},
  volume  = {69},
  number  = {2},
  pages   = {1125--1140},
  year    = {2022}
}

@inproceedings{parkinson2024,
  author    = {Samuel Parkinson and Greg Ongie and Rebecca Willett and Ohad Shamir and Nathan Srebro},
  title     = {Depth Separation in Norm-Bounded Infinite-Width Neural Networks},
  booktitle = {Proceedings of the 37th Annual Conference on Learning Theory},
  pages     = {4082--4114},
  publisher = {PMLR},
  year      = {2024}
}

@inproceedings{savarese2019,
  author    = {Pedro Savarese and Itay Evron and Daniel Soudry and Nathan Srebro},
  title     = {How Do Infinite Width Bounded Norm Networks Look in Function Space?},
  booktitle = {Conference on Learning Theory},
  pages     = {2667--2690},
  publisher = {PMLR},
  year      = {2019}
}

@article{shenouda2024,
  author  = {Joseph Shenouda and Rahul Parhi and Kangwook Lee and Robert D. Nowak},
  title   = {Variation Spaces for Multi-Output Neural Networks: Insights on Multi-Task Learning and Network Compression},
  journal = {Journal of Machine Learning Research},
  volume  = {25},
  number  = {231},
  pages   = {1--40},
  year    = {2024}
}

@article{siegel2023,
  author  = {Jonathan W. Siegel and Jinxin Xu},
  title   = {Characterization of the Variation Spaces Corresponding to Shallow Neural Networks},
  journal = {Constructive Approximation},
  volume  = {57},
  number  = {3},
  pages   = {1109--1132},
  year    = {2023}
}

@article{siegel2024,
  author  = {Jonathan W. Siegel and Jinxin Xu},
  title   = {Sharp Bounds on the Approximation Rates, Metric Entropy, and n-Widths of Shallow Neural Networks},
  journal = {Foundations of Computational Mathematics},
  volume  = {24},
  number  = {2},
  pages   = {481--537},
  year    = {2024}
}

@inproceedings{telgarsky2016,
  author    = {Matus Telgarsky},
  title     = {Benefits of Depth in Neural Networks},
  booktitle = {Conference on Learning Theory},
  pages     = {1517--1539},
  publisher = {PMLR},
  year      = {2016}
}

@inproceedings{vardi2020,
  author    = {Gal Vardi and Ohad Shamir},
  title     = {Neural Networks with Small Weights and Depth-Separation Barriers},
  booktitle = {Advances in Neural Information Processing Systems},
  volume    = {33},
  pages     = {19433--19442},
  year      = {2020}
}

@article{venturi2022,
  author  = {Luca Venturi and Samy Jelassi and Tomasz Ozuch and Joan Bruna},
  title   = {Depth Separation Beyond Radial Functions},
  journal = {Journal of Machine Learning Research},
  volume  = {23},
  number  = {122},
  pages   = {1--56},
  year    = {2022}
}

@article{wojtowytsch2020,
  author  = {E, Weinan and Wojtowytsch, Stephan},
  title   = {On the Banach Spaces Associated with Multi-Layer {ReLU} Networks: Function Representation, Approximation Theory and Gradient Descent Dynamics},
  journal = {CSIAM Transactions on Applied Mathematics},
  volume  = {1},
  number  = {3},
  pages   = {387--440},
  year    = {2020}
}

@misc{nakhleh2026deep,
  author       = {Julia Nakhleh and Robert D. Nowak},
  title        = {Deep Neural Variation Spaces: A Unifying Perspective on Depth and Complexity},
  year         = {2026},
  eprint       = {2607.05546},
  archivePrefix= {arXiv},
  primaryClass = {stat.ML},
  url          = {https://arxiv.org/abs/2607.05546}
}

\end{document}